\documentclass[twoside,10.5pt]{article}

\usepackage[preprint]{aistats2027}

\usepackage[colorlinks=true,urlcolor=blue,linkcolor=red,citecolor=teal,backref=page]{hyperref}
\usepackage{makecell}
\usepackage[textsize=tiny]{todonotes}

\usepackage{mathtools}
\usepackage{amsmath, amssymb, graphicx, url}
\usepackage{algorithm}
\usepackage{algpseudocode}
\usepackage{wrapfig}
\usepackage{booktabs}

\usepackage{tcolorbox}
\usepackage[shortlabels]{enumitem}
\usepackage{dsfont, bm}
\usepackage{footnote}
\usepackage{caption}
\usepackage{subcaption}

\newcommand{\bx}{\mathbf{x}}

\newcommand\pol{\ensuremath{\boldsymbol{\pi}}}

\newcommand \defn {\mathrel{\triangleq}}

\DeclareMathOperator*{\argmax}{arg\,max}
\DeclareMathOperator*{\argmin}{arg\,min}

\newcommand{\indicator}{\mathds{1}}

\newcommand{\reals}{\mathbb{R}}

\newcommand{\obo}{Option Bellman operator }

\usepackage{amsthm}
\newtheorem{theorem}{Theorem}[section]
\newtheorem{corollary}[theorem]{Corollary}
\newtheorem{lemma}[theorem]{Lemma}
\newtheorem{proposition}[theorem]{Proposition}
\newtheorem{definition}{Definition}

\newtheorem{remark}{Remark}

\newtheorem{property}{Property}

\newtheorem{example}{Example}
\newtheorem{assumption}{Assumption}

\makeatletter
\newtheorem*{rep@theorem}{\rep@title}
\newcommand{\newreptheorem}[2]{%
	\newenvironment{rep#1}[1]{%
		\def\rep@title{\textbf{#2} \ref{##1}}%
		\begin{rep@theorem}}%
		{\end{rep@theorem}}}
\makeatother
\newreptheorem{theorem}{Theorem}
\newreptheorem{lemma}{Lemma}
\newreptheorem{proposition}{Proposition}
\newreptheorem{assumption}{Assumption}
\newreptheorem{corollary}{Corollary}

\allowdisplaybreaks%
\newif\ifsinglecol
\singlecolfalse

\usepackage[nohints]{minitoc}

\usepackage{tikz}
\usetikzlibrary{arrows.meta}
\usetikzlibrary{automata}
\usetikzlibrary{bayesnet, arrows, calc, shapes, backgrounds,arrows,decorations.pathmorphing,fit,positioning}
\tikzset{
   container/.style = {rectangle, rounded corners, draw=yellow, dashed,
fit=#1, inner sep=6mm, node contents={}},
circle-label/.style = {circle, draw}
        }
\tikzset{box/.style={draw, diamond, thick, text centered, minimum height=0.5cm, minimum width=1cm, text width=0.9cm}}
\tikzset{line/.style={draw, thick, -latex'}}

\def\bx{\mathbf{x}}

\def\bbR{\mathbb{R}}

\def\cA{\mathcal{A}}

\def\cF{\mathcal{F}}

\def\cL{\mathcal{L}}
\def\cM{\mathcal{M}}

\def\cO{\mathcal{O}}

\def\cQ{\mathcal{Q}}

\def\cS{\mathcal{S}}

\def\cX{\mathcal{X}}

\def\cZ{\mathcal{Z}}
\def\E{\mathbb{E}}
\newcommand{\expect}[1]{\mathbb{E}\left[#1\right]}

\def\sp{\mathrm{sp}}

\newtheorem{observation}{Observation}
\newcommand{\tildeO}{\tilde{\mathcal{O}}}
\def\barr{\bar{r}}
\def\barl{\bar{l}}
\def\barp{\bar{p}}
\newtcolorbox{question}{%
colback=cyan!4, colframe=cyan!50!black, boxrule=0.6pt, arc=2pt, left=6pt, right=6pt, top=4pt, bottom=4pt, boxsep=2pt,
fontupper=\small}

\renewcommand\eqref[1]{\ensuremath{(\textrm{\ref{#1}})}}

\usepackage[sort&compress]{natbib}

\newif\ifdoublecol
\doublecolfalse

\begin{document}


\runningtitle{Non-asymptotic Convergence of Average-reward Q-learning with Options}
\runningauthor{Saha, Deb, Jonsson, Basu}

\ifdoublecol
\twocolumn[

\aistatstitle{
Non-asymptotic Convergence of\\
Average-reward Q-learning with Options
}


\aistatsauthor{ Kintan Saha\\
Indian Institute of Science\\
\texttt{kintansaha@iisc.ac.in}
\And
Ahana Deb\\
Universitat Pompeu Fabra\\
\texttt{ahana.deb@upf.edu}
\AND
Anders Jonsson\\
Universitat Pompeu Fabra\\
\texttt{anders.jonsson@upf.edu}
\And
Debabrota Basu\\
Univ. Lille, Inria, CNRS- CRIStAL, Centrale Lille\\
\texttt{debabrota.basu@inria.fr}
}
\aistatsaddress{}
]
\setcounter{parttocdepth}{3}
\parttoc 
\faketableofcontents 

\else 
\onecolumn 
\aistatstitle{
Non-asymptotic Convergence of\\
Average-reward Q-learning with Options
}


\aistatsauthor{ Kintan Saha\\
Indian Institute of Science\\
\texttt{kintansaha@iisc.ac.in}
\And
Ahana Deb\\
Universitat Pompeu Fabra\\
\texttt{ahana.deb@upf.edu}
\AND
Anders Jonsson\\
Universitat Pompeu Fabra\\
\texttt{anders.jonsson@upf.edu}
\And
Debabrota Basu\\
Univ. Lille, Inria, CNRS- CRIStAL, Centrale Lille\\
\texttt{debabrota.basu@inria.fr}
}
\aistatsaddress{}

\fi 

\begin{abstract}
We study the finite-sample convergence of Q-learning algorithms for the
average-reward formulation of Reinforcement Learning (RL) with options.
While options allow temporarily extended and hierarchical actions, the average-reward formulation incorporates continual tasks beyond discounted and episodic RL.
\textit{We establish the first $\tildeO(1/\textrm{iterations})$
convergence guarantees, in the mean-square sense, for average-reward option
Q-learning algorithms in both the inter and intra-option settings.}
First, we prove that both inter- and intra-option Bellman operators satisfy
multi-step span contractions under unichain assumptions.
Then, we propose synchronous estimators of these multi-step
option Bellman operators leading to two variants of option Q-learning.
Finally, we extend the stochastic approximation framework to prove convergence rates of these two algorithms under the minimal assumptions used in the literature.
As a result, we obtain the first non-asymptotic convergence rates for Q-learning algorithms in average-reward RL with hierarchies and reward machines.
Our experimental results across RL environments corroborate the theoretical
findings.
\end{abstract}

\ifdoublecol
\,
\else 
\setcounter{parttocdepth}{3}
\parttoc 
\tableofcontents 
\newpage\fi 

\section{\MakeUppercase{Introduction}}
In Reinforcement Learning (RL), an agent learns a behavioural strategy (aka \textit{policy}) while sequentially interacting with an environment by taking a sequence of \textit{actions}, collecting new observations regarding the environment, and goodness of the actions (also known as \textit{rewards}). The aim is to find a policy optimising the cumulative \textit{utility} of the agent's actions~\citep{sutton1998reinforcement}.
In practice, we encounter RL problems in which the agent has to solve complex tasks and predict the consequences of its actions over multiple levels of abstraction. 
For example, hierarchical RL (HRL) decomposes a complex high-level task into several low-level tasks, which are often simpler to solve~\citep{barto2003recent}.
In RL with reward machines, the agent has access to a reward machine, i.e.~an automaton whose states compactly represent the interaction history and in turn determines the reward signal~\citep{toro_reward_machines}.
In these settings, the agent has to take temporally extended actions and learn their consequences to optimise the long-term utilities.
\textbf{RL with \textit{options}} framework~\citep{10.1007/BFb0026709,SUTTON1999181} allows learning and planning with such temporally extended actions (aka options). \textit{An option is a set of initial states, an internal policy for  primitive actions, and a termination condition, yielding a stochastic holding time of the option.} 

\begin{example}[Four Rooms~\citep{SUTTON1999181}]
    An agent navigates through the four points A, B, C, and D situated at the corners of four rooms connected by hallways. The agent's primitive actions are going left, right, up, or down, while the options are going from one hallway to the other 
    and following a room-specific policy to take primitive actions inside it. 
    Here, an RL agent can take two approaches~\citep{SUTTON1999181}. (a) \textbf{Inter-option:} Learn and plan by treating options as actions, and update the estimates on goodness of an option only after its termination. 
    (b) \textbf{Intra-option:} Learn and plan by taking primitive actions, and update the estimates for the options after each action.
\end{example}\vspace*{-.5em}

Options allow hierarchical learning and planning \citep{mitchener2022detect, ravindran2002model}, skill sharing across tasks \citep{ramesh2019successor, klissarov2021flexible}, and faster exploration via multi-step actions \citep{klissarov2021flexible}.
RL with options was originally proposed for Markov Decision Processes (MDPs: a model of RL interaction) with discounted cumulative rewards as utilities~\citep{SUTTON1999181,bacon2017option,tiwari2019natural}. The framework was later extended to the average-reward setting~\citep{wan2021average}, in which the agent aims to maximise the rate of reward obtained per-step over an infinite number of possible interactions.
It allows control of the long-term behaviour of continuing RL algorithms outside the episodic setting~\citep{JMLR:v8:ghavamzadeh07a}.
Thus, we focus on the \textit{average-reward setting} of RL with options.

But convergence analysis of average-reward RL is often more challenging than its discounted counterpart, in part due to the absence of standard norm contractions. 
While the asymptotic analysis of average-reward RL with primitive actions is well-developed~\citep{doi:10.1137/S0363012999361974, doi:10.1137/S0363012998346621, yu2024asynchronous}, fewer theoretical guarantees are available for its finite-time analysis. 
Recently, finite-time convergence guarantees were established for average-reward Q-learning with primitive actions~\citep{chen2025nonasymptoticguaranteesaveragerewardqlearning, chen2025nonasymptotictheoryseminormlyapunov}. 
However, \textit{there is no finite-time convergence guarantee available for average-reward option-based Q-learning} as the existing results are only asymptotic~\citep{wan2024convergenceaveragerewardqlearningweakly,yu2024asynchronous}.
In this paper, \textit{we bridge this gap}. 

\textbf{Our Contributions.} With four axes, we theoretically and numerically answer the question:
\begin{question}
    \textit{How do the expected errors in inter- \& intra-option Q-learning, i.e. the discrepancies between an estimate $\hat{Q}_k$ yielded by algorithms and the optimal Q-function $Q^*$, decrease with the number of iterations $k$?}
\end{question}

\begin{enumerate}[leftmargin=*,nosep]
    \item The convergence of Q-learning algorithms depends on the contraction of the option Bellman operator in a compatible seminorm. We establish $J$-step contraction for inter- and intra-option Bellman operators with respect to span seminorm extending results for MDPs~\citep{mdp_book}. 
    \item We develop new $J$-step estimators for synchronous inter-
and intra-option Q-learning. An unbiased estimator would require intermediate
greedy option selections based on unavailable greedy finite-horizon Bellman values. We replace these oracle decisions by approximately
greedy choices constructed from empirical Bellman operators, and develop a
perturbation analysis to control the resulting bias.  The inter-option setting
introduces an additional source of bias because the Bellman updates depend on
quantities determined by the random option durations and rewards, which must
also be estimated online from past interactions. Our analysis jointly controls
both sources of bias and further rectifies the estimators and corresponding
rates for average-reward Q-learning with primitive actions~\citep{zhang2021finite, chen2025nonasymptotictheoryseminormlyapunov, naskar2026parameter}. 
    \item We prove the first finite-time convergence bounds for synchronous inter and intra-option Q-learning (Table~\ref{tab: ftb}). We establish $\widetilde \cO(1/k)$ convergence rate for both algorithms with a learning rate of $\Theta(1/k)$. Here, we assume only bounded second moment of the holding times of options~\citep{yu2024asynchronous}, instead of a uniform lower-upper bounds on the holding times~\citep{wan2024convergenceaveragerewardqlearningweakly}. Our analysis also extends existing seminorm stochastic-approximation theory~\citep{chen2025nonasymptotictheoryseminormlyapunov} by allowing decaying biased updates, rather than requiring  unbiased updates.
    \item As a consequence, we obtain finite-time convergence guarantees for RL with Hierarchical Reward Machines, whereas most prior convergence bounds are limited to the discounted and episodic settings~\citep{gurtler2021hierarchical, pmlr-v70-vezhnevets17a, barto2003recent}. 
    We numerically validate across multiple environments that using our $J$-step estimators lead to better convergence and contraction for both inter and intra-option Q-learning in comparison with their one-step variants.
    \vspace*{-.3em}
\end{enumerate}

Further discussion on related works is in Appendix~\ref{sec:related_work}. 

\setlength{\textfloatsep}{6pt}
\begin{table}[t!]
\centering
\caption{Finite-time bounds on $\Delta_k^2 \defn \mathbb{E}\!\left[\sp(Q_k-Q^*)^2\right]$. We present the bounds in $\tildeO(\cdot)$. $\sp(\bx)$ denotes span: $(\max_{i} \bx(i) - \min_{i'} \bx(i'))/2$. 
}\label{tab: ftb}
\ifdoublecol
\resizebox{0.9\columnwidth}{!}{
\begin{tabular}{l|c}
\toprule
\textbf{Stepsize} & $\alpha_k = \frac{4}{(1-\gamma)(k+h)}$\\
\hline
\textbf{Inter-option} (Theorem \ref{thm:sync-inter-ftb})
&
$\frac{1}{k}
\left(\Delta_0^2 + \frac{\sp(Q^*)^2+|\cS||\cO|J^4}{(1-\gamma)^3}\right)$
\\[2ex]
\textbf{Intra-option} (Theorem \ref{thm: sync-intra-ftb})
&
$\frac{1}{k} \left(\Delta_0^2 +\frac{\sp(Q^*)^2+|\cS|J^4}{(1-\gamma)^3}\right)$
\\[.2ex]
\hline
\end{tabular}}
\else 
\begin{tabular}{l|c}
\toprule
\textbf{Stepsize} & $\alpha_k = \frac{4}{(1-\gamma)(k+h)}$\\
\hline
\textbf{Inter-option} (Theorem \ref{thm:sync-inter-ftb})
&
$\frac{1}{k}
\left(\Delta_0^2 + \frac{\sp(Q^*)^2+|\cS||\cO|J^4}{(1-\gamma)^3}\right)$
\\[2ex]
\textbf{Intra-option} (Theorem \ref{thm: sync-intra-ftb})
&
$\frac{1}{k} \left(\Delta_0^2 +\frac{\sp(Q^*)^2+|\cS|J^4}{(1-\gamma)^3}\right)$
\\[.2ex]
\hline
\end{tabular}
\fi 
\end{table}

\section{\MakeUppercase{A Primer on Average- reward Option Q-learning}}\label{sec:formalism}
We begin by formalising the average-reward RL with options as a Semi MDP (SMDP), and provide the details of the inter and intra-option Q-learning. 

\textbf{Average-reward Semi MDPs.} 
A finite state-action SMDP (or tabular SMDP) is a tuple $\mathcal{M} \triangleq (\cS, \cA,  r, p, \cL)$~\citep{SCHWEITZER1971495,bradtke1994reinforcement,parr1998hierarchical}. 
Here, $\cS$ is a finite state space. $\cA$ is a finite action space. $\ell(s,a) \sim \cL$ denotes a stochastic holding time of action $a$ at state $s$ after which the environment moves to next state $s' \sim p(\cdot \mid s,a, \ell(s,a)) \in \Delta_{\cS}$ and yields the corresponding reward $r(s,a)$. $p$ is called the transition kernel. Note that SMDPs include the classical MDPs~\citep{mdp_book} as a special case with unit holding time.

At each step $t \in \mathbb{N}$, a learning agent observes a state $S_t$ and takes an action $A_t$ according to a policy $\pi \in \Delta_{\cA}$. The policy $\pi$ can use either the information of the present state (Markovian) or the history of interactions to take an action at each step $t$. 
In return, after a stochastic holding time $\ell(S_t, A_t)$, the agent observes corresponding reward $r(S_t, A_t)$, and the environment transits to a next state $S \sim p(\cdot|S_t, A_t, \ell(S_t, A_t))$.
Given an initial state $s_0 \in \cS$, the \textit{average reward} (or reward-rate) of playing policy $\pi$ is defined as
\begin{align}
\bar r^\pi(s_0) \triangleq \lim_{T \to \infty} \frac{1}{T} \mathbb{E}_{\pi, \cM} \left[\sum_{t=0}^{T-1} r(S_t, A_t)\mid S_0 =s_0\right]\,.    \tag{average-reward}
\end{align}
We state a minimal requirement for well-definedness and boundedness of $\bar r^\pi(s_0)$ across policies $\pi$ and initial states $s_0$~\citep{SCHWEITZER1971495,wan2024convergenceaveragerewardqlearningweakly}.
\begin{assumption}[Well-behaved SMDP]\label{assm:smdp}
For all state-action pairs, the second moment of holding time $\E[\ell^2(s,a)] \le c^2 < \infty$, and the underlying primitive one-step rewards are bounded, i.e. $r(s,a) \in [0,1]$.  
\end{assumption}
A policy attaining the \textit{optimal reward rate} $\bar r^* \triangleq \sup_{\pi} r^\pi(s_0)$ for all initial states is referred to be an \textit{optimal policy}. 
If the \textit{optimal reward rate} is independent of the initial state, $(\bar r^*, Q^*)$ is a solution of the average-reward optimality equation \eqref{eq:aoe} over $(r^*,Q)$~\citep{SCHWEITZER1971495,mdp_book}:
\begin{align}\label{eq:aoe}
Q(s,a)
&=\barr(s,a) - r^* \barl(s,a)+\sum_{s'} \barp(s'|s,a)\max_{a'} Q(s',a')\notag\\
&\triangleq (HQ)(s,a)\,, \tag{AOE}
\end{align}
where $\barr(s,a) = \mathbb{E}[r \mid s,a]$, $\barl(s,a) = \mathbb{E}[\ell \mid s,a]$, and $\barp(s'|s,a) = \sum_{r,\ell} p(s',r,\ell|s,a)$. 
Under Assumption~\ref{assm:smdp}, $\bar r^*$ is achieved by a deterministic stationary optimal policy~\citep{SCHWEITZER1971495}. Further, $H$ is referred to as a \textit{Bellman operator}, while $Q^*$ is one of its fixed points.

In \eqref{eq:aoe}, $Q: \cS \times \cA \rightarrow \reals$ is a \textit{Q-function}, and $Q^*$ is referred to as an \textit{optimal Q-function}.
Further, for a unichain SMDP~\citep{mdp_book}, a Q-function solving~\eqref{eq:aoe} is defined up to an additive constant leading to the equivalence class $\mathcal{Q}^\ast \triangleq \{ Q^\ast + c\mathbf{1} : c \in \mathbb{R} \}$.
\textit{A \textbf{Q-learning} algorithm tries to iteratively solve \eqref{eq:aoe} by plugging in the estimates of Q-functions, rewards, and holding times updated with collected samples}~\citep{sutton1998reinforcement,wan2021average}.



\textbf{Average-reward MDPs with Options.}
\cite{SUTTON1999181} propose the options framework to model RL problems with temporarily extended actions. Specifically, a tabular MDP with option is defined as $\cM \triangleq (\cS, \cA, \cO, r, p)$. While $\cS, \cA, r, p$ are as in the MDPs, $\cO$ denotes the option set, where each option $o$ is defined by a tuple $(\pi(\cdot|\cdot, o), \beta(\cdot,o))$. 
Here, $\pi(a|s,o)$ is the option's policy dictating the probability of taking an action $a$ while executing an option $o$ at state $s$, and $\beta(s,o)$ is a probability dictating the option’s termination while executing at state $s$.
For simplicity, we consider that any option can start from any state~\citep{wan2021average}.
Executing an option at state $s$ generates: a trajectory of primitive actions, a cumulative reward, a random duration $\ell$, and a terminal state $s'$. 
Thus, options induce an SMDP $(\cS, \cO, \hat{r}, \hat{p}, \hat{L})$ over the state-option space $(\cS, \cO)$. Here, $\hat{r}$ is the set of all possible options’ cumulative rewards, $\hat{L}$ is the set of all possible lengths of options dictated by the termination probabilities $\{\beta(\cdot|\cdot,o): o \in \cO\}$, and $\hat{p}$ is the options' transition kernel, such that $\hat{p}(s', \hat{r}, {\ell} \mid s,o) \triangleq \Pr(S_{t+1}=s', \hat{r}_t = \hat{r}, \hat{L}_t = {\ell}\mid S_{t}=s, O_t = o)$. Given a stationary hierarchical (option-level) policy $\pol: \cS \rightarrow \Delta_{\cO}$, and initial $(s,o)$, the option-level average-reward is~\citep{Ross_1970}, 
\begin{align}\label{eq:OAR}
\bar r^{\pol}(s, o) &\triangleq \lim_{T \to \infty} \mathbb{E}_{\pol,\cM} \left[\frac{1}{T} \sum_{t=0}^{T-1} r_t \mid S_0 =s, O_0 = o\right]\notag\\
&= \lim_{n \to \infty} \frac{\mathbb{E}_{\pol,\cM} \left[\sum_{i=0}^{n-1} \hat{r}_i \mid S_0 =s, O_0 = o\right]}{\mathbb{E}_{\pol,\cM} \left[\sum_{i=0}^{n-1} {\ell}_i \mid S_0 =s, O_0 = o\right]}.\tag{OAR}
\end{align}
Here, $\hat{r}(s,o)$, $\hat{\ell}(s,o)$, and $\hat{p}(\cdot\mid s,o)$ denote the expected reward, the expected holding time, and the state-option transition probability, respectively. 
Q-learning algorithms~\citep{wan2021average,wan2024convergenceaveragerewardqlearningweakly} aim to solve \eqref{eq:aoe} corresponding to $(\cS, \cO, \hat{r}, \hat{p}, \hat{L})$ and option-level average-reward \eqref{eq:OAR}, and to obtain the optimal reward rate $\bar{r}^*$ and an optimal option Q-function $Q^*: \cS \times \cO \rightarrow \reals$.
\cite{SUTTON1999181} propose two approaches to solve this: inter-option and intra-option Q-learning.

\textbf{I. Inter-option learning.} Inter-option learning treats options as primitive actions in an SMDP.
Hence, the \textit{option Q-function} $Q \in \mathbb{R}^{S \times O}$ leads to the \textit{inter-option Bellman operator}
\begin{equation}
\label{eq:inter-operator}
(H_{\mathrm{inter}} Q)(s,o)
\triangleq \hat r(s,o)
+
\sum\nolimits_{s'} \hat p(s'|s,o)\max_{o'} Q(s',o').
\end{equation}
This formulation captures the option-level reward $\hat{r}$ evaluated at the option termination times.
This yields the iterates of the \textit{inter-option differential Q-learning} algorithm~\citep{wan2021average}
\ifdoublecol
\begin{align}\label{eq:inter_option_async_q}
\hspace{-0.5em}    Q_{k+1}(s,&o) ~\gets Q_k(s,o) + \alpha_k(s,o)  \times\notag\\
&\left((\hat{H}_{\rm inter}Q_k)(s,o) - Q_k(s,o) -
\bar r_k \,\hat \ell(s,o)\right)
\end{align}
\else
\begin{align}\label{eq:inter_option_async_q}
\hspace{-0.5em}    Q_{k+1}(s,o) ~\gets Q_k(s,o) + \alpha_k(s,o)  \times\left((\hat{H}_{\rm inter}Q_k)(s,o) - Q_k(s,o) -
\bar r_k \,\hat \ell(s,o)\right)
\end{align}
for all $(s,o)$. Here, $\alpha_k(s,o) \in \reals_+$ is the stepsize for state-option pair $(s,o)$,$(\hat{H}_{\rm inter}Q_k)$ is an estimate of the inter-option Bellman operator and $\bar r_k$ is an estimate of $\bar r^*$.
\eqref{eq:inter_option_async_q} defines a \textit{synchronous inter-option Q-learning} algorithm, when the iterates are updated for all state-options simultaneously.

\textbf{II. Intra-option learning.} Intra-option Q-learning updates the Q-function after every primitive action. This further induces the option-averaged dynamics,
$r(s,o) = \sum_a \pi(a|s, o) r(s,a)$ and
$p(s'|s,o) = \sum_a \pi(a|s, o) p(s'|s,a)$. The corresponding intra-option Bellman operator is
\ifdoublecol
\begin{align}\label{eq:intra-operator}
\hspace{-2em}&(H_{\mathrm{intra}} Q)(s,o)
\triangleq
r(s,o)
+
\sum_{s'\in\cS} p(s'|s,o) \times \notag\\
\hspace{-1em}&\Big[
\underbrace{(1-\beta(s', o))}_{\text{non-termination}}Q(s',o)
+
\underbrace{\beta(s', o)}_{\text{termination}}\max_{o'} Q(s',o')
\Big]
\end{align}
\else 
\begin{align}\label{eq:intra-operator}
\hspace{-2em}(H_{\mathrm{intra}} Q)(s,o)
\triangleq
r(s,o)
+
\sum_{s'\in\cS} p(s'|s,o) \times \Big[
\underbrace{(1-\beta(s', o))}_{\text{non-termination}}Q(s',o)
+
\underbrace{\beta(s', o)}_{\text{termination}}\max_{o'} Q(s',o')
\Big]
\end{align}
This operator reflects the fact that with probability $1-\beta$, the option continues, and with probability $\beta$, control switches to a new option. 
The \textit{synchronous intra-option differential Q-learning}~\citep{wan2021average} iterates are 
\ifdoublecol \begin{align}
Q_{k+1}(s,o) \gets
&~Q_{k}(s,o) + \alpha_k(s,o) \nonumber\\
&\left((\widehat{H}_{\rm intra}Q_k)(s,o) - \bar r_k -Q_k(s,o)\right).\label{eq:intra_option_async_q}
\end{align}
\else
\begin{align}
Q_{k+1}(s,o) \gets
~Q_{k}(s,o) + \alpha_k(s,o)\left((\widehat{H}_{\rm intra}Q_k)(s,o) - \bar r_k -Q_k(s,o)\right).\label{eq:intra_option_async_q}
\end{align}
for all $(s,o)$. Here, $(\hat{H}_{\rm intra}Q_k)$ is an estimate of the intra-option Bellman operator constructed by sampling $a \sim \pi(\cdot |s,o)$ and using the corresponding reward for each $(s,o)$, and $\bar r_k$ is an estimate of $\bar r^*$. 

\textbf{Goal:} We measure the convergence rate of average-reward inter- and intra-option Q-learning algorithms by bounding the decay of the mean square error of the iterates from the optimal Q-function, i.e. $\E[\sp(Q_k - Q^*)^2]$, with the number of iterations $k$. 
We use the \textit{span seminorm}, i.e. $\mathrm{sp}(Q) \triangleq (\max_{s,a} Q(s,a) - \min_{s,a} Q(s,a))/2$, to measure errors of Q-value iterates as it is invariant up to additive constants.
In this work, \textit{we prove multi-step Q-learning with options as contractions, design multi-step inter-/intra-option Q-learnings, and finally, prove their convergence rates to be $\tildeO(1/k)$ for properly tuned stepsizes.}

\section{\MakeUppercase{$J$-step Contraction of Option Bellman Operator}}\label{sec:J-contraction}\vspace*{-.5em}
Contraction is a key property of Bellman operators that ensure convergence of Q-learning algorithms.
While contraction is immediate in the discounted setting, it is often unavailable in the average-reward setting~\citep{mdp_book, zhang2021finite}. Thus, even for classical MDPs, convergence analysis opens with proving the contraction of the $J$-step Bellman operator, \textit{i.e. the $J$-fold composition $H^J$}, in the span seminorm~\citep{chen2025nonasymptotictheoryseminormlyapunov}.  
\textit{We prove that both $J$-step inter- and intra-option Bellman operators are contractions}.

\textbf{Normalized inter-option operator.}
The standard inter-option Bellman operator in Section \ref{sec:formalism} is the one used in the
average-reward options literature. But a direct stochastic approximation of \eqref{eq:aoe} requires estimating both
the differential value function and the unknown optimal reward rate $\bar r^\star$, naturally leading to a coupled or two-timescale recursion.
To obtain a single-timescale recursion, we instead develop the equivalent
Schweitzer-transformed operator \citep{wan2024convergenceaveragerewardqlearningweakly}. 
Specifically, for
$\lambda(s,o)\triangleq(\hat\ell(s,o))^{-1}$, $\rho(s,o)\triangleq{\hat r(s,o)}\lambda(s,o)$,
we define
\ifdoublecol
\begin{align}\label{eq:normalised_inter_operator}
\hspace{-1em}(\bar H_{\rm inter}Q)(s,o)
&\triangleq
\rho(s,o)
+
(1-\lambda(s,o))Q(s,o)\notag\\
&+
\lambda(s,o)
\sum_{s'}\hat p(s'|s,o)\max_{o'}Q(s',o').    
\end{align}
\else
\begin{align}\label{eq:normalised_inter_operator}
\hspace{-1em}(\bar H_{\rm inter}Q)(s,o)
\triangleq
\rho(s,o)
+
(1-\lambda(s,o))Q(s,o)+
\lambda(s,o)
\sum_{s'}\hat p(s'|s,o)\max_{o'}Q(s',o').    
\end{align}

Multiplying \eqref{eq:aoe} by
$\lambda(s,o)$ gives
$
\bar H_{\rm inter}Q^\star-Q^\star
=
\bar r^\star \mathbf 1.
$
Hence, $Q^\star$ is a fixed point of $\bar H_{\rm inter}$ in the span seminorm. Now, we state the structures required to prove $\bar H_{\rm inter}^J$ to be a contraction.
Define $\cX \triangleq \cS \times \cO$.
\setlength{\textfloatsep}{6pt}
\begin{figure}[t!]
        \centering
        \ifdoublecol
        \includegraphics[width=0.8\linewidth]{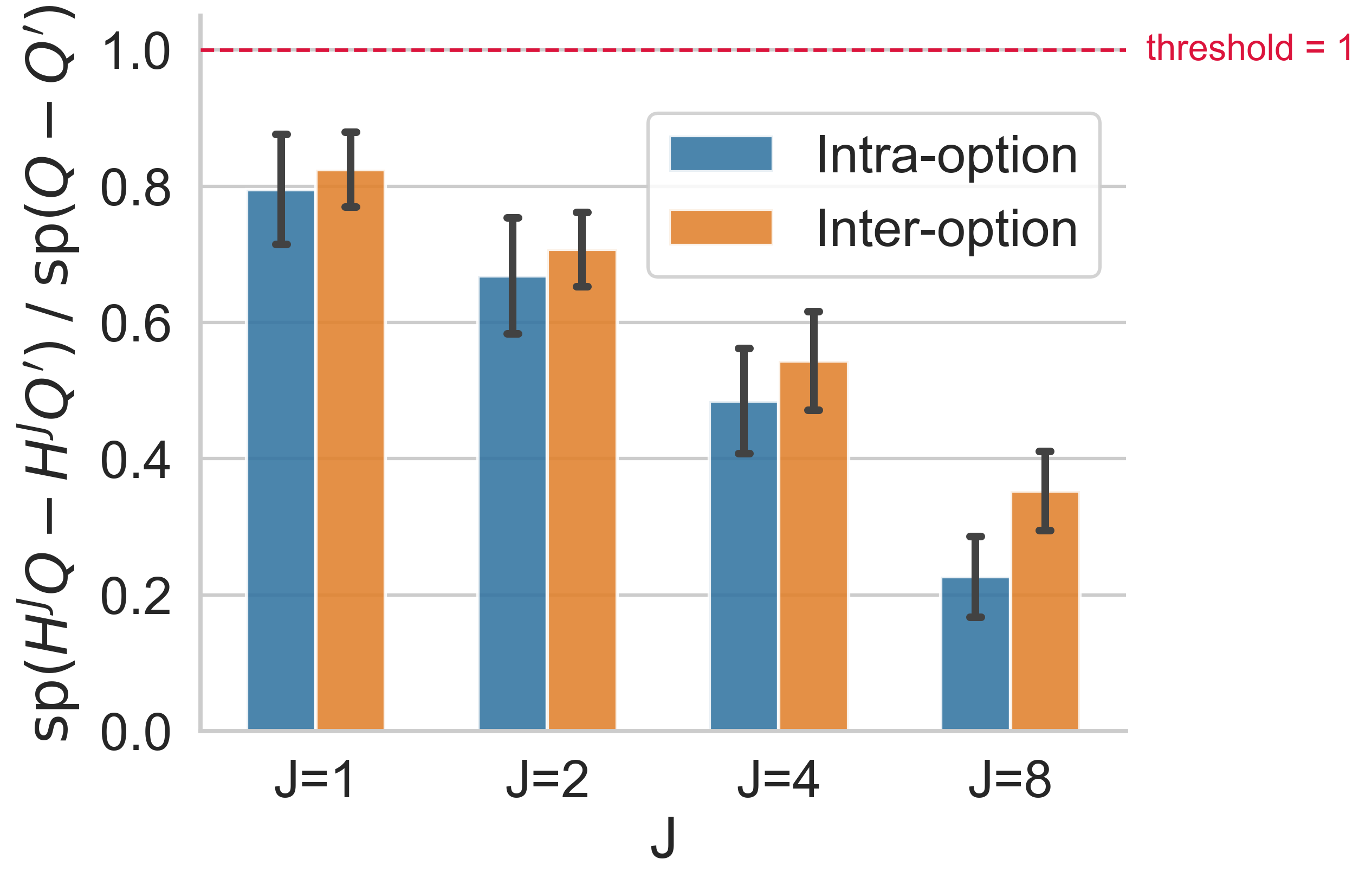}\vspace*{-1em}
        \else 
        \includegraphics[width=0.5\linewidth]{images/contraction.png}
        \fi
        \caption{Contraction factors ($\gamma$) for intra-option and inter-option in Four Rooms.}\label{fig:contraction}
\end{figure}

\begin{assumption}[Inter-option unichain SMDP]
\label{ass:inter-assumption}
For every deterministic option-selection rule $d:\cX\to\cO$, consider the
Markov chain on $\cX$ with transition kernel $
\widetilde P_d((s',o')\mid(s,o))
=
(1-\lambda(s,o))
\indicator \{(s',o')=(s,o)\}+
\lambda(s,o)\widehat p(s'\mid s,o)
\indicator\{o'=d(s', o)\}$.
We assume that $\widetilde P_d$ is unichain for every $d$. Moreover,
$\widehat p(s\mid s,o)>0,\lambda(s,o)<1$ for every $(s,o)\in\cX$.
\end{assumption}

Unichain assumption is widely adopted for analysing average-reward RL~\citep{wan2021average,chen2025nonasymptoticguaranteesaveragerewardqlearning,chen2025nonasymptotictheoryseminormlyapunov}. A deterministic stationary policy in a unichain MDP induces Markov chains with a single recurrent class ensuring well-definedness of average-reward and its invariance with respect to the initial state.


\begin{theorem}[Inter-option contraction]\label{thm: J-inter}
Let $\bar H_{\mathrm{inter}}^J$ denote $J$ compositions of $\bar H_{\mathrm{inter}}$. Under Assumption \ref{ass:inter-assumption}, there exists $1 \le J \le |\cS||\cO|$ and $\gamma_{\rm inter} \in (0,1)$ s.t.
\[
\mathrm{sp}(\bar H_{\mathrm{inter}}^J Q - \bar H_{\mathrm{inter}}^J Q')
\le
\gamma_{\rm inter} \, \mathrm{sp}(Q - Q')\,.
\]
\end{theorem}

While Assumption \ref{ass:inter-assumption} concerns the option-level SMDP where decisions are made only at option termination times, the intra-option operator concerns the primitive-timescale intra-option dynamics on $\cS\times\cO$, where the active option evolves at every primitive step through the continuation/termination mechanism. In particular, unichain properties need not transfer between these two constructions.
\begin{assumption}[Intra-Option unichain MDP]\label{ass:intra-assump}
For every deterministic option-selection rule $d:\cX\to\cO$, consider the
Markov chain on $\cX$ with transition kernel $
P_d((s',o')\mid(s,o))
=
p(s'\mid s,o)
\Big[
(1-\beta(s',o))\indicator\{o'=o\}+
\beta(s',o)\indicator\{o'=d(s',o)\}
\Big]$.
We assume that $P_d$ is unichain for every $d$. Moreover,
$p(s\mid s,o)>0, \beta(s,o)<1$ for every $(s,o)\in\cX$.
\end{assumption}

\begin{theorem}[Intra-option contraction]\label{thm: J-intra}
Under Assumption~\ref{ass:intra-assump}, there exists $1 \le J \le |\cS||\cO|$ and $\gamma_{\rm intra} \in (0,1)$ such that
\[
\mathrm{sp}(H_{\mathrm{intra}}^J Q - H_{\mathrm{intra}}^J Q')
\le
\gamma_{\rm intra} \, \mathrm{sp}(Q - Q').
\]
\end{theorem}\vspace*{-.5em}

These results extend the span contraction of average-reward MDPs~\citep{mdp_book} to options. 
The proofs are deferred to Appendix~\ref{sec:J-step-contraction}. The strict positivity/self-loop conditions in Assumptions~\ref{ass:inter-assumption}, \ref{ass:intra-assump} can be weakened via aperiodicity transform (Appendix~\ref{sec:J-step-contraction}). We empirically demonstrate the contractions in Figure~\ref{fig:contraction}, and write $\gamma$ for either
$\gamma_{\rm inter}$ or $\gamma_{\rm intra}$ whenever the setting is clear.

\section{\MakeUppercase{Algorithm Design: $J$-step Inter- and Intra-option Q-learning}}\label{sec:algo}

Enabled with the J-step span contractions, we design estimators of $J$-step inter- and intra-option Bellman operators that further yield two variants of option Q-Learning. We then prove the sublinear convergence of these option Q-learning algorithms (Table~\ref{tab: ftb}).

\textbf{The Premise of Synchronous Q-learning.}
For the $J$-step option Q-functions, the optimal option Q-function is given by the fixed points of $\sp({H}^J(Q) - Q) = 0$. Leveraging the stochastic approximation framework to find the fixed points of an operator~\citep{borkar2008stochastic}, 
we perform the synchronous updates for all the state-options $(s,o)$ simultaneously as
\ifdoublecol \begin{align}\label{eq:sync_update}
&Q_{k+1}(s,o)\notag\\
&= Q_k(s,o) + \alpha_k \Big( \widehat{T}^J(Q_k)(s,o) - Q_k(s,o) \Big)\,,   
\end{align}
\else 
\begin{align}\label{eq:sync_update}
Q_{k+1}(s,o)= Q_k(s,o) + \alpha_k \Big( \widehat{T}^J(Q_k)(s,o) - Q_k(s,o) \Big)\,,   
\end{align}
where $\widehat{T}^J(Q_k)$ is a stochastic estimator of (inter-/intra-option) Bellman operator $H^J(Q_k)$, and $\alpha_k >0$ is the learning rate. Note that, the proposed synchronous algorithms do not require an explicit centering step, unlike Relative Value Iteration (RVI)~\citep{SCHWEITZER1971495,mdp_book}. Throughout, by the \emph{empirical Bellman operator} we mean the
corresponding Bellman operator with the unknown transition kernel replaced
by its empirical estimate constructed from past interactions.

\subsection{{Inter-option Q-learning}}


At every iteration $k$, for each initial state-option pair
$x \triangleq(s,o)$, we generate a fresh $J$-step trajectory $\{R_k^{(i)}, L_k^{(i)}\}_{i=0}^{J-1}$ via a procedure we will outline below.
Using all trajectories observed before iteration $k$, we maintain the following
rolling estimates of $\rho, \lambda:$
\ifdoublecol\begin{align*}
\widehat\lambda_k(s,o)
&\triangleq
\frac{1+k}
{1+\sum_{t=0}^{k-1}L_t^{(0)}(s,o)},\\
\text{and}\quad \widehat\rho_k(s,o)
&\triangleq
\frac{\sum_{t=0}^{k-1}R_t^{(0)}(s,o)}
{1+\sum_{t=0}^{k-1}L_t^{(0)}(s,o)}.
\end{align*}
\else 
\begin{align*}
\widehat\lambda_k(s,o)
\triangleq
\frac{1+k}
{1+\sum_{t=0}^{k-1}L_t^{(0)}(s,o)},\quad \widehat\rho_k(s,o)
\triangleq
\frac{\sum_{t=0}^{k-1}R_t^{(0)}(s,o)}
{1+\sum_{t=0}^{k-1}L_t^{(0)}(s,o)}.
\end{align*}
These are used to estimate the normalized inter-option Bellman operator $H_k$, which is obtained by replacing $\rho, \lambda$ with $\hat \rho_k, \hat \lambda_k$ in Eq.~\eqref{eq:normalised_inter_operator}. 

Now, we outline how the J-step trajectory is obtained. For each $X_k^{(0)} = (s,o)$, at step $i$ of the $J$-step rollout, given
$X_k^{(i)}=(S_k^{(i)},O_k^{(i)})$, we independently sample an option execution
$(\widetilde R_k^{(i)},L_k^{(i)},\bar S_k^{(i+1)}) \sim \hat p(\cdot, \cdot, \cdot | X_k^{(i)})$, and draw
$D_k^{(i+1)}\sim\operatorname{Bernoulli}(\widehat\lambda_k(X_k^{(i)}))$. If $D_k^{(i+1)}=0$, we set
$X_k^{(i+1)}=X_k^{(i)}$.
If $D_k^{(i+1)}=1$, we set $X_k^{(i+1)} = \bigl(\bar S_k^{(i+1)}, \arg\max_{o'\in\cO}
\bigl(
\widehat H_{\rm k, inter}^{\,J-i-1}Q_k
\bigr)(\bar S_k^{(i+1)},o') \bigr)$, where $\widehat H_{\rm k, inter}$ is the empirical Bellman operator corresponding to $H_k$. Thus, $D_k^{(i+1)}$ stochastically realizes the mixture induced by
$\widehat\lambda_k$ in the normalized inter-option Bellman
operator~\eqref{eq:normalised_inter_operator}: $D_k^{(i+1)}=0$ produces the self-loop
term and $D_k^{(i+1)}=1$ produces the greedy-transition term. Using the resulting option-level trajectory
$\{X_k^{(i)},R_k^{(i)},L_k^{(i)}\}_{i=0}^{J-1}$, we construct the estimator of the $J$-step
normalized inter-option Bellman operator 
\ifdoublecol
\begin{align}
(\widehat T_{\mathrm{k, inter}}^{J}Q)(x)
\triangleq\;&
\sum_{i=0}^{J-1}
\widehat\rho_k(X_k^{(i)})+\bigl(1-D_k^{(J)}\bigr)
Q(X_k^{(J-1)})
\notag\\
&+
D_k^{(J)}
\max_{o'\in\cO}Q(S_k^{(J)},o').
\label{eq:sync_inter}
\end{align}
\else
\begin{align}
(\widehat T_{\mathrm{k, inter}}^{J}Q)(x)
\triangleq\;
\sum_{i=0}^{J-1}
\widehat\rho_k(X_k^{(i)})+\bigl(1-D_k^{(J)}\bigr)
Q(X_k^{(J-1)})+
D_k^{(J)}
\max_{o'\in\cO}Q(S_k^{(J)},o').
\label{eq:sync_inter}
\end{align}

for all $x$. Then, we use a learning rate $\alpha_k(x)=\frac{\alpha}{k+h}$ for all state-options to run the Q-learning updates as per Eq.~\eqref{eq:sync_update}. We present the pseudocode in Algorithm~\ref{alg:sync-inter}.

\subsection{{Intra-option Q-learning}}


In the intra-option setting, we instead take primitive actions using the action-level policy $A \sim \pi(\cdot | X)$ from each state-option $X$, and construct the estimator of the intra-option Bellman operator.

At every iteration $k$, for every initial state-option pair $X_k^{(0)} = (s,o)$, we generate $J$-step trajectories of primitive actions $A_k^{(i)} \sim \pi(\cdot\mid X_k^{(i)})$, reward-transitions $
(R_k^{(i)},S_k^{(i+1)})\sim p(\cdot,\cdot\mid S_k^{(i)},A_k^{(i)})$, and option terminations $B_k^{(i+1)}\sim \mathrm{Bernoulli}\bigl(\beta(S_k^{(i+1)},O_k^{(i)})\bigr)$. $B_k^{(i+1)} = 0$ implies $O_k^{(i+1)} = O_k^{(i)}$. $B_k^{(i+1)}=1$ indicates termination of the $i$-th option and the next option is chosen according to $ O_k^{(i+1)}\in \arg\max_{o'\in O}\bigl(\widehat H_{\rm k, intra}^{\,J-i-1}Q_k\bigr)(S_k^{(i+1)},o')$, where $\widehat H_{\rm k, intra}$ denotes the empirical Bellman operator associated with  $H_{\rm intra}$ at $k$-th iteration. Thus, $B_k^{(i+1)}$ stochastically realizes the mixture induced by
$\beta(S_k^{(i+1)}, \cdot)$ in the intra-option Bellman
operator~\eqref{eq:intra-operator}: $B_k^{(i+1)}=0$ produces the self-loop
term and $B_k^{(i+1)}=1$ produces the greedy-transition term for next option selection. Using this data, we construct an estimator of $J$-step intra-option Bellman operator
\ifdoublecol
\begin{align}\label{eq:sync_intra}
(\widehat T^{\,J}_{\mathrm{k, intra}}~Q)(x)
&\triangleq \sum_{i=0}^{J-1} R_k^{(i+1)} + (1-B_k^{(J)})Q(S_k^{(J)},O_k^{(J-1)})\notag\\
&\quad+ B_k^{(J)} \max_{o'\in O} Q(S_k^{(J)},o').
\end{align}
\else 
\begin{align}\label{eq:sync_intra}
(\widehat T^{\,J}_{\mathrm{k, intra}}~Q)(x)
\triangleq \sum_{i=0}^{J-1} R_k^{(i+1)} + (1-B_k^{(J)})Q(S_k^{(J)},O_k^{(J-1)})+ B_k^{(J)} \max_{o'\in O} Q(S_k^{(J)},o').
\end{align}

for all $x=(s,o)$. Then, we use a learning rate $\alpha_k(x)=\frac{\alpha}{k+h}$ for all state-options to run the Q-learning updates as per Eq.~\eqref{eq:sync_update}. The pseudocode is described in Algo.~\ref{alg:sync-intra}.

\section{\MakeUppercase{Convergence Analysis}}

Now, we derive suitable learning rates and finite-time convergence guarantees for the two proposed option Q-learning algorithms. 
We guarantee convergence of the Q-values only up to an additive constant, i.e., set-wise convergence, to the optimality class $\cQ^\ast$ of the Average Optimality Equation \eqref{eq:aoe}.



\begin{theorem}[\textbf{Inter-option $Q$-learning}]
\label{thm:sync-inter-ftb}
If Assumption~\ref{assm:smdp}-\ref{ass:inter-assumption} hold true, and we set
$
\alpha_k=\frac{4}{(1-\gamma)(k+h)}
$ for some $h>0$. Then, $\Delta_k^2
\defn \E\!\left[\sp(Q_k-Q^*)^2\right] =$
\begin{align*}
\hspace{-1em}\widetilde \cO\!\left(
\frac{1}{k}
\left(
\Delta_0^2 + \left( \sp(Q^*)^2 + |\cS||\cO| c^2 J^4 + |\cS|J^4\right)
\frac{
\log(|\cS||\cO|)
}{
(1-\gamma)^3
}\right)
\right)
\end{align*}
\end{theorem}

\begin{theorem}[\textbf{Intra-option Q-learning}]\label{thm: sync-intra-ftb}
If Assumption~\ref{assm:smdp} and~\ref{ass:intra-assump} hold true, then setting the stepsizes as $\alpha_k = \frac{4}{(1-\gamma)(k+h)}$ for some $h>0$, yields
  \begin{align*}
      \Delta_k^2 = \widetilde\cO\left(\frac1k \left( \Delta_0^2 + \bigl(\sp(Q^*)^2+|\cS| J^4\bigr) \frac{\log |\cS||\cO|}{(1-\gamma)^3}\right) \right)\,.
  \end{align*}
\end{theorem}

As a result, both algorithms achieve sublinear decay of the mean-square span error: intra-option $Q$-learning achieves $\widetilde \cO(1/k)$, while inter-option $Q$-learning achieves $\widetilde{\cO}(1/k)$ with additional errors from estimating the reward rates and inverse holding times online.
In addition, these bounds imply a sample complexity of $\widetilde{\cO}(1/\varepsilon^2)$ to guarantee $\expect{\sp(Q_k-Q^*)}\le\varepsilon$ (Corollaries~\ref{cor:sample-complexity-intra} and~\ref{cor:sync_inter_new_rate}), 
consistent with prior finite-time analyses of average-reward $Q$-learning with primitive actions~\citep{zhang2021finite,chen2025nonasymptotictheoryseminormlyapunov}.


\textbf{Inter- vs. Intra-option.}
The additional $|\cS||\cO|c^2J^4$ term in the inter-option bound
(Theorem~\ref{thm:sync-inter-ftb}) comes from online estimation of the
reward rates and inverse mean holding times $(\rho,\lambda)$, and hence
depends on the holding-time second-moment bound $c$. 
No such estimation is required in the intra-option setting.

\textbf{Proof sketch.}
Our analysis builds on the generalized Moreau-envelope framework of
\cite{chen2025nonasymptotictheoryseminormlyapunov}, with a key difference: our $J$-step estimators are not conditionally
unbiased estimators of $H^JQ_k$, where $H$ denotes $\bar H_{\rm inter}$ or $H_{\rm intra}$. An unbiased $J$-step estimator would require
intermediate option selections based on finite-horizon values
$H^mQ_k, m \in \{1, 2,\dots, J-1\}$, because at an intermediate state with
$m$ steps remaining, the maximization in the nested composition $H^JQ_k$
is over $H^mQ_k$, rather than over $Q_k$ itself. However, these are unavailable in a model-free setting. We instead sample the intermediate options greedily from $\widehat H_k^mQ_k$. Our proof, therefore, has to combine the $J$-step span contraction with a careful control of the
bias induced by empirical approximately greedy option selection. The proofs are established via the following steps.

\begin{enumerate}[leftmargin=*,itemsep=2pt,topsep=1pt]
\item \textbf{Moreau-envelope Lyapunov recursion.}
We first use the $J$-step span contraction (Theorem~\ref{thm: J-inter} and~\ref{thm: J-intra}) to construct a smooth generalized Moreau-envelope Lyapunov function
$M(Q)=p_m(Q)^2/2$ equivalent to $\sp(Q)^2$ modulo constant vectors. By denoting
$m_k \triangleq \E[M(Q_k-Q^*)]$, smoothness yields a one-step recursion of the form
\ifdoublecol
\begin{align*}
m_{k+1}
\le
m_k
+
\alpha_k
\bigl(
T_{\rm drift}
+
T_{\rm noise}
+&
T_{\rm greedy}
+
T_{\rm est}
\bigr)
\\&+L\alpha_k^2T_{\rm quad},
\end{align*}
\else
\begin{align*}
m_{k+1}
\le
m_k
+
\alpha_k
\bigl(
T_{\rm drift}
+
T_{\rm noise}
+
T_{\rm greedy}
+
T_{\rm est}
\bigr)+L\alpha_k^2T_{\rm quad},
\end{align*}
where $T_{\rm est}$ appears only in the inter-option setting. The
$J$-step contraction implies $T_{\rm drift}\le-cm_k$ for $c>0$,
providing the negative drift in the recursion.

\item \textbf{Removing Bellman-oracle access.}
A direct $J$-step estimator would require the intermediate greedy choice
$\argmax_o(H^m Q_k)(s,o)$, whenever $m$ steps remain. 
Since these values are unavailable, we construct an empirical Bellman operator $\widehat H_k$ from past interactions, and use
$\widehat d_{k,m}(s)\in\arg\max_o(\widehat H_k^mQ_k)(s,o)$. The fresh rollout
still uses the true environment, so its expectation, denoted by $\widetilde H_k^JQ_k$, is neither
$\widehat H_k^JQ_k$ nor $H^JQ_k$. This yields the first-order perturbation
$T_{\rm greedy}
=
\E[\langle\nabla M(Q_k-Q^*),
\widetilde H_k^JQ_k-H^JQ_k\rangle]$.

\item \textbf{Controlling the approximately greedy bias.}
We first prove that the iterates a.s. satisfy $\|Q_k\|_\infty=\cO(\log k)$. Using this bound together with finite-sample
deviation bounds for the empirical Bellman operator gives (Lemma~\ref{lem:empirical-greedy-J}, \ref{lem:inter-finite-j-greedy}):\vspace*{-.2em}
\begin{align*}
\E\left[
\sp\!\left(
\widetilde H_k^JQ_k-H^JQ_k
\right)^2
\right]
=
\widetilde\cO(1/k).
\end{align*}
We use properties of Moreau envelope to
absorb a part of $T_{\rm greedy}$ into $T_{\rm drift}$, while remainder contributes $\widetilde{\cO}(k^{-2})$ after multiplied with $\alpha_k=\Theta(1/k)$.

\item \textbf{Additional bias in inter-option.}
The normalized inter-option Bellman operator introduces a second source of bias due to the unknown reward rate
$\rho$ and inverse mean holding time $\lambda$. We estimate them \textit{online} by
$\widehat\rho_k$ and $\widehat\lambda_k$ yielding finite-time
mean-square bounds of order $\widetilde\cO(1/k)$. Propagating these errors
through the finite $J$-step Bellman composition yields the additional
first-order term $T_{\rm est}$. Conceptually, the inter-option
conditional-mean error decomposes as
$\widetilde H_k^JQ_k-H^JQ_k
=
\bigl(
\widetilde H_k^JQ_k-\bar H_k^JQ_k
\bigr)
+
\bigl(
\bar H_k^JQ_k-H^JQ_k
\bigr),
$
where the first term is due to approximate greedy option selection and the
second is due to online estimation of $\rho$ and $\lambda$.\textit{ The intra-option
setting has only the first source of bias.}

\item \textbf{Rollout noise and the quadratic term.}
After conditioning on the past, the difference between the sampled
$J$-step target and its expectation is a martingale-difference term.
Hence, the linear noise term vanishes. The second-order smoothness remainder is shown
to be $L\alpha_k^2T_{\rm quad}={\cO}(\alpha_k^2m_k)$.

\item \textbf{Unrolling the recursion.}
Combining the preceding bounds gives, for all $k$,
\begin{align*}
m_{k+1}
\le
\left(
1-\cO(\alpha_k)
\right)m_k
+
\widetilde{\cO}(\alpha_k^2).
\end{align*}
Unrolling the recursion yields the stated
$\widetilde\cO(1/k)$ convergence rate.
\end{enumerate}\vspace*{-.5em}
The complete proofs are in Appendix~\ref{sec:synchronous}. 

\textbf{Comparison with asymptotic analysis.}
\cite{wan2021average,wan2024convergenceaveragerewardqlearningweakly} conduct asymptotic analysis of their proposed 1-step inter- and intra-option differential Q-learning algorithms with RVI-type recentering. The analysis differs in a couple of ways. First, they do not need to construct a $J$-step estimator of the option Bellman operators since the contraction property of the Bellman operators is not required in the asymptotic analysis. Rather, asymptotic analysis requires only that the option Bellman operators are Lipschitz-- a property which the $1$-step option Bellman operators satisfy, and hence, they use the $1$-step option Bellman operators. Second, their analysis for their inter-option algorithm employs a coupled Stochastic Approximation (SA) scheme to estimate option lengths whilst simultaneously updating the Q-values. This results in a single-timescale SA with vanishing perturbations. 
This is similar to our derivation of finite-time rates in both the settings, which are modelled as single-timescale SA with decaying perturbations.




\textbf{Rectifying $J$-step estimators in prior work.} \label{rem:prior-J-step-estimators}
Prior finite-time analyses of average-reward Q-learning with primitive
actions \citep{chen2025nonasymptotictheoryseminormlyapunov, zhang2021finite, naskar2026parameter} use a $J$-step estimator of
the form
$
\widehat H_J(Q_k)
=
\sum_{i=0}^{J-1}R_{k+i}
+
\max_{a'}Q_k(S_{k+J},a'),$
with intermediate actions selected greedily from $Q_k$. This estimator is
claimed to be unbiased for the $J$-step Bellman operator. But in general, its
conditional expectation is not $H^JQ_k$. Indeed, the $J$-fold Bellman
composition requires the intermediate action with $m$ steps remaining to
be selected greedily with respect to $H^mQ_k$, rather
than $Q_k$ itself. 

Consider the following zero-reward MDP: let $s_0\xrightarrow{a_0}s_1$, with
$a,b$ at $s_1$ leading deterministically to $x,y$, respectively, and let $x,y$
have only self-loops via $c$. Choose $Q(s_1,a)=1, \enspace  Q(s_1,b)=0, \enspace Q(x,c)=0, \enspace Q(y,c)=10.$
Then $(HQ)(s_1,a)=0$ and $(HQ)(s_1,b)=10$, so
$(H^2Q)(s_0,a_0)=10$. In contrast, an estimator that selects the intermediate
action greedily from $Q$ chooses $a$ instead of $b$ and returns $0$. Thus $\widehat H_J$ is
biased (here $J=2$), and the intermediate greedy choice must generally use $H^mQ$ rather
than $Q$.
Consequently, the corresponding finite-time analyses based on this
unbiasedness property do not directly apply as stated. 

\textit{Our construction rectifies this issue} by replacing the unavailable greedy values with
empirical Bellman iterates and explicitly controlling the resulting bias.
This yields corrected $J$-step estimators and corresponding finite-time rates for average-reward Q-learning with primitive actions, while also extending the methodology to inter- and intra-option Q-learning. Notably, specializing our analysis to average-reward Q-learning with primitive
actions~\citep{chen2025nonasymptotictheoryseminormlyapunov, zhang2021finite}, the harmonic-stepsize rate becomes
$\widetilde{\cO}(1/k)$ rather than $\cO(1/k)$, with an additional
$\widetilde{\cO}(|\cS|J^4/k)$ contribution arising from the empirical
approximately-greedy selections. 


\begin{figure}[t!]
\centering
\ifdoublecol
\includegraphics[width=0.8\linewidth]{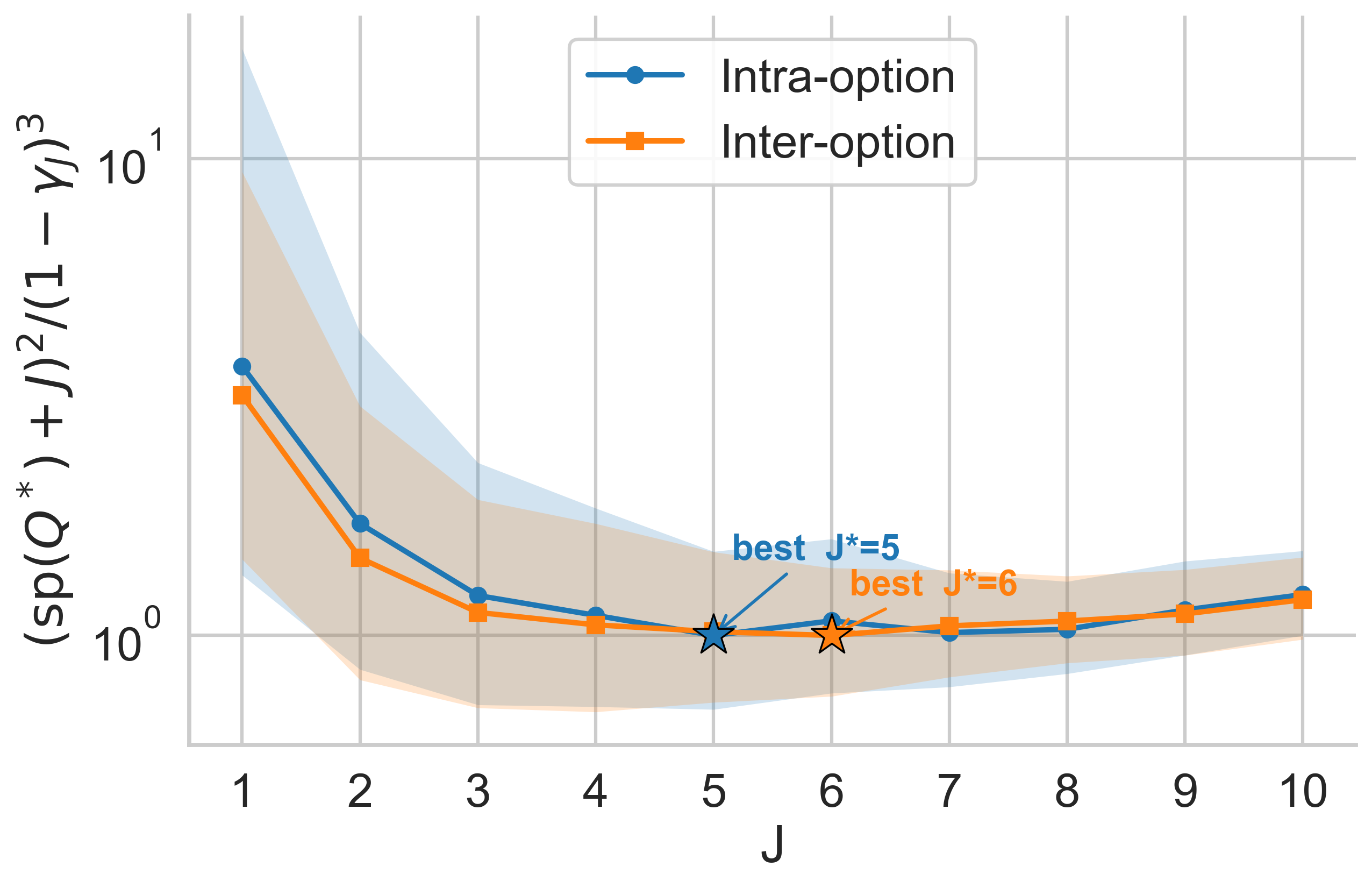}\vspace*{-.5em}
\else 
\includegraphics[width=0.5\linewidth]{images/4room/j_gamma_tradeoff/c_finite_time_bound.png}
\fi
\caption{Guideline to select $J$ with $\frac{(\sp(Q^*)+J)^2}{(1-\gamma)^3}$.}\label{fig:j-gamma-sync}
\end{figure}



\textbf{Guideline to Select $J$.}
The dependence of $J,\gamma$ on the performance bounds can be inferred from
Theorem~\ref{thm:sync-inter-ftb} and~\ref{thm: sync-intra-ftb}.
Specifically, $\frac{J}{1-\gamma(J)}$ reflects the effect of the
choice of $J$ and corresponding $\gamma$ on the finite-time rates.
We empirically investigate this connection (Figure~\ref{fig:j-gamma-sync}) in \textit{Four Rooms}. We observe that there are optimal $J$'s, namely 6 and 5 for inter- and intra-option settings, by tracking the quantity $\frac{(\sp(Q^*)+J)^2}{(1-\gamma)^3}$, since it captures the trade-off between stronger $J$-step contraction and the rollout-length contribution, while excluding the conservative empirical-model estimation term (Appendix~\ref{app:guideline}). Thus, instead of the worst-case bound of $|\cS||\cO|$, we can use small values of $J$ in practice, and achieve the performance gain.

\textbf{Choices of stepsizes.} Our analysis uses the harmonic stepsize $\alpha_k=\Theta(1/k)$, for which
the pathwise growth bound $\|Q_k\|_\infty=\cO(\log k)$ leads to the final
mean-square rate $\widetilde{\cO}(1/k)$. The same proof framework can be
extended to yield guarantees for polynomial stepsizes. In particular, for
$\alpha_k=\alpha/(k+h)^p$ with $p\in(1/2,1)$, the pathwise argument gives
$\|Q_k\|_\infty=\cO(k^{1-p})$. Consequently, the empirical
approximately-greedy perturbation scales as
$\widetilde{\cO}(k^{1-2p})$, implying a corresponding mean-square rate
of $\widetilde{\cO}(k^{1-2p})$. 


\section{\MakeUppercase{RL with Hierarchical Reward Machines}}
Prior works on average-reward HRL~\citep{JMLR:v8:ghavamzadeh07a} derive the average-reward Bellman equations and corresponding Q-learning algorithms, but do not establish convergence of the proposed updates, leaving a gap between the fixed-point characterization and the algorithmic behaviour. Our proposed algorithms are the first to exhibit non-asymptotic convergence to hierarchically optimal policies. Specifically, the analysis establishes convergence of the proposed algorithms to the solutions of \eqref{eq:aoe} with options, and closes this gap.

\textbf{Hierarchical Reward Machines (HRMs) to options.} 
RM reward is Markovian on $\mathcal{X}=S\times U$, where $U$ is RM's states.
Every transition in RM $\langle u,u_t\rangle$ induces an option that starts in state $u$ and terminates with a change in the RM state. Therefore, $\bar H_{\mathrm{inter}}$, $H_{\mathrm{intra}}$ on both $J$-step algorithms defined in Section~\ref{sec:algo} are consistent with the new state space $\mathcal{X}$. In this case, the option holding times are the hitting times of RM transitions which are not in general bounded. Our analysis allow this by only requiring their second moments to be bounded (Assumption~\ref{assm:smdp}).

\textbf{Finite-time guarantees.} 
Under Assumptions~\ref{ass:inter-assumption} and~\ref{ass:intra-assump} and the assumption on the cross-product space (Assumption~\ref{ass:rm-unichain}), both intra- and inter-option operators are $J$-step span contractions (Theorem~\ref{thm:rm-hrm-contraction}). Therefore the $\widetilde{O}(1/k)$ bounds from Theorem~\ref{thm:sync-inter-ftb} and~\ref{thm: sync-intra-ftb} hold but with $|S||\mathcal{O}|$ replaced by $|S||U||\mathcal{O}|$. The intra-option and inter-option bounds vary logarithmically and linearly respectively with RM size $|U|$.

\textbf{Hierarchical optimality.} 
Our analysis guarantees convergence of the Q-values, up to an additive constant, to the optimality class $\mathcal{Q}^*$ of \eqref{eq:aoe} for the option SMDP on $\mathcal{X}$. Under standard assumptions on options (e.g., non-zero probability of option termination, which in HRMs is equivalent to non-zero probability of satisfying the logical formula associated with an outgoing RM transition), the policies obtained by option Q-learning are hierarchically optimal with respect to the given option policies~\citep{JMLR:v8:ghavamzadeh07a}. Specifically, they are optimal among all policies that select between the HRM options. Optimality over all primitive policies requires the option policies to be optimal for their sub-tasks, which depends on the options rather than algorithms. 

A detailed reduction of MDPs with HRMs to options, and implications of our convergence guarantees is in Appendix~\ref{app:HRM}.

\vspace*{-.5em}\section{\MakeUppercase{Experimental Analysis}}\vspace*{-.5em}\label{sec:experiments_main}

\begin{figure}[t!]
    \centering
    \ifdoublecol
    \begin{subfigure}[b]{0.23\textwidth}
        \centering
        \includegraphics[width=\linewidth]{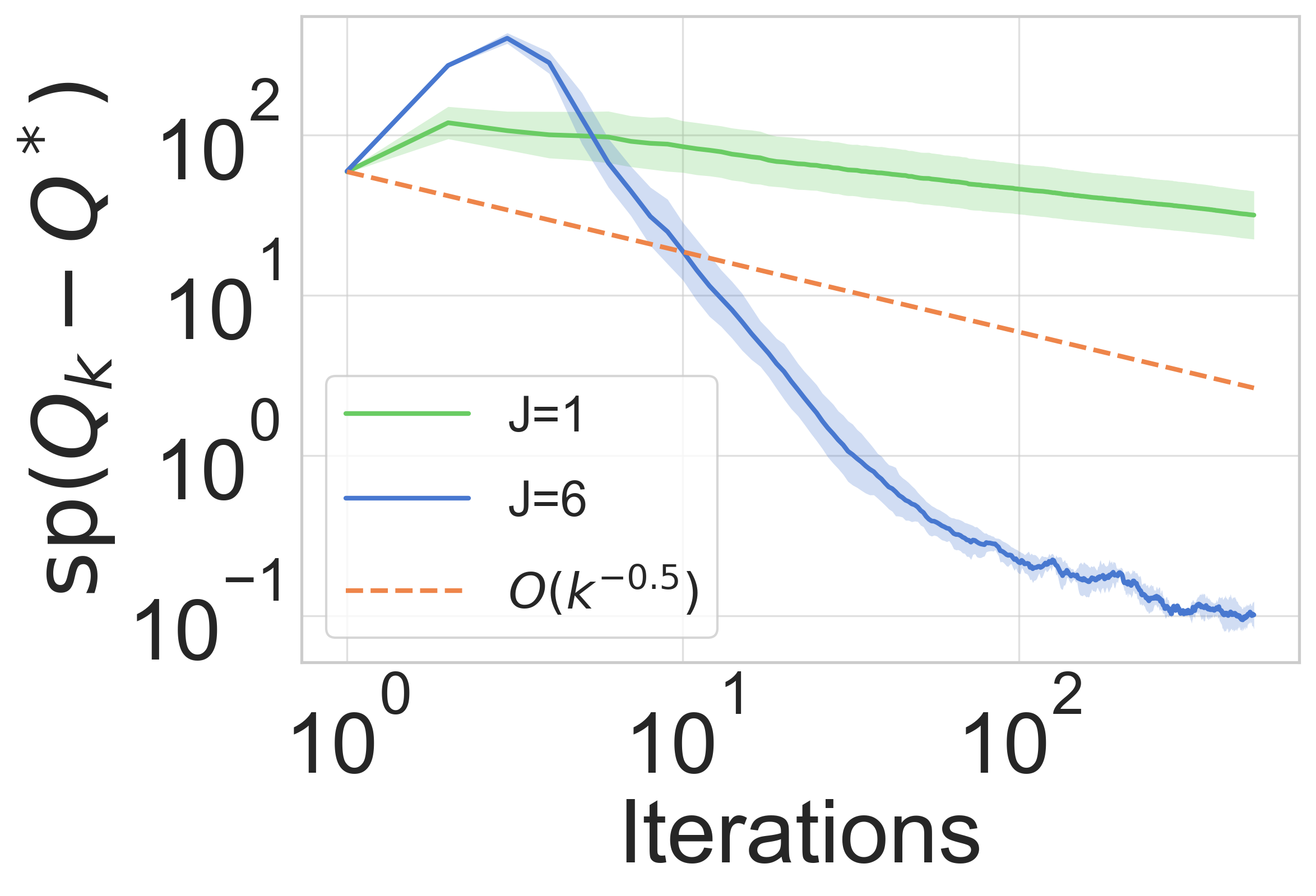}\vspace*{-.5em}
        \caption{Inter-option}
    \end{subfigure}
    \hfill
    \begin{subfigure}[b]{0.23\textwidth}
        \centering
        \includegraphics[width=\linewidth]{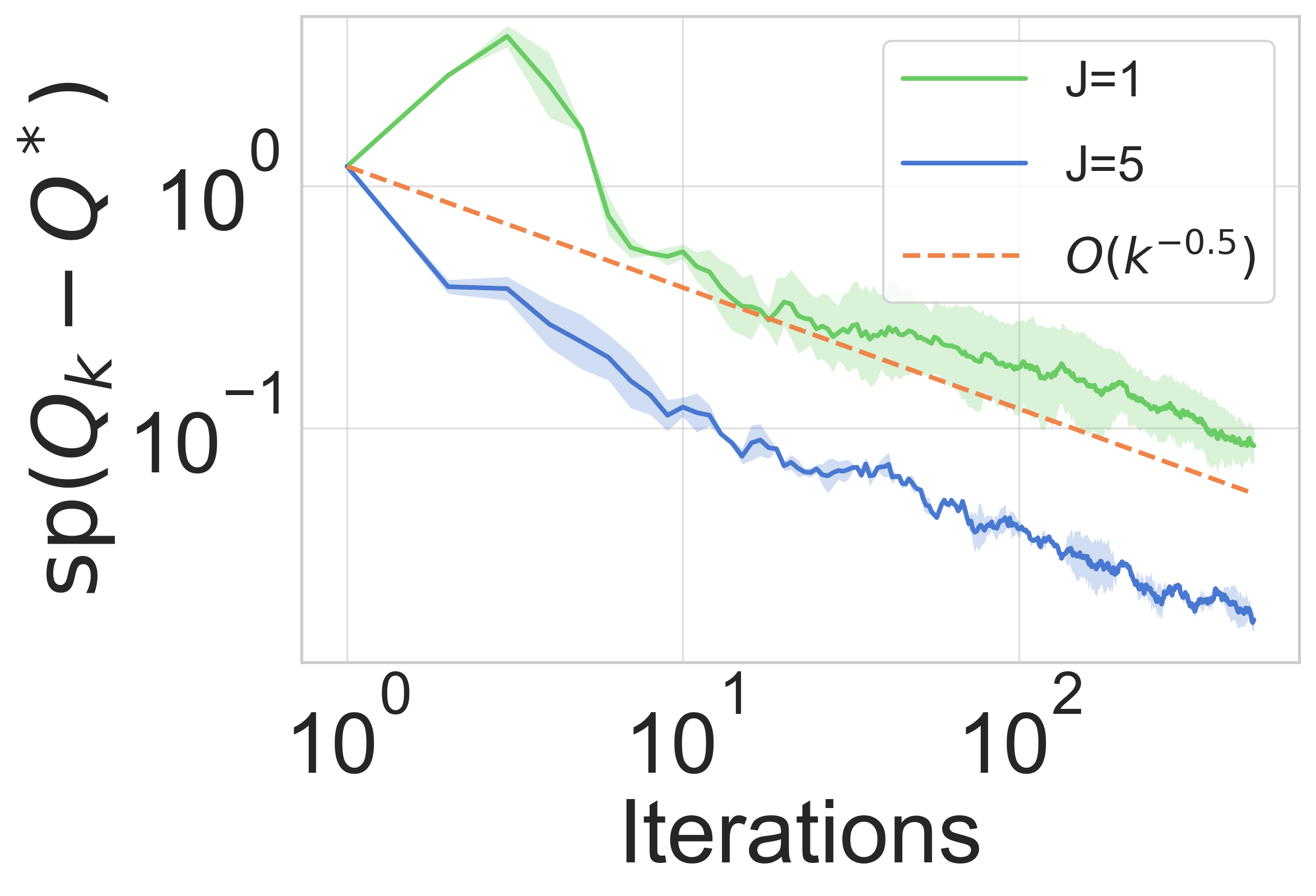}\vspace*{-.5em}
        \caption{Intra-option}
    \end{subfigure}
    \else 
    \begin{subfigure}[b]{0.4\textwidth}
        \centering
        \includegraphics[width=\linewidth]{images/fourrooms_inter_sync.png}\vspace*{-.5em}
        \caption{Inter-option}
    \end{subfigure}
    \hspace{2em}
    \begin{subfigure}[b]{0.4\textwidth}
        \centering
        \includegraphics[width=\linewidth]{images/fourrooms_intra_sync.png}\vspace*{-.5em}
        \caption{Intra-option}
    \end{subfigure}
    \fi 
    
    \caption{Evolution of $\mathrm{sp}(Q_k-Q^*)$ (mean $\pm$ std. over $5$ runs) with $k$ iterations for Four Rooms.}\label{fig:four_rooms_plots}
\end{figure}

We empirically validate our findings across three hierarchical reinforcement learning (HRL) benchmarks: \emph{Four Rooms}~\citep{SUTTON1999181}, \emph{Taxi}~\citep{taxi_diet}, and \emph{Office gridworld}~\citep{toro_reward_machines}. We report the mean square deviations from the optimal option Q-function over iterations. For Four Rooms, Figure~\ref{fig:four_rooms_plots} shows $\mathrm{sp}(Q_k - Q^\ast)$ on log-log axes for the 6-step inter- and 5-step intra-option algorithms, and the one-step Q-learning ($J=1$) over $k$ iterations (averaged over $5$ independent seeds). 
Results for \emph{Taxi} and \emph{Office gridworld} are in Appendix~\ref{sec:extra-experiments}, along with environment and hyper-parameter details.

\textbf{Observations.} (a) We observe that the $J=6$ and $J=5$ step estimators yield strictly \textbf{faster convergence} in inter- and intra-option settings than classic option Q-learnings with $J=1$.
(b) In both the settings, the empirical convergence rates match $\mathcal{O}(k^{-1/2})$ theoretical reference for the intra- and inter-option Q-learning, as stated by Theorem~\ref{thm:sync-inter-ftb} and~\ref{thm: sync-intra-ftb}.


\section{\MakeUppercase{Discussions}}\vspace*{-.5em}

We design two variants of option Q-learning algorithms. We derive the first non-asymptotic sublinear convergence for these option Q-learning algorithms. The proof proceeds by (a) $J$-step contraction for inter- and intra-option Q-learning, (b) designing $J$-step option Bellman operator estimates and Q-update steps with proper learning rates, (c) tackling the effect of only second moment bounded stochastic holding times for options, and (d) controlling the bias introduced by approximately greedy option selection and online estimation of reward rate and inverse mean holding times. Our results establish the first non-asymptotic convergence bound for Q-learning with hierarchical RMs.

We analyse the single timescale framework for synchronous option Q-learning. The natural next step is to analyse the asynchronous setting, where only a
subset of the Q-estimates are updated at each iteration. Another direction is
the analysis of two-timescale option Q-learning
~\citep{manenti2025convergencestabilityqlearninghierarchical}. Second, an important open question is whether the unichain assumptions can be relaxed to more general multichain settings for options, in line with recent developments in average-reward RL~\citep{lee2025optimal, Lee2026PolicyGA}.
Finally, understanding the role of function approximation in the options framework, in terms of stability and convergence, remains an open direction.

\ifdoublecol
\clearpage
\subsection*{AI use statement}

In this work, we used generative AI tools solely for proofreading the manuscript and for suggestions to improve the writing style and clarity.
We have \textbf{not} used generative AI tools to develop theoretical analysis or conceptual frameworks, formulate mathematical claims, provide critical ingredients for proving mathematical claims, assist in the writing of proofs, propose or refine hypotheses, design or provide feedback on research methodology or experiments. 
We take responsibility for the final content of this work.
\fi

\bibliography{cite}

\ifdoublecol
\section*{Checklist}

\begin{enumerate}

\item For all models and algorithms presented, check if you include:
\begin{enumerate}
    \item A clear description of the mathematical setting, assumptions,
    algorithm, and/or model. [Yes]

    \item An analysis of the properties and complexity (time, space,
    sample size) of any algorithm. [Yes]

    \item (Optional) Anonymized source code, with specification of all
    dependencies, including external libraries. [Yes, in Appendix~\ref{sec:extra-experiments}]
\end{enumerate}

\item For any theoretical claim, check if you include:
\begin{enumerate}
    \item Statements of the full set of assumptions of all theoretical
    results. [Yes]

    \item Complete proofs of all theoretical results.
    [Yes]

    \item Clear explanations of any assumptions.
    [Yes]
\end{enumerate}

\item For all figures and tables that present empirical results, check if
you include:
\begin{enumerate}
    \item The code, data, and instructions needed to reproduce the main
    experimental results. [Yes]

    \item All the training details.
    [Yes]

    \item A clear definition of the specific measure or statistics and
    error bars. [Yes]

    \item A description of the computing infrastructure used.
    [Yes]
\end{enumerate}
Check Appendix~\ref{sec:extra-experiments} and Section~\ref{sec:experiments_main} for details.

\item If you are using existing assets or curating/releasing new assets,
check if you include:
\begin{enumerate}
    \item Citations of the creator if your work uses existing assets.
    [Not Applicable]

    \item The license information of the assets, if applicable.
    [Not Applicable]

    \item New assets either in the supplemental material or as a URL,
    if applicable. [Not Applicable]

    \item Information about consent from data providers/curators.
    [Not Applicable]

    \item Discussion of sensitive content if applicable.
    [Not Applicable]
\end{enumerate}

\item If you used crowdsourcing or conducted research with human subjects,
check if you include:
\begin{enumerate}
    \item The full text of instructions given to participants and screenshots.
    [Not Applicable]

    \item Descriptions of potential participant risks, with links to
    Institutional Review Board (IRB) approvals if applicable.
    [Not Applicable]

    \item The estimated hourly wage paid to participants and the total
    amount spent on participant compensation.
    [Not Applicable]
\end{enumerate}
\end{enumerate}
\fi

\clearpage
\appendix
\ifdoublecol
\thispagestyle{empty}

\onecolumn

\aistatstitle{
Non-asymptotic Convergence of Average-reward Q-learning with Options: 
Supplementary Materials
}
\part{}
\parttoc
\newpage
\else 
\newpage\part{APPENDIX}
\fi

\section*{\MakeUppercase{Overview of the Appendix}}
\begin{itemize}[leftmargin=*]
    \item \textbf{Appendix \ref{sec:related_work}: Additional Related Work.}
    Reviews prior literature on average-reward Q-learning, stochastic approximation, hierarchical reinforcement learning with options, and finite-time analyses for seminorm-contractive stochastic approximation schemes.

    \item \textbf{Appendix \ref{sec:J-step-contraction}: J-step contraction results for option Bellman operators.}
    Establishes Puterman-style span seminorm contractions for the inter- and intra-option Bellman operators proving Theorems \ref{thm: J-inter} and \ref{thm: J-intra}. 

    \item \textbf{Appendix \ref{sec:synchronous}: Analysis of synchronous option Q-learning.}
    Introduces the synchronous J-step estimators and develops the finite-time analysis of synchronous inter- and intra-option Q-learning. The section constructs the Lyapunov framework based on a smooth Moreau-envelope approximation of the span seminorm and derives the corresponding non-asymptotic convergence rates proving Theorems \ref{thm:sync-inter-ftb}, \ref{thm: sync-intra-ftb}.




    \item \textbf{Appendix \ref{app:HRM}: Application to average-reward reinforcement learning with reward machines.}
    Shows how hierarchical reward machines induce options, verifies the required assumptions, establishes the corresponding J-step contractions, and derives finite-time convergence guarantees for average-reward hierarchical reinforcement learning.

    \item \textbf{Appendix \ref{sec:extra-experiments}: Additional experimental results.}
    Provides supplementary experiments, implementation details, hyperparameter settings, and additional plots for the Four Rooms, Taxi, and Office environments.
\end{itemize}

\newpage

\section{TABLE OF NOTATIONS}

\begin{table}[h!]
\centering
\caption{Notation for the average-reward option-learning setup and algorithms.}
\label{tab:notation_setup}
\setlength{\tabcolsep}{6pt}
\renewcommand{\arraystretch}{1.15}
\begin{tabular}{l|p{0.78\linewidth}}
\toprule
\textbf{Symbol} & \textbf{Meaning} \\
\midrule

$\cS,\cA,\cO$
& Finite state, primitive-action, and option spaces. \\

$s,a,o$
& Generic state, primitive action, and option. \\

$S_k,A_k,O_k$
& State, primitive action, and option at iteration $k$. \\

$p(s'|s,a)$
& Primitive MDP transition kernel. \\

$r(s,a)$
& Primitive one-step reward. \\

$\pi(a|s,o)$
& Intra-option policy associated with option $o$. \\

$\beta(s,o)$
& Termination probability of option $o$ at state $s$. \\

$\ell(s,o)$
& Random holding time of option $o$ initiated at state $s$. \\

$c$
& Uniform holding-time second-moment parameter:
$\E[\ell(s,o)^2]\le c^2$. \\

$\hat r(s,o)$
& Expected cumulative reward over one execution of option $o$. \\

$\hat\ell(s,o)$
& Expected holding time of option $o$. \\

$\hat p(s'|s,o)$
& Option-level transition kernel. \\

$\bar r^\pi,\bar r^*$
& Average reward rate under policy $\pi$ and the optimal average reward rate. \\

$Q,Q^*,Q_k$
& Option $Q$-function, optimal differential $Q$-function, and $k$-th iterate. \\

$\mathbf 1$
& All-ones vector. \\

$\indicator$ & Indicator random variable \\

$\sp(Q)$
& Span seminorm,
$\sp(Q):=(\max_{s,o}Q(s,o)-\min_{s,o}Q(s,o))/2$. \\

$\Delta_0$
& Initial span error, $\Delta_0:=\sp(Q_0-Q^*)$. \\

$\bar H_{\rm inter},H_{\rm intra}$
& Normalized inter-option and intra-option Bellman operators. \\

$H^J$
& $J$-fold composition of a Bellman operator. \\

$J$
& Multi-step Bellman composition / rollout horizon. \\

$\gamma_{\rm inter}, \gamma_{\rm intra}$
& Contraction factor of the $J$-step Bellman operator in span seminorm. \\

$\rho(s,o)$
& Normalized option reward rate,
$\rho(s,o):=\hat r(s,o)/\hat\ell(s,o)$. \\

$\lambda(s,o)$
& Inverse mean holding time,
$\lambda(s,o):=1/\hat\ell(s,o)$. \\

$\hat\rho_k,\hat\lambda_k$
& Rolling estimates of $\rho$ and $\lambda$. \\

$A_k,B_k$
& Rolling cumulative reward and holding-time statistics used to construct
$\hat\rho_k$ and $\hat\lambda_k$. \\

$\alpha_k,\alpha$
& Stepsize and stepsize-scale parameter. \\

$h$
& Positive offset/regularization parameter in the stepsize. \\

$N_k(s,o)$
& Number of activations/visits of coordinate $(s,o)$ up to iteration $k$. \\
$H_k$ & Approximate normalized inter-option Bellman operator obtained by replacing $\rho, \lambda$ with $\widehat \rho_k, \widehat \lambda_k$ in~\eqref{eq:normalised_inter_operator}.\\
$\widehat H_{\rm k, inter}, \widehat H_{\rm k, intra}$ & Empirical Bellman operators at iteration k corresponding to $H_k$ and $H_{\rm intra}$ respectively.\\
$\widehat T_{\rm k, inter}^J, \widehat T_{\rm k, intra}^J $ & Estimators of $H_k^J$ and $H_{\rm intra}^J$ respectively. \\
$\widetilde H_{\rm k, inter}^J, \widetilde H_{\rm k, intra}^J$ & Expectation of the estimators $\widehat T_{\rm k, inter}^J$ and $\widehat T_{\rm k, intra}^J$ respectively. \\
\bottomrule
\end{tabular}
\end{table}

\clearpage
\section{\MakeUppercase{Additional Related Work}}
\label{sec:related_work}

\textbf{Average-reward Q-learning.}
Classical algorithms for average-reward Q-learning include Relative Value Iteration (RVI) Q-learning and stochastic shortest path (SSP) Q-learning, whose asymptotic convergence properties were established in \citep{doi:10.1137/S0363012999361974} using tools from asynchronous stochastic approximation \citep{doi:10.1137/S0363012995282784}. These results were subsequently extended under weaker assumptions and more general settings in \citep{doi:10.1137/S0363012998346621, wan2024convergenceaveragerewardqlearningweakly, yu2024asynchronous, yu2025averagerewardreinforcementlearningsemimarkov}. However, these works focus primarily on asymptotic convergence and stability, and do not provide finite-time guarantees.

\textbf{Asymptotic stochastic approximation for options.}
The options framework \citep{SUTTON1999181} provides a principled approach to hierarchical reinforcement learning (HRL). While most prior work in HRL has focused on the discounted setting \citep{gurtler2021hierarchical, pmlr-v70-vezhnevets17a, barto2003recent}, comparatively less is understood in the average-reward regime, which is particularly relevant for continuing tasks \citep{JMLR:v8:ghavamzadeh07a}. Existing work \citep{wan2024convergenceaveragerewardqlearningweakly} establishes asymptotic convergence guarantees for RVI-style \citep{mdp_book} option-based Q-learning algorithms, but does not address finite-time analysis. Moreover, the assumptions and proof techniques used for asymptotic stochastic approximation differ substantially from those required for finite-time guarantees.

\textbf{Finite-sample analysis for stochastic approximation.}
Finite-time analysis of stochastic approximation (SA) depends critically on the contraction properties of the underlying fixed-point operator and the behaviour of the noise in the underlying stochastic process. Finite-time guarantees for SA with contractive operators have been extensively studied in \citep{NEURIPS2020_5d44ee6f, chen_lyapunov_finite}, while more recent works have extended these techniques to seminorm contractive operators \citep{chen2025nonasymptotictheoryseminormlyapunov}. However, these results are not directly applicable to average-reward RL with options. In particular, they assume contraction properties that must first be established for the \obo (Theorems \ref{thm: J-inter} and \ref{thm: J-intra}), and often rely on almost sure bounds on the noise. In the SMDP setting, such assumptions would require the holding times to be almost surely bounded, which is overly restrictive. Instead, we assume only bounded second moments of the holding times (Assumption \ref{assm:smdp}), which aligns naturally with assumptions commonly used in asymptotic analyses \citep{yu2025averagerewardreinforcementlearningsemimarkov}.

\textbf{Options.}
Options have been shown to have extensive applications in deep reinforcement learning \citep{machado2023temporal, machado2025representation, machado2017laplacian, bacon2017option, machado2017eigenoption, klissarov2023deep}. Options had been introduced as a general framework for planning by using temporally extended actions \citep{SUTTON1999181}. It allows hierarchical learning and planning \citep{mitchener2022detect, ravindran2002model}, skill sharing across tasks \citep{ramesh2019successor, klissarov2021flexible}, and faster exploration via multi-step actions \citep{klissarov2021flexible}.

\begin{remark}
\label{rem:oracle}

Consider the following MDP with states $(s_0,s_1,x,y)$, zero rewards, and transitions $$
s_0\xrightarrow{a_0}s_1,\qquad
s_1\xrightarrow{a}x,\qquad
s_1\xrightarrow{b}y.
$$
Let $x$ and $y$ each have only self loops, and define
$$
Q(s_1,a)=1,\qquad Q(s_1,b)=0,\qquad Q(x,c)=0,\qquad Q(y,c)=10.
$$

\begin{center}
\begin{tikzpicture}[
    >=stealth,
    node distance=2.2cm,
    every node/.style={font=\small},
    state/.style={circle,draw,minimum size=8mm,inner sep=0pt}
]
    \node[state] (s0) {$s_0$};
    \node[state, right=of s0] (s1) {$s_1$};
    \node[state, above right=1.1cm and 2.0cm of s1] (x) {$x$};
    \node[state, below right=1.1cm and 2.0cm of s1] (y) {$y$};

    \draw[->] (s0) -- node[above] {$a_0$} (s1);
    \draw[->] (s1) -- node[above left] {$a$} (x);
    \draw[->] (s1) -- node[below left] {$b$} (y);

    \draw[->] (x) edge[loop right] node[right] {$c$} (x);
    \draw[->] (y) edge[loop right] node[right] {$c$} (y);
\end{tikzpicture}
\end{center}
Let $\mathcal H$ denote the average reward Bellman operator i.e. $\mathcal H(Q)(s, a) = r(s, a) + \sum_{s'} P(s' |s,a) \max_{a'} Q(s', a')$
Then
$$
(\mathcal HQ)(s_1,a)=0,\qquad
(\mathcal HQ)(s_1,b)=10,
$$
and hence
$$
(\mathcal H^2Q)(s_0,a_0)=\max_{u\in\{a,b\}}(\mathcal HQ)(s_1,u)=10.
$$

However, the rollout construction \citep{zhang2021finite, chen2025nonasymptotictheoryseminormlyapunov, naskar2026parameter} chooses the intermediate action greedily with respect to $Q$. Since
$$
Q(s_1,a)>Q(s_1,b),
$$
it selects $a$, reaches $x$, and returns
$$
\hat{\mathcal H}_{\mathrm{roll}}^2Q(s_0,a_0)=0+0+\max_cQ(x,c)=0.
$$
Therefore,
$$
\mathbb E[\hat{\mathcal H}_{\mathrm{roll}}^2Q(s_0,a_0)]=0\neq10=(\mathcal H^2Q)(s_0,a_0).
$$

Thus, selecting every intermediate action according to $\arg\max_a Q(s,a)$ does not, in general, give an unbiased estimator of $\mathcal H^JQ$.

\end{remark}

\newpage
\section{\MakeUppercase{J-step contraction results for Option Bellman Operators}}
\label{sec:J-step-contraction}


In this section we establish Puterman-type $J$-step span contractions for the intra-option and
inter-option Bellman operators stated in Theorems \ref{thm: J-inter} and \ref{thm: J-intra}. We proceed with the analysis by establishing the following:
\begin{enumerate}
    \item a uniform overlap condition on the induced $J$-step state kernels, and
    \item a policy-sandwich argument for the $J$-fold Bellman iterate.
\end{enumerate}

\subsection{An overlap inequality}
\begin{lemma}[Overlap inequality]
\label{le:overlap}
Let $P,Q$ be probability distributions on a finite set $S$, and let $f:S\to\mathbb{R}$. Then
\[
\bigl|P^\top f-Q^\top f\bigr|
\le
\left(
1-\sum_{j\in S}\min\{P(j),Q(j)\}
\right)2\sp(f).
\]
\end{lemma}

\begin{proof}
Let
\[
c(j)\triangleq \min\{P(j),Q(j)\},
\qquad
\alpha\triangleq \sum_{j\in S} c(j)\in[0,1].
\]
If $\alpha=1$, then $P=Q$ and the claim is trivial. Otherwise define the residual measures
\[
\bar P(j)\triangleq P(j)-c(j)\ge 0,
\qquad
\bar Q(j)\triangleq Q(j)-c(j)\ge 0,
\]
so that $\sum_j \bar P(j)=\sum_j \bar Q(j)=1-\alpha$. Normalize:
\[
\tilde P(j)\triangleq \frac{\bar P(j)}{1-\alpha},
\qquad
\tilde Q(j)\triangleq \frac{\bar Q(j)}{1-\alpha}.
\]
Then $\tilde P,\tilde Q$ are probability distributions and
\[
P^\top f-Q^\top f
=
\sum_j (P(j)-Q(j))f(j)
=
\sum_j (\bar P(j)-\bar Q(j))f(j)
=
(1-\alpha)(\tilde P^\top f-\tilde Q^\top f).
\]
Since $\tilde P^\top f$ and $\tilde Q^\top f$ both lie in the interval
$[\min_j f(j),\max_j f(j)]$, we have
\[
|\tilde P^\top f-\tilde Q^\top f|
\le
\max_j f(j)-\min_j f(j)
=
2\sp(f).
\]
Hence
\[
|P^\top f-Q^\top f|
\le
(1-\alpha)2\sp(f)
=
\left(
1-\sum_{j\in S}\min\{P(j),Q(j)\}
\right)2\sp(f),
\]
as required.
\end{proof}

\subsection{Intra-option $J$-step contraction}

Recall the intra-option Bellman operator
\[
(H_{\mathrm{intra}}Q)(s,o)
=
r(s,o)+\sum_{s'\in S} p(s'\mid s, o)\,U_Q(s',o),
\]
where
\[
U_Q(s',o)
=
(1-\beta(s',o))Q(s',o)+\beta(s',o)\max_{o'\in O}Q(s',o').
\]

For $\Delta\in\mathbb{R}^{S\times O}$ define the state-envelope
\[
V_\Delta(s)\triangleq \max_{o\in O}\Delta(s,o),\qquad s\in S.
\]
Then
\[
\sp(V_\Delta)\le \sp(\Delta),
\]
because $V_\Delta$ is a pointwise maximum over the option coordinate.

\paragraph{Policy}
Fix an initial pair $(s,o)\in S\times O$ and a deterministic Markov option-selection policy
$\pi=(\pi_t)_{t\ge 1}$, where
\[
\pi_t:S\times O\to O.
\]
Starting from $(S_0,O_0)=(s,o)$, define the intra-option evolution by
\[
S_{t+1}\sim p(\cdot\mid S_t, O_t),
\]
and after observing $S_{t+1}$,
\[
O_{t+1}
=
\begin{cases}
O_t, & \text{with probability }1-\beta(S_{t+1},O_t),\\[1mm]
\pi_{t+1}(S_{t+1},O_t), & \text{with probability }\beta(S_{t+1},O_t).
\end{cases}
\]
Let $P^J_{\pi}(\cdot\mid s, o)$ denote the law of $S_J$ under this process:
\[
P^J_{\pi}(j\mid s, o)
\triangleq 
\mathbb{P}_\pi(S_J=j\mid S_0=s,O_0=o).
\]

We prove the following Property \ref{property:intra-mix} in Lemma \ref{lem:intra-overlap}. 

\begin{property}\label{property:intra-mix}
Assume that there exists an integer $J\ge 1$ such that for every pair of deterministic Markov
policies $\pi_1,\pi_2$,
\[
\eta_{\mathrm{intra}}(\pi_1,\pi_2)
\triangleq 
\min_{s,u\in S\atop o,o'\in O}
\sum_{j\in S}\min\Bigl\{P^J_{\pi_1}(j\mid s, o),\,P^J_{\pi_2}(j\mid u, o')\Bigr\}
>0.
\]
\end{property}

We first prove a policy-sandwich bound.

\begin{lemma}[Intra-option sandwich bound]
\label{lem:intra-sandwich}
Let $\Delta\triangleq Q-Q'$. Then for every $J\ge 1$ and every $(s,o)\in S\times O$,
\[
(H_{\mathrm{intra}}^JQ)(s,o)-(H_{\mathrm{intra}}^JQ')(s,o)
\le
\sup_{\pi}\sum_{j\in S}P^J_{\pi}(j\mid s, o)V_\Delta(j).
\]
A symmetric lower bound holds with $\inf_\pi$ in place of $\sup_\pi$.
\end{lemma}

\begin{proof}
We prove the upper bound by induction on $J$.

For $J=1$, the reward term cancels:
\[
(H_{\mathrm{intra}}Q)(s,o)-(H_{\mathrm{intra}}Q')(s,o)
=
\sum_{s'} p(s'\mid s, o)\bigl(U_Q(s',o)-U_{Q'}(s',o)\bigr).
\]
Now
\[
U_Q(s',o)-U_{Q'}(s',o)
=
(1-\beta(s',o))\Delta(s',o)
+
\beta(s',o)\bigl(\max_{o_1}Q(s',o_1)-\max_{o_1}Q'(s',o_1)\bigr).
\]
Using
\[
\max_{o_1}Q(s',o_1)-\max_{o_1}Q'(s',o_1)\le \max_{o_1}\Delta(s',o_1)=V_\Delta(s')
\]
and $\Delta(s',o)\le V_\Delta(s')$, we get
\[
U_Q(s',o)-U_{Q'}(s',o)\le V_\Delta(s').
\]
Hence
\[
(H_{\mathrm{intra}}Q)(s,o)-(H_{\mathrm{intra}}Q')(s,o)
\le
\sum_{s'} p(s'\mid s, o)V_\Delta(s').
\]
But this is exactly $\sum_j P^1_{\pi}(j\mid s, o)V_\Delta(j)$, independent of $\pi$, so the claim
holds for $J=1$.

Now suppose the claim holds for some $J\ge 1$. Then
\[
H_{\mathrm{intra}}^{J+1}=H_{\mathrm{intra}}\circ H_{\mathrm{intra}}^J,
\]
and therefore
\[
(H_{\mathrm{intra}}^{J+1}Q)(s,o)-(H_{\mathrm{intra}}^{J+1}Q')(s,o)
=
\sum_{s'} P(s'\mid s, o)
\Bigl(
U_{H_{\mathrm{intra}}^JQ}(s',o)-U_{H_{\mathrm{intra}}^JQ'}(s',o)
\Bigr).
\]
For fixed $s'$, expanding the continuation operator yields
\begin{align*}
&U_{H_{\mathrm{intra}}^JQ}(s',o)-U_{H_{\mathrm{intra}}^JQ'}(s',o)\\
&=
(1-\beta(s',o))
\Bigl(
(H_{\mathrm{intra}}^JQ)(s',o)-(H_{\mathrm{intra}}^JQ')(s',o)
\Bigr)\\
&\qquad
+
\beta(s',o)
\Bigl(
\max_{o_1}(H_{\mathrm{intra}}^JQ)(s',o_1)-\max_{o_1}(H_{\mathrm{intra}}^JQ')(s',o_1)
\Bigr)\\
&\le
(1-\beta(s',o))
\sup_\pi \sum_j P^J_{\pi}(j\mid s', o)V_\Delta(j)
+
\beta(s',o)
\max_{o_1}\sup_\pi \sum_j P^J_{\pi}(j\mid s', o_1)V_\Delta(j).
\end{align*}
To dominate both the continuation and termination branches, we construct a policy in our policy class \( \pi_t : S \times O \to O \). Define \( \tilde{\pi} \) such that, upon termination into a state \( s' \) from option \( o \), it selects an option \( o^\star(s', o) \) achieving the maximizing termination branch, and subsequently follows a maximizing continuation policy. This construction ensures that \( \tilde{\pi} \) dominates both branches. Thus, conditioning on the termination event then gives
\[
U_{H_{\mathrm{intra}}^JQ}(s',o)-U_{H_{\mathrm{intra}}^JQ'}(s',o)
\le
\sup_\pi \sum_j P^{J+1}_{\pi}(j\mid s', o)V_\Delta(j).
\]
Substituting back and averaging over $S_1\sim p(\cdot\mid s, o)$ yields
\[
(H_{\mathrm{intra}}^{J+1}Q)(s,o)-(H_{\mathrm{intra}}^{J+1}Q')(s,o)
\le
\sup_\pi \sum_j P^{J+1}_{\pi}(j\mid s, o)V_\Delta(j).
\]
This completes the induction.
\end{proof}

\begin{theorem}[J-step span contraction for the intra-option Bellman operator]\label{thm:J-intra-proxy}
Under Property \ref{property:intra-mix}, define
\[
\gamma_{\mathrm{intra}}
\triangleq 
1-\min_{\pi_1,\pi_2}\eta_{\mathrm{intra}}(\pi_1,\pi_2)\in[0,1).
\]
Then we have for all $Q,Q'\in\mathbb{R}^{S\times O}$,
\[
\sp\!\bigl(H_{\mathrm{intra}}^JQ-H_{\mathrm{intra}}^JQ'\bigr)
\le
\gamma_{\mathrm{intra}}\,\sp(Q-Q').
\]
\end{theorem}

\begin{proof}
Let
\[
\Delta\triangleq Q-Q',
\qquad
F\triangleq H_{\mathrm{intra}}^JQ-H_{\mathrm{intra}}^JQ'.
\]
By definition of the span seminorm on $S\times O$,
\[
2\sp(F)
=
\max_{(s,o),(u,o')}\bigl(F(s,o)-F(u,o')\bigr).
\]
Fix $(s,o)$ and $(u,o')$. Using Lemma \ref{lem:intra-sandwich}, the upper sandwich bound at $(s,o)$ and the lower sandwich
bound at $(u,o')$ yields,
\[
F(s,o)-F(u,o')
\le
\sup_{\pi_1}\sum_j P^J_{\pi_1}(j\mid s, o)V_\Delta(j)
-
\inf_{\pi_2}\sum_j P^J_{\pi_2}(j\mid u, o')V_\Delta(j).
\]
Hence
\[
F(s,o)-F(u,o')
\le
\sup_{\pi_1,\pi_2}
\left(
\sum_j P^J_{\pi_1}(j\mid s, o)V_\Delta(j)
-
\sum_j P^J_{\pi_2}(j\mid u, o')V_\Delta(j)
\right).
\]
Now applying Lemma \ref{le:overlap} with
\[
P(j)=P^J_{\pi_1}(j\mid s, o),
\qquad
Q(j)=P^J_{\pi_2}(j\mid u, o'),
\qquad
f=V_\Delta.
\]
This gives
\[
\left|
\sum_j P(j)V_\Delta(j)-\sum_j Q(j)V_\Delta(j)
\right|
\le
\left(
1-\sum_j \min\{P(j),Q(j)\}
\right)2\sp(V_\Delta).
\]
By Property \ref{property:intra-mix},
\[
\sum_j \min\{P(j),Q(j)\}\ge \eta_{\mathrm{intra}}(\pi_1,\pi_2),
\]
uniformly over $(s,o),(u,o')$. Therefore
\[
F(s,o)-F(u,o')
\le
2\bigl(1-\eta_{\mathrm{intra}}(\pi_1,\pi_2)\bigr)\sp(V_\Delta)
\le
2\bigl(1-\eta_{\mathrm{intra}}(\pi_1,\pi_2)\bigr)\sp(\Delta).
\]
Taking the supremum over all $(s,o),(u,o')$ and then over $\pi_1,\pi_2$ yields
\[
2\sp(F)
\le
2\left(
1-\min_{\pi_1,\pi_2}\eta_{\mathrm{intra}}(\pi_1,\pi_2)
\right)\sp(\Delta)
=
2\gamma_{\mathrm{intra}}\sp(Q-Q').
\]
\[
\implies \sp(H_{\rm intra}^JQ - H_{\rm intra}^J Q') \le \gamma_{\rm intra}\sp(Q-Q').
\]
\end{proof}

Now we will show Assumption \ref{ass:intra-assump} implies Property \ref{property:intra-mix} thus completing the proof of Theorem \ref{thm: J-intra}.

For a fixed deterministic decision rule $d: \cS \times \cO \rightarrow \cO$, define the transition kernel on the augmented state space $\cX \triangleq \cS \times \cO$ as,
\[
P_d\big((s',o')\mid(s,o)\big)
=
p(s'\mid s,o)
\Big[
(1-\beta(s',o))\indicator\{o'=o\}
+
\beta(s',o)\indicator\{o'=d(s',o)\}
\Big].
\]
An option-selection policy is specified by a sequence of deterministic
decision rules
\[
\mu=(d^{(1)},d^{(2)},\ldots),
\]
where $d^{(k)}:S\times O\to O$ is applied upon termination of an option at
step $k$. Given $\mu$, the resulting process is a time-inhomogeneous Markov
chain on the augmented state space $\mathcal X\triangleq S\times O$, whose
transition kernel at step $k$ is
\[
P_\mu^{(k)}\big((s',o')\mid(s,o)\big)
=
p(s'\mid s,o)
\Big[
(1-\beta(s',o))\indicator\{o'=o\}
+
\beta(s',o)\indicator\{o'=d^{(k)}(s',o)\}
\Big].
\]
Accordingly, $P_\mu^n$ denotes the $n$-step composition of these
kernels, with the $k$th factor using the corresponding decision rule
$d^{(k)}$.

\begin{lemma}
\label{lem:intra-overlap}
Fix two deterministic option-selection policies
\[
\mu_i = (d_i^{(1)}, d_i^{(2)}, \ldots), \quad i=1,2,
\]
and initial states $x_i=(s_i,o_i)\in\mathcal X$.
Define the reachable sets
\[
X_i(n) \triangleq  \{ x\in\mathcal X : \tilde P_{\mu_i}^n(x \mid x_i) > 0 \}.
\]

Under Assumption~\ref{ass:intra-assump}, there exists $J \le |\cS||\cO|$ such that
\[
X_1(J) \cap X_2(J) \neq \emptyset.
\]
Consequently, Property~\ref{property:intra-mix} holds.
\end{lemma}

\begin{proof}
Under Assumption \ref{ass:intra-assump}, for each $i$ the sequence $\{X_i(n)\}_{n\ge 0}$ is monotone increasing due to the positivity of the transition kernel, i.e.,
$X_i(n) \subseteq X_i(n+1)$.

Assume for contradiction that
\[
X_1(|\mathcal X|) \cap X_2(|\mathcal X|) = \emptyset.
\]

Since the sets are monotone increasing, by the pigeonhole principle there exists $m_1, m_2 \le |\mathcal X|$ such that
\[
X_i(m_i) = X_i(m_i+1), \quad i=1,2.
\]

Since $X_i(m_i) = X_i(m_i+1)$, no transitions under $\tilde P_{d_i^{(m_i)}}$
can leave $X_i(m_i)$; hence $X_i(m_i)$ is closed under $\tilde P_{d_i^{(m_i)}}$.

We now construct a policy $\delta : \mathcal X \to O$ as
\[
\delta (s, o) =
\begin{cases}
d_1^{(m_1)}(s, o), & (s,o)\in X_1(m_1),\\
d_2^{(m_2)}(s, o), & (s,o)\in X_2(m_2),\\
\text{arbitrary}, & \text{otherwise}.
\end{cases}
\]

By construction, for any $x\in X_1(m_1)$ and $y\notin X_1(m_1)$,
\[
\tilde P_{\delta}(y \mid x) = \tilde P_{d_1^{(m_1)}}(y \mid x) = 0,
\]
and similarly for $X_2(m_2)$. Thus both $X_1(m_1)$ and $X_2(m_2)$ are closed
under $\tilde P_{\delta}$.

This implies that the Markov chain induced by $\delta$ has at least two
disjoint closed communicating classes, contradicting the unichain assumption of Assumption \ref{ass:intra-assump}.

Therefore,
\[
X_1(|\mathcal X|) \cap X_2(|\mathcal X|) \neq \emptyset.
\]
Hence there exists $x \in \mathcal X$ and $J \le |\mathcal X|$ such that
\[
\tilde P_{\mu_1}^J(x \mid x_1) > 0
\quad \text{and} \quad
\tilde P_{\mu_2}^J(x \mid x_2) > 0,
\]
which establishes Property~\ref{property:intra-mix}. Since $\mu_1$ and $\mu_2$
were arbitrary, the result follows. 
\end{proof}

Thus, the proof of Theorem \ref{thm: J-intra} follows from Theorem \ref{thm:J-intra-proxy} under Assumption \ref{ass:intra-assump}.


\subsection{Inter-option $J$-step contraction}

Recall the normalized inter-option Bellman operator
\[
(H_{\mathrm{inter}}Q)(s,o)
=
\rho(s,o)
+
(1-\lambda(s,o))Q(s,o)
+
\lambda(s,o)
\sum_{s'\in S}\hat p(s'\mid s,o)\max_{o'\in O}Q(s',o'),
\]
where
\[
\rho(s,o)\triangleq \frac{\hat r(s,o)}{\hat\ell(s,o)},
\qquad
\lambda(s,o)\triangleq \frac{1}{\hat\ell(s,o)}.
\]
Since $\rho$ is independent of $Q$, it cancels when taking differences.
The remaining terms can be viewed as a controlled Markov transition on the
state-option space
\[
\mathcal X\triangleq S\times O.
\]

\paragraph{Policy and transformed kernel.}
Fix $X_0=(S_0,O_0)=(s,o)$ and let
$\pi=(\pi_t)_{t\ge1}$ be a deterministic Markov option-selection policy,
where
\[
\pi_t:\cX\to \cO.
\]
For a decision rule $d:S\to O$, define the transformed state-option kernel
\begin{equation}
\label{eq:inter-transition-kernel}
\widetilde P_d\bigl((s',o')\mid(s,o)\bigr)
\triangleq 
(1-\lambda(s,o))
\indicator \{(s',o')=(s,o)\}
+
\lambda(s,o)\hat p(s'\mid s,o)
\indicator \{o'=d(s',o)\}.
\end{equation}
Thus, with probability $1-\lambda(s,o)$ the transformed process remains at
$(s,o)$, while with probability $\lambda(s,o)$ an option-level transition
occurs according to $\hat p(\cdot\mid s,o)$ and the next option is selected
using $d$.

Note that the Bellman-greedy selectors used form the state-only subclass \(d(s',o)=\bar d(s')\). Hence $\bar d(s')$ can be used instead of $d(s',o)$ in~\eqref{eq:inter-transition-kernel}, since the next option selection in the inter-option setting only depends on the next state if the chain has transitioned to a new  state. We enlarge the admissible policy class is only for the overlap argument in Lemma~\ref{lem:inter-mix}.

Under policy $\pi$, let
\[
P^J_{\pi}(y\mid x)
\triangleq 
\mathbb P_\pi(X_J=y\mid X_0=x),
\qquad x,y\in\mathcal X,
\]
denote the resulting $J$-step transition law.

We prove the following Property \ref{property:inter-mix} in Lemma \ref{lem:inter-mix}.
\begin{property}\label{property:inter-mix}
There exists $J\ge1$ such that for every pair of deterministic Markov
policies $\pi_1,\pi_2$,
\[
\eta_{\mathrm{inter}}(\pi_1,\pi_2)
\triangleq 
\min_{x,u\in\mathcal X}
\sum_{y\in\mathcal X}
\min\Bigl\{
P^J_{\pi_1}(y\mid x),
P^J_{\pi_2}(y\mid u)
\Bigr\}
>0.
\]
\end{property}

\begin{lemma}[Inter-option sandwich bound]
Let $\Delta\triangleq Q-Q'$. Then, for every $J\ge1$ and every $x\in\mathcal X$,
\[
\inf_{\pi}
\sum_{y\in\mathcal X}P^J_{\pi}(y\mid x)\Delta(y)
\le
(H_{\mathrm{inter}}^JQ)(x)
-
(H_{\mathrm{inter}}^JQ')(x)
\le
\sup_{\pi}
\sum_{y\in\mathcal X}P^J_{\pi}(y\mid x)\Delta(y).
\]
\end{lemma}

\begin{proof}
We prove the upper bound by induction on $J$; the lower bound follows
analogously.

For $J=1$,
\[
\begin{aligned}
&(H_{\mathrm{inter}}Q-H_{\mathrm{inter}}Q')(s,o)\\
&=
(1-\lambda(s,o))\Delta(s,o)
+
\lambda(s,o)
\sum_{s'\in S}\hat p(s'\mid s,o)
\left(
\max_{o_1\in O}Q(s',o_1)
-
\max_{o_1\in O}Q'(s',o_1)
\right).
\end{aligned}
\]
Using
\[
\max_{o_1\in O}Q(s',o_1)
-
\max_{o_1\in O}Q'(s',o_1)
\le
\max_{o_1\in O}\Delta(s',o_1),
\]
we obtain
\[
\begin{aligned}
&(H_{\mathrm{inter}}Q-H_{\mathrm{inter}}Q')(s,o)\\
&\le
(1-\lambda(s,o))\Delta(s,o)
+
\lambda(s,o)
\sum_{s'\in S}\hat p(s'\mid s,o)
\max_{o_1\in O}\Delta(s',o_1).
\end{aligned}
\]
Choosing, for every successor state $s'$,
\[
\pi_1(s')
\in
\arg\max_{o_1\in O}\Delta(s',o_1),
\]
the right-hand side is precisely
\[
\sum_{y\in\mathcal X}
P^1_{\pi}(y\mid(s,o))\Delta(y)
\]
for an admissible policy $\pi$. Hence
\[
(H_{\mathrm{inter}}Q-H_{\mathrm{inter}}Q')(s,o)
\le
\sup_\pi
\sum_{y\in\mathcal X}
P^1_{\pi}(y\mid(s,o))\Delta(y).
\]

Now assume the claim holds for some $J\ge1$, and define
\[
F_J
\triangleq 
H_{\mathrm{inter}}^JQ
-
H_{\mathrm{inter}}^JQ'.
\]
Then
\[
F_{J+1}
=
H_{\mathrm{inter}}(H_{\mathrm{inter}}^JQ)
-
H_{\mathrm{inter}}(H_{\mathrm{inter}}^JQ'),
\]
and therefore
\[
\begin{aligned}
F_{J+1}(s,o)
&=
(1-\lambda(s,o))F_J(s,o)+
\lambda(s,o)
\sum_{s'\in S}\hat p(s'\mid s,o)
\bigg(
\max_{o_1\in O}(H_{\mathrm{inter}}^JQ)(s',o_1)
-
\max_{o_1\in O}(H_{\mathrm{inter}}^JQ')(s',o_1)
\bigg).
\end{aligned}
\]
Consequently,
\[
F_{J+1}(s,o)
\le
(1-\lambda(s,o))F_J(s,o)
+
\lambda(s,o)
\sum_{s'\in S}\hat p(s'\mid s,o)
\max_{o_1\in O}F_J(s',o_1).
\]
By the induction hypothesis,
\[
F_J(x)
\le
\sup_{\pi}
\sum_{y\in\mathcal X}
P^J_{\pi}(y\mid x)\Delta(y).
\]
Hence,
\[
\begin{aligned}
F_{J+1}(s,o)
&\le
(1-\lambda(s,o))
\sup_{\pi}
\sum_yP^J_{\pi}(y\mid(s,o))\Delta(y)+
\lambda(s,o)
\sum_{s'}\hat p(s'\mid s,o)
\max_{o_1}
\sup_{\pi}
\sum_yP^J_{\pi}(y\mid(s',o_1))\Delta(y).
\end{aligned}
\]

For every successor state $s'$, choose an option attaining the maximizing
branch and thereafter follow a maximizing continuation policy. Together with
a maximizing continuation policy on the self-loop branch, these choices
define an admissible Markov policy $\tilde\pi$. Therefore,
\[
F_{J+1}(s,o)
\le
\sum_{y\in\mathcal X}
P^{J+1}_{\tilde\pi}(y\mid(s,o))\Delta(y)
\le
\sup_{\pi}
\sum_{y\in\mathcal X}
P^{J+1}_{\pi}(y\mid(s,o))\Delta(y).
\]
This proves the upper bound.

For the lower bound, use
\[
\max_o Q(s,o)-\max_o Q'(s,o)
\ge
\min_o\Delta(s,o)
\]
and repeat the same induction with minimizing policies.
\end{proof}

\begin{theorem}[J-step span contraction for the inter-option Bellman operator]
\label{thm:J-inter-proxy}
Under Property~\ref{property:inter-mix}, define
\[
\gamma_{\mathrm{inter}}
\triangleq 
1-
\min_{\pi_1,\pi_2}
\eta_{\mathrm{inter}}(\pi_1,\pi_2)
\in[0,1).
\]
Then, for all $Q,Q'\in\mathbb R^{S\times O}$,
\[
\sp\!\bigl(
H_{\mathrm{inter}}^JQ
-
H_{\mathrm{inter}}^JQ'
\bigr)
\le
\gamma_{\mathrm{inter}}\sp(Q-Q').
\]
\end{theorem}

\begin{proof}
Let
\[
\Delta\triangleq Q-Q',
\qquad
F\triangleq H_{\mathrm{inter}}^JQ-H_{\mathrm{inter}}^JQ'.
\]
Fix $x,u\in\mathcal X$. By the upper and lower sandwich bounds,
\[
F(x)-F(u)
\le
\sup_{\pi_1,\pi_2}
\left(
\sum_{y\in\mathcal X}
P^J_{\pi_1}(y\mid x)\Delta(y)
-
\sum_{y\in\mathcal X}
P^J_{\pi_2}(y\mid u)\Delta(y)
\right).
\]
For fixed $\pi_1,\pi_2$, apply Lemma~\ref{le:overlap} with
\[
P(y)=P^J_{\pi_1}(y\mid x),
\qquad
q(y)=P^J_{\pi_2}(y\mid u),
\qquad
f(y)=\Delta(y).
\]
By Property~\ref{property:inter-mix},
\[
\sum_{y\in\mathcal X}\min\{P(y),q(y)\}
\ge
\eta_{\mathrm{inter}}(\pi_1,\pi_2).
\]
Therefore,
\[
F(x)-F(u)
\le
2\bigl(1-\eta_{\mathrm{inter}}(\pi_1,\pi_2)\bigr)
\sp(\Delta),
\]
where we used
\[
\max_y\Delta(y)-\min_y\Delta(y)=2\sp(\Delta).
\]
Taking the supremum over $x,u$ gives
\[
2\sp(F)
\le
2
\left(
1-\min_{\pi_1,\pi_2}
\eta_{\mathrm{inter}}(\pi_1,\pi_2)
\right)
\sp(\Delta).
\]
Hence
\[
\sp(F)
\le
\gamma_{\mathrm{inter}}\sp(Q-Q').
\]
\end{proof}

We now show that Assumption~\ref{ass:inter-assumption} implies
Property~\ref{property:inter-mix}, completing the proof of
Theorem~\ref{thm: J-inter}.

For a deterministic stationary option-selection rule $d:\cX\to \cO$, let
$\widetilde P_d$ denote the corresponding transformed kernel on
$\mathcal X$:
\[
\widetilde P_d\bigl((s',o')\mid(s,o)\bigr)
=
(1-\lambda(s,o))
\indicator \{(s',o')=(s,o)\}
+
\lambda(s,o)\hat p(s'\mid s,o)
\indicator \{o'=d(s', o)\}.
\]

An option-selection policy is specified by a sequence of deterministic
decision rules
\[
\pi=(d^{(0)},d^{(1)},\ldots),
\]
where $d^{(k)}:\cX \to \cO$. Given $\pi$, the resulting process is a time-inhomogeneous Markov
chain on the augmented state space $\mathcal X$, whose
transition kernel at step $k$ is
\[
\widetilde P_\pi^{(k)}\big((s',o')\mid(s,o)\big)
=
(1-\lambda(s,o))
\indicator \{(s',o')=(s,o)\}
+
\lambda(s,o)\hat p(s'\mid s,o)
\indicator \{o'=d^{(k)}(s',o)\}.
\]
Accordingly, $\widetilde P_\pi^n$ denotes the $n$-step composition of these
kernels, with the $k$th factor using the corresponding decision rule
$d^{(k)}$.

For $i=1,2$, fix deterministic option selection policies $\pi_i = \{d_i^{(n)}\}_{n \ge 0}$ and initial
state-option pairs $x_i\in\mathcal X$. Define the reachable sets
\[
X_i(n)
\triangleq 
\left\{
y\in\mathcal X:
P^n_{\pi_i}(y\mid x_i)>0
\right\}.
\]

\begin{lemma}
\label{lem:inter-mix}
Under Assumption~\ref{ass:inter-assumption}, there exists
$J\le |\mathcal X|=|\cS||\cO|$ such that
\[
X_1(J)\cap X_2(J)\neq\emptyset.
\]
Consequently, Property~\ref{property:inter-mix} holds.
\end{lemma}

\begin{proof}
Because every transformed kernel contains a positive self-loop,
the reachable sets are monotone:
\[
X_i(n)\subseteq X_i(n+1),
\qquad i=1,2.
\]

Suppose, for contradiction, that
\[
X_1(n)\cap X_2(n)=\emptyset
\qquad
\text{for all }n\le|\mathcal X|.
\]
Since $\mathcal X$ is finite and the sets $X_i(n)$ are monotone, there
exists $m_1, m_2\le|\mathcal X|$ for which the reachable sets stabilize:
\[
X_i(m_i)=X_i(m_i+1),
\qquad i=1,2.
\]
At the stabilization time, let $d_i^{(m_i)}$ denote the deterministic
decision rule used by $\pi_i$. Since
\[
X_i(m_i)=X_i(m_i+1),
\]
no transition under
$\widetilde P_{d_i^{(m_i)}}$ can leave $X_i(m_i)$. Thus $X_i(m_i)$ is closed
under $\widetilde P_{d_i^{(m_i)}}$.

Because $X_1(m_1)$ and $X_2(m_2)$ are disjoint, define a deterministic
stationary decision rule $\delta:\cX\to \cO$, similar to the intra-option setting, which agrees with the respective
decision rules on the two closed sets and choosing arbitrary options
elsewhere. Equivalently, on the state-option space, $\delta$ induces a
kernel satisfying
\[
\widetilde P_\delta(y\mid x)
=
\widetilde P_{d_1^{(m_1)}}(y\mid x),
\qquad
x\in X_1(m_1),
\]
and
\[
\widetilde P_\delta(y\mid x)
=
\widetilde P_{d_2^{(m_2)}}(y\mid x),
\qquad
x\in X_2(m_2).
\]
Consequently, both $X_1(m_1)$ and $X_2(m_2)$ are closed under
$\widetilde P_\delta$.

The Markov chain induced by $\delta$ therefore contains at least two
disjoint closed communicating classes, contradicting the unichain
assumption.

Hence, for some $J\le|\mathcal X|$,
\[
X_1(J)\cap X_2(J)\neq\emptyset.
\]
Thus there exists $y\in\mathcal X$ such that
\[
P^J_{\pi_1}(y\mid x_1)>0,
\qquad
P^J_{\pi_2}(y\mid x_2)>0.
\]
Therefore,
\[
\sum_{z\in\mathcal X}
\min\Bigl\{
P^J_{\pi_1}(z\mid x_1),
P^J_{\pi_2}(z\mid x_2)
\Bigr\}>0.
\]
Since the state-option space and the set of deterministic decision rules
are finite, taking the minimum over all initial pairs and admissible policy
pairs yields a strictly positive overlap coefficient. Hence
Property~\ref{property:inter-mix} follows.
\end{proof}

Combining the preceding lemma with
Theorem~\ref{thm:J-inter-proxy} establishes
Theorem~\ref{thm: J-inter}.

\paragraph{Relaxing the strict positivity/self-loop conditions.}
\label{app:lazy}
The strict conditions $p(s\mid s,o)>0$, $\lambda(s,o)<1$, and
$\beta(s,o)<1$ are sufficient conditions ensuring aperiodicity of
the augmented state--option chains. They can be relaxed using the standard
aperiodicity transformation~\citep[Sec.~8.5.4]{mdp_book}.

\begin{enumerate}[leftmargin=*,nosep]
\item \textbf{Inter-option.}
For any $\eta\in(0,1)$, define
\[
\bar H_{\mathrm{\eta, inter}}Q
\triangleq
(1-\eta)Q+\eta \bar H_{\mathrm{inter}}Q.
\]
The associated controlled kernel is
$\widetilde P_d^{(\eta)}=(1-\eta)I+\eta\widetilde P_d$ and therefore has
positive self-loops for every selector $d$. Hence the overlap and policy-sandwich
arguments used for Theorem~\ref{thm: J-inter} apply verbatim to
$\bar H_{\mathrm{inter}}^{(\eta)}$, yielding a $J$-step span contraction.
Moreover,
\[
\bar H_{\mathrm{inter}}Q-Q=c\mathbf 1
\quad\Longleftrightarrow\quad
\bar H_{\mathrm{\eta, inter}}Q-Q=\eta c\mathbf 1,
\]
so the quotient-space optimality class is preserved, and hence the greedy optimal
option-selection policies, are unchanged.

\item \textbf{Intra-option.}
Similarly, define
\[
H_{\mathrm{\eta, intra}}Q
\triangleq
(1-\eta)Q+\eta H_{\mathrm{intra}}Q.
\]
For every deterministic selector $d$, the associated kernel is
$P_d^{(\eta)}=(1-\eta)I+\eta P_d$, which has positive self-loops uniformly
over $d$. Thus the proof of Theorem~\ref{thm: J-intra} applies unchanged to
$H_{\mathrm{\eta, intra}}$, giving a $J$-step span contraction. Also,
\[
H_{\mathrm{intra}}Q-Q=c\mathbf 1
\quad\Longleftrightarrow\quad
H_{\mathrm{\eta, intra}}Q-Q=\eta c\mathbf 1,
\]
so the quotient-space optimality class and the induced greedy optimal
option-selection policies are preserved.
\end{enumerate}

\newpage
\section{\MakeUppercase{Analysis of Synchronous Option Algorithms}}
\label{sec:synchronous}

In this section we show that the synchronous inter- and intra-option algorithms outlined in Algorithm \ref{alg:sync-inter} and Algorithm \ref{alg:sync-intra} respectively can be
written as stochastic approximation schemes driven by $J$-step Bellman operators. In section \ref{subsec:sync_estimators}, we formulate the  problem setup and construct the estimators of the J-step Bellman operators and analyze their properties. Section~\ref{subsec:sync-intra} analyzes synchronous intra-option
Q-learning: we first establish an almost-sure growth bound for the iterates,
then control the bias induced by empirically approximately-greedy intermediate
option selection, and finally combine these bounds with the generalized
Moreau-envelope framework of~\cite{chen2025nonasymptotictheoryseminormlyapunov} to obtain the finite-time rate.
Section~\ref{subsec:sync-inter} follows the same analysis for synchronous inter-option Q-learning,
with one additional first-order perturbation arising from the online estimation
of the reward rates and inverse mean holding times.

\begin{algorithm}[t!]
\caption{Synchronous inter-option $J$-step Q-learning}
\label{alg:sync-inter}
\begin{algorithmic}[1]
\Require Horizon $J$, step-sizes $\{\alpha_k\}_{k\ge0}$, initial iterate $Q_0$

\State Initialize $A_0(x)\gets0$, $B_0(x)\gets1$, $C_0\gets1$, and
$N_0(x,s')\gets0$ for every $x=(s,o)\in\cS\times\cO$ and $s'\in\cS$.
\State Set $\widehat p_0(\cdot|x)$ to any distribution on $\cS$.

\For{$k=0,1,2,\ldots$}

    \State For every $x = (s, o)\in\cS\times\cO$, set
    $\widehat\rho_k(x)\gets A_k(x)/B_k(x)$ and
    $\widehat\lambda_k(x)\gets C_k/B_k(x)$.

    \If{$k\ge1$}
        \State For every $x\in\cS\times\cO$ and $s'\in\cS$, set
        $\widehat p_k(s'|x)\gets N_k(x,s')/k$.
    \EndIf

    \State Define the empirical Bellman operator
    \begin{align*}
    (\widehat H_{\rm k, inter}Q)(s,o)\gets
    \widehat\rho_k(s,o)
    +(1-\widehat\lambda_k(s,o))Q(s,o)
    +\widehat\lambda_k(s,o)\sum_{s'}\widehat p_k(s'|s,o)\max_{o'}Q(s',o') \quad \forall (s,o) \in \cS \times \cO.
    \end{align*}
    \ForAll{$x=(s,o)\in\cS\times\cO$}

        \State $X_k^{(0)}\gets(s,o)$

        \For{$i=0,\ldots,J-1$}

            \State Write $X_k^{(i)}=(S_k^{(i)},O_k^{(i)})$.

            \State Independently execute option $O_k^{(i)}$ from $S_k^{(i)}$
            and observe $(\widetilde R_k^{(i)},L_k^{(i)},\overline S_k^{(i+1)})$.

            \State Independently sample
            $D_k^{(i+1)}\sim\operatorname{Bernoulli}(\widehat\lambda_k(X_k^{(i)}))$.

            \If{$D_k^{(i+1)}=0$}
                \State $X_k^{(i+1)}\gets X_k^{(i)}$.
            \ElsIf{$i\le J-2$}
                \State $S_k^{(i+1)}\gets\overline S_k^{(i+1)}$.
                \State $O_k^{(i+1)}\gets
                \arg\max_{o'\in\cO}(\widehat H_{\rm k, inter}^{\,J-i-1}Q_k)(S_k^{(i+1)},o')$.
                \State $X_k^{(i+1)}\gets(S_k^{(i+1)},O_k^{(i+1)})$.
            \EndIf

        \EndFor

        \State Compute
        \begin{align*}
        \widehat T_{k,\mathrm{inter}}^{J}(Q_k)(s,o)
        \gets
        \sum_{i=0}^{J-1}\widehat\rho_k(X_k^{(i)})
        +(1-D_k^{(J)})Q_k(S_k^{(J-1)},O_k^{(J-1)})
        +D_k^{(J)}\max_{o'\in\cO}Q_k(\overline S_k^{(J)},o').
        \end{align*}

        \State Update
        $Q_{k+1}(s,o)\gets
        Q_k(s,o)+\alpha_k\bigl(
        \widehat T_{k,\mathrm{inter}}^{J}(Q_k)(s,o)-Q_k(s,o)
        \bigr)$.

        \State Set
        \begin{align*}
         &A_{k+1}(x)\gets A_k(x)+\widetilde R_k^{(0)}\\
         &B_{k+1}(x)\gets B_k(x)+L_k^{(0)}\\
         &N_{k+1}(x,s')\gets N_k(x,s')+\indicator \{\overline S_k^{(1)}=s'\} \quad \forall s'\in\cS
        \end{align*}
    \EndFor
    \State Set $C_{k+1}\gets C_k+1$.

\EndFor
\end{algorithmic}
\end{algorithm}






\begin{algorithm}[t!]
\caption{Synchronous intra-option $J$-step Q-learning}
\label{alg:sync-intra}
\begin{algorithmic}[1]
\Require Horizon $J$, step-sizes $\{\alpha_k\}_{k\ge0}$, initial iterate $Q_0$

\State Initialize $R_0(x)\gets0$ and $N_0(x,s')\gets0$ for every
$x=(s,o)\in\cS\times\cO$ and $s'\in\cS$.
\State Set $\widehat r_0(x)\gets0$ and $\widehat p_0(\cdot|x)$ to any
distribution on $\cS$.

\For{$k=0,1,2,\ldots$}

    \If{$k\ge1$}
        \State For every $x\in\cS\times\cO$, set
        $\widehat r_k(x)\gets R_k(x)/k$.
        \State For every $x\in\cS\times\cO$ and $s'\in\cS$, set
        $\widehat p_k(s'|x)\gets N_k(x,s')/k$.
    \EndIf

    \State Define the empirical Bellman operator
    \[
    (\widehat H_{k,\mathrm{intra}}Q)(s,o)
    \gets
    \widehat r_k(s,o)
    +
    \sum_{s'}\widehat p_k(s'|s,o)
    \bigl[
    (1-\beta(s',o))Q(s',o)
    +
    \beta(s',o)\max_{o'}Q(s',o')
    \bigr],
    \quad \forall(s,o)\in\cS\times\cO.
    \]

    \ForAll{$x=(s,o)\in\cS\times\cO$}

        \State $S_k^{(0)}\gets s$, $O_k^{(0)}\gets o$.

        \For{$i=0,\ldots,J-1$}

            \State Sample
            $A_k^{(i)}\sim\pi(\cdot|S_k^{(i)},O_k^{(i)})$.

            \State Sample
            $(R_k^{(i+1)},S_k^{(i+1)})
            \sim p(\cdot,\cdot|S_k^{(i)},A_k^{(i)})$.

            \State Sample
            $B_k^{(i+1)}
            \sim\operatorname{Bernoulli}
            (\beta(S_k^{(i+1)},O_k^{(i)}))$.

            \If{$i\le J-2$}
                \If{$B_k^{(i+1)}=0$}
                    \State $O_k^{(i+1)}\gets O_k^{(i)}$.
                \Else
                    \State $O_k^{(i+1)}
                    \gets
                    \arg\max_{o'\in\cO}
                    (\widehat H_{k,\mathrm{intra}}^{\,J-i-1}Q_k)
                    (S_k^{(i+1)},o')$.
                \EndIf
            \EndIf

        \EndFor

        \State Compute
        \[
        \widehat T_{k,\mathrm{intra}}^{J}(Q_k)(s,o)
        \gets
        \sum_{i=0}^{J-1}R_k^{(i+1)}
        +(1-B_k^{(J)})Q_k(S_k^{(J)},O_k^{(J-1)})
        +B_k^{(J)}\max_{o'\in\cO}Q_k(S_k^{(J)},o').
        \]

        \State Update
        $Q_{k+1}(s,o)\gets
        Q_k(s,o)+\alpha_k
        \bigl(
        \widehat T_{k,\mathrm{intra}}^{J}(Q_k)(s,o)
        -Q_k(s,o)
        \bigr)$.

        \State Set
        \begin{align*}
        &R_{k+1}(x)\gets R_k(x)+R_k^{(1)}\\
        &N_{k+1}(x,s')\gets
        N_k(x,s')+\indicator \{S_k^{(1)}=s'\} \quad \forall s' \in \cS
        \end{align*}.

    \EndFor

\EndFor
\end{algorithmic}
\end{algorithm}

\subsection{$J$-step estimators for the synchronous option algorithms}
\label{subsec:sync_estimators}

In the synchronous setting, all coordinates $(s,o)\in S\times O$ are
updated simultaneously at each iteration. Accordingly, for every
$(s,o)\in S\times O$, we construct an independent $J$-step rollout
initialized from that coordinate. We construct the estimators for both the
intra-option and inter-option settings.

\paragraph{Synchronous intra-option estimator.}
Recall that the intra-option Bellman operator is
\begin{align*}
(H_{\mathrm{intra}}Q)(s,o)
=
\sum_{a\in A}\pi(a\mid s,o)
\sum_{s',r} p(s',r\mid s,a)
\Bigl(
r(s,a)+(1-\beta(s',o))Q(s',o)
+\beta(s',o)\max_{o'\in O}Q(s',o')
\Bigr).
\end{align*}

Define the empirical reward and transition kernel by
\[
\widehat r_k(s,o)
\triangleq
\frac1k\sum_{t=0}^{k-1}R_t^{(1)}(s,o),
\qquad
\widehat p_k(s'\mid s,o)
\triangleq
\frac1k\sum_{t=0}^{k-1}
\mathbf 1\{S_t^{(1)}(s,o)=s'\}.
\]
Let $\widehat H_{\rm k, intra}$ then denote the empirical intra-option Bellman operator at iteration k,
\[
(\widehat H_{\rm k, intra}Q)(s,o)
=
\widehat r_k(s,o)
+
\sum_{s'}
\widehat p_k(s'\mid s,o)
\left[
(1-\beta(s',o))Q(s',o)
+
\beta(s',o)\max_{o'}Q(s',o')
\right].
\]

We now describe the construction of the estimator. At iteration k, for each $(s,o)\in S\times O$ and let $S_k^{(0)}=s$, $O_k^{(0)}=o$. For
$i=0,1,\dots,J-1$, generate
\begin{align*}
A_k^{(i)}\sim\pi(\cdot\mid S_k^{(i)},O_k^{(i)}),
\qquad
(R_k^{(i+1)},S_k^{(i+1)})\sim p(\cdot,\cdot\mid S_k^{(i)},A_k^{(i)}),
\end{align*}
and
\begin{align*}
B_k^{(i+1)}
\sim
\mathrm{Bernoulli}\bigl(\beta(S_k^{(i+1)},O_k^{(i)})\bigr).
\end{align*}
If $i\le J-2$, define
\begin{align*}
O_k^{(i+1)}
=
\begin{cases}
O_k^{(i)}, & B_k^{(i+1)}=0,\\[2mm]
\displaystyle
\arg\max_{o'\in O}
\bigl(\widehat H_{\rm k, intra}^{\,J-i-1}Q\bigr)(S_k^{(i+1)},o'),
& B_k^{(i+1)}=1.
\end{cases}
\end{align*}
The corresponding $J$-step target is
\begin{align*}
\widehat T_{\mathrm{k, intra}}^{J}(Q)(s,o)
\triangleq
\sum_{i=0}^{J-1}R_k^{(i+1)}
+
(1-B_k^{(J)})Q(S_k^{(J)},O_k^{(J-1)})
+
B_k^{(J)}\max_{o'\in O}Q(S_k^{(J)},o').
\end{align*}

For a fixed $Q$, define, for
$m=1,\ldots,J-1$, the empirical greedy decision rule
\[
\widehat d_{k,m}(s)
\in
\arg\max_{o\in\cO}
\bigl(\widehat H_{\rm k, intra}^{\,m}Q\bigr)(s,o).
\]
Thus, whenever there are $m$ Bellman steps remaining after an option
terminates, the next option is chosen greedily with respect to
$\widehat H_{\rm k, intra}^{\,m}Q$, rather than the unavailable oracle quantity
$H_{\rm intra}^mQ$.

For a deterministic decision rule $d:\cS\to\cO$, define
\begin{equation}
\label{eq:Hd-intra}
(H_d^{\rm intra} Q)(s,o)
\triangleq
r(s,o)
+
\sum_{s'}p(s'|s,o)
\left[
(1-\beta(s',o))Q(s',o)
+
\beta(s',o)Q(s',d(s'))
\right].
\end{equation}
The intra-option Bellman operator satisfies $H_{\rm intra}Q=\max_d H_d^{\rm intra}Q$ coordinatewise.
The expectation of $\widehat T_{\rm k, intra}^J$ is therefore
\[
\E[\widehat T_{\rm k, intra}^JQ]=\widetilde H_{\rm k, intra}^J Q \triangleq
H_{\widehat d_{k,J-1}}^{\rm intra}
H_{\widehat d_{k,J-2}}^{\rm intra}
\cdots
H_{\widehat d_{k,1}}^{\rm intra}
H_{\rm intra}Q.
\]
Notice that $\widetilde H_{\rm k, intra}^J$ is generally not the $J$-fold composition
of a single operator, since the greedy rule depends on the remaining
horizon.

\paragraph{Synchronous inter-option estimator.}
Let $
\mathcal X\triangleq S\times O,
\quad
x=(s,o)\in\mathcal X$.
For the inter-option setting, recall the normalized Bellman operator
\begin{align*}
(H_{\mathrm{inter}}Q)(s,o)
=
\rho(s,o)
+
(1-\lambda(s,o))Q(s,o)
+
\lambda(s,o)
\sum_{s'\in S}
\hat p(s'\mid s,o)
\max_{o'\in O}Q(s',o'),
\end{align*}
where
$
\rho(s,o)
\triangleq
\frac{\hat r(s,o)}{\hat\ell(s,o)},
\lambda(s,o)
\triangleq
\frac{1}{\hat\ell(s,o)}.
$
Here
$
\hat r(s,o)
=
\E[\widetilde R\mid s,o],
\hat\ell(s,o)
=
\E[L\mid s,o],
$
where $\widetilde R$ and $L$ denote respectively the cumulative reward
and holding time of one execution of option $o$ starting from $s$.

Under Assumption~\ref{assm:smdp}, primitive rewards satisfy
$|R_t|\le 1$. Consequently, $
|\widetilde R|\le L.$
Moreover, since $L\ge1$ almost surely,
\begin{align*}
\hat\ell(s,o)\ge1,
\qquad
0<\lambda(s,o)\le1,
\qquad
|\rho(s,o)|\le1.
\end{align*}

Unlike the intra-option operator, the quantities $\rho$ and $\lambda$
are not known. We therefore estimate them online from
previous synchronous generative samples.

For each $x=(s,o)\in\mathcal X$, let
$
\bigl\{
\widetilde R_t(x),L_t(x)
\bigr\}_{t\ge0}
$
denote the sequence of first-step option-execution samples generated
from the rollout initialized at $x$ at each iteration. Across iterations, these samples
are i.i.d. conditional on the fixed initial coordinate $x$, with
$
\E[\widetilde R_t(x)]
=
\hat r(x),
\E[L_t(x)]
=
\hat\ell(x).
$

At the beginning of iteration $k$, define the  rolling statistics
\begin{align*}
A_k(x)
\triangleq
\sum_{t=0}^{k-1}\widetilde R_t(x), \quad 
B_k(x)
\triangleq
1+\sum_{t=0}^{k-1}L_t(x),
\quad
C_k\triangleq k+1,
\end{align*}
and the corresponding estimates
\begin{align*}
\widehat\rho_k(x)
\triangleq
\frac{A_k(x)}{B_k(x)},
\qquad
\widehat\lambda_k(x)
\triangleq
\frac{C_k}{B_k(x)}.
\end{align*}
The pseudocount in $B_k(x)$ ensures that both quantities are well-defined
also at $k=0$.

Since
$
|A_k(x)|
\le
\sum_{t=0}^{k-1}L_t(x)
\le
B_k(x), \quad
B_k(x)
=
1+\sum_{t=0}^{k-1}L_t(x)
\ge
k+1=C_k,
$
we have pathwise
\begin{align*}
|\widehat\rho_k(x)|\le1,
\qquad
0<\widehat\lambda_k(x)\le1.
\end{align*}
Let $H_k$ denote the Bellman operator obtained by replacing
$(\rho,\lambda)$ by $(\widehat\rho_k,\widehat\lambda_k)$:
\begin{align*}
(H_kQ)(s,o)
=
\widehat\rho_k(s,o)
+
(1-\widehat\lambda_k(s,o))Q(s,o)
+
\widehat\lambda_k(s,o)
\sum_{s'\in S}
\hat p(s'\mid s,o)
\max_{o'\in O}Q(s',o').
\end{align*}

Similar to the intra-option setting, we will define the empircal Bellman operator associated with $H_k$ and then use it to construct a $J$-step estimator.

Define the empirical transition kernel by
\[
\widehat p_k(s'\mid s,o)
\triangleq
\frac1k\sum_{t=0}^{k-1}
\mathbf 1\{\bar S_t^{(1)}(s,o)=s'\}.
\]
Let $\widehat H_{\rm k, inter}$ then denote the empirical inter-option Bellman operator at iteration k,
\[
(\widehat H_{\rm k, inter}Q)(s,o)
=
\widehat\rho_k(s,o)
+
(1-\widehat\lambda_k(s,o))Q(s,o)
+
\widehat\lambda_k(s,o)
\sum_{s'\in S}
\hat p_k(s'\mid s,o)
\max_{o'\in O}Q(s',o').
\]

For each initial coordinate $X_k^{(0)}=(S_k^{(0)},O_k^{(0)})=(s,o)$. At stage
$i=0,\dots,J-1$, independently generate an option-level sample
$
\bigl(
\widetilde R_i,L_i,\overline S_k^{(i+1)}
\bigr)
$
from the option-execution law corresponding to $X_k^{(i)}=(S_k^{(i)},O_k^{(i)})$, and
independently generate
\begin{align*}
D_k^{(i+1)}
\sim
\mathrm{Bernoulli}
\bigl(
\widehat\lambda_k(X_k^{(i)})
\bigr).
\end{align*}

If $D_k^{(i+1)}=0$, set
$
X_k^{(i+1)}=X_k^{(i)}.
$

If $D_k^{(i+1)}=1$ and $i\le J-2$, set
$
S_k^{(i+1)}=\overline S_k^{(i+1)},
$
and choose
$
O_k^{(i+1)}
\in
\arg\max_{o'\in O}
\bigl(
\widehat H_{\rm k, inter}^{\,J-i-1}Q
\bigr)(S_k^{(i+1)},o').
$

The $J$-step synchronous target is then
\begin{align*}
\widehat{T}_{k,\mathrm{inter}}^{J}
(Q)(s,o)
\triangleq
\sum_{i=0}^{J-1}
\widehat\rho_k(X_k^{(i)})
+ (1-D_k^{(J)})
Q(S_k^{(J-1)}, O_k^{(J-1)}) + D_k^{(J)}\max_{o \in \cO} Q(\bar S_K^{(J)}, o).
\end{align*}

Similar to the intra-option setting, we will now characterize the conditional mean of the estimator $\widehat{T}_{k,\mathrm{inter}}^{J}
(Q)$. 

For a fixed $Q$, define, for
$m=1,\ldots,J-1$, the empirical greedy decision rule
\[
\widehat d_{k,m}(s)
\in
\arg\max_{o\in\cO}
\bigl(\widehat H_{\rm k, inter}^{\,m}Q\bigr)(s,o).
\]

For a deterministic decision rule $d:\cS\to\cO$, define
\begin{equation}
\label{eq:Hd-inter}
(H_d^{\rm k, inter} Q)(s,o)
\triangleq
\widehat\rho_k(s,o)
+
(1-\widehat\lambda_k(s,o))Q(s,o)
+
\widehat\lambda_k(s,o)
\sum_{s'\in S}
\hat p(s'\mid s,o)Q(s',d(s')).
\end{equation}
The Bellman operator $H_k$ satisfies $H_kQ=\max_d H_d^{\rm k, inter}Q$ coordinatewise.
The expectation of $\widehat T_{\rm k, inter}^J$ is therefore
\[
\E[\widehat T_{\rm k, inter}^JQ]=\widetilde H_{\rm k, inter}^J Q \triangleq
H_{\widehat d_{k,J-1}}^{\rm k, inter}
H_{\widehat d_{k,J-2}}^{\rm k, inter}
\cdots
H_{\widehat d_{k,1}}^{\rm k, inter} H_kQ.
\]
Notice that $\widetilde H_{\rm k, inter}^J$ is generally not the $J$-fold composition
of a single operator, since the greedy rule depends on the remaining
horizon.

\paragraph{Synchronous SA recursions.}
Let $\cF_k$ denote the history available immediately before generating the
fresh rollouts at iteration $k$. More precisely, let
\[
\cF_k
\triangleq
\sigma\!\left(
Q_0,\,
\{\cZ_t(x):0\le t\le k-1,\ x\in\cS\times\cO\}
\right),
\]
where $\cZ_t(x)$ collects all rollout samples generated from initial
coordinate $x$ at iteration $t$. Hence, $Q_k$ and all empirical quantities
constructed from past samples, including the empirical Bellman operators
$\widehat H_k$, are $\cF_k$-measurable, while the rollouts generated at
iteration $k$ are fresh conditional on $\cF_k$.

For intra-option learning, the synchronous recursion is
\[
Q_{k+1}
=
Q_k
+
\alpha_k
\left(
\widehat T_{k,\mathrm{intra}}^{J}(Q_k)
-
Q_k
\right).
\]
By construction,
\[
\E\!\left[
\widehat T_{k,\mathrm{intra}}^{J}(Q_k)
\mid
\cF_k
\right]
=
\widetilde H_{k,\mathrm{intra}}^JQ_k,
\]
where $\widetilde H_{k,\mathrm{intra}}^JQ_k$ is the conditional mean of the
fresh $J$-step rollout when the intermediate options are selected using the
empirical continuation values $\widehat H_{k, intra}^mQ_k$. Define
\[
w_k^{\mathrm{intra}}
\triangleq
\widehat T_{k,\mathrm{intra}}^{J,\mathrm{sync}}(Q_k)
-
\widetilde H_{k,\mathrm{intra}}^JQ_k.
\]
Then $\E[w_k^{\mathrm{intra}}\mid\cF_k]=0$, so
$\{w_k^{\mathrm{intra}}\}_{k\ge0}$ is a martingale-difference noise sequence
with respect to $\{\cF_k\}_{k\ge0}$. Moreover,
\begin{align*}
\widehat T_{k,\mathrm{intra}}^{J}(Q_k)
-
H_{\mathrm{intra}}^JQ_k
=
w_k^{\mathrm{intra}}
+
\underbrace{
\left(
\widetilde H_{k,\mathrm{intra}}^JQ_k
-
H_{\mathrm{intra}}^JQ_k
\right)
}_{\text{approximately-greedy bias}}.
\end{align*}
Therefore, the intra-option recursion can be written as
\[
Q_{k+1}
=
Q_k
+
\alpha_k
\left(
H_{\mathrm{intra}}^JQ_k
-
Q_k
+
w_k^{\mathrm{intra}}
+
B_k^{\mathrm{intra}}(Q_k)
\right),
\]
where
\[
B_k^{\mathrm{intra}}(Q_k)
\triangleq
\widetilde H_{k,\mathrm{intra}}^JQ_k
-
H_{\mathrm{intra}}^JQ_k.
\]

For inter-option learning, the synchronous recursion is
\[
Q_{k+1}
=
Q_k
+
\alpha_k
\left(
\widehat T_{k,\mathrm{inter}}^{J}(Q_k)
-
Q_k
\right).
\]
Its conditional mean satisfies
\[
\E\!\left[
\widehat T_{k,\mathrm{inter}}^{J}(Q_k)
\mid
\cF_k
\right]
=
\widetilde H_{k,\mathrm{inter}}^JQ_k.
\]
Recall that $H_k$ denotes the normalized inter-option Bellman operator
obtained by replacing $(\rho,\lambda)$ with
$(\widehat\rho_k,\widehat\lambda_k)$ while retaining the population
option-transition kernel. Define
\[
w_k^{\mathrm{inter}}
\triangleq
\widehat T_{k,\mathrm{inter}}^{J,\mathrm{sync}}(Q_k)
-
\widetilde H_{k,\mathrm{inter}}^JQ_k.
\]
Then $\E[w_k^{\mathrm{inter}}\mid\cF_k]=0$, so
$\{w_k^{\mathrm{inter}}\}_{k\ge0}$ is again a martingale-difference noise
sequence. The total inter-option error decomposes as
\begin{align*}
\widehat T_{k,\mathrm{inter}}^{J}(Q_k)
-
H_{\mathrm{inter}}^JQ_k
=
w_k^{\mathrm{inter}}
+
\underbrace{
\left(
\widetilde H_{k,\mathrm{inter}}^JQ_k
-
H_k^JQ_k
\right)
}_{\text{approximately-greedy bias}}
+
\underbrace{
\left(
H_k^JQ_k
-
H_{\mathrm{inter}}^JQ_k
\right)
}_{(\widehat\rho_k,\widehat\lambda_k)\text{ bias}}.
\end{align*}
Hence,
\[
Q_{k+1}
=
Q_k
+
\alpha_k
\left(
H_{\mathrm{inter}}^JQ_k
-
Q_k
+
w_k^{\mathrm{inter}}
+
B_{k,\mathrm{greedy}}^{\mathrm{inter}}(Q_k)
+
B_{k,\mathrm{est}}^{\mathrm{inter}}(Q_k)
\right),
\]
where
\[
B_{k,\mathrm{greedy}}^{\mathrm{inter}}(Q_k)
\triangleq
\widetilde H_{k,\mathrm{inter}}^JQ_k-H_k^JQ_k,
\qquad
B_{k,\mathrm{est}}^{\mathrm{inter}}(Q_k)
\triangleq
H_k^JQ_k-H_{\mathrm{inter}}^JQ_k.
\]
Thus, in both settings $w_k$ represents the fresh martingale-difference
rollout noise, while the remaining terms are predictable first-order
perturbations that are controlled separately in the finite-sample analysis.
\subsection{Finite-time bounds for synchronous intra-option Q-learning}
\label{subsec:sync-intra}
We recall that the synchronous intra-option Q-learning update is
\begin{equation}
\label{eq:intra_sync_algorithm}
Q_{k+1}
=
Q_k+\alpha_k\Bigl(\widetilde H_{\rm k, intra}^{\,J}(Q_k)-Q_k+\omega_k\Bigr),\quad 
\omega_k
\triangleq
\widehat T_{\rm k, intra}^{J}(Q_k)-\widetilde H_{\rm k, intra}^{\,J}(Q_k), \alpha_k = \frac{\alpha}{k+h}.
\end{equation}
where
\begin{align*}
\widehat T_{\rm k, intra}^{J}(Q_k)
=
\sum_{i=0}^{J-1}R_k^{(i+1)}
+
(1-B_k^{(J)})Q_k(S_k^{(J)},O_k^{(J-1)})
+
B_k^{(J)}\max_{o'\in\cO}Q_k(S_k^{(J)},o').
\end{align*}

We first state a result inspired by~\citep{chen2025nonasymptoticguaranteesaveragerewardqlearning}, which quantifies the growth of the iterates generated by Algorithm~\ref{alg:sync-intra}.

\begin{lemma}[Almost-sure growth of the synchronous intra-option iterates]
\label{le:growth-intra-sync}
For every $k\ge1$, almost surely,
\[
\|Q_k\|_\infty
\le
b_k,
\]
where
\[
b_k
\triangleq
\|Q_0\|_\infty
+
\alpha J
\left(
\frac1h
+
\log\left(\frac{k+h}{h}\right)
\right)
=
\cO(\log k).
\]
\end{lemma}

\begin{proof}
Since the primitive rewards are bounded in $[0,1]$, the $J$-step
intra-option estimator in~\eqref{eq:sync_intra} satisfies, pathwise,
\[
\min_{s',o'}Q_k(s',o')
\le
\widehat{T}_{\rm k,intra}^{J}(Q_k)(s,o)
\le
J+\max_{s',o'}Q_k(s',o')
\]
for every $(s,o)$. Indeed, the accumulated reward over the $J$ steps
lies in $[0,J]$, while the terminal contribution is either a single
Q-value or the maximum over the option values at the terminal state.

The synchronous update is
\[
Q_{k+1}(s,o)
=
(1-\alpha_k)Q_k(s,o)
+
\alpha_k
\widehat T_{\rm k, intra}^{J}(Q_k)(s,o),
\]
where $\alpha_k\in[0,1]$. Hence, for every $(s,o)$,
\[
Q_{k+1}(s,o)
\le
\max_{s',o'}Q_k(s',o')
+
J\alpha_k, \quad 
Q_{k+1}(s,o)
\ge
\min_{s',o'}Q_k(s',o').
\]
Therefore,
\[
\max_{s,o}Q_{k+1}(s,o)
\le
\max_{s,o}Q_k(s,o)+J\alpha_k, \quad 
\min_{s,o}Q_{k+1}(s,o)
\ge
\min_{s,o}Q_k(s,o).
\]
It follows that
\[
\|Q_{k+1}\|_\infty
\le
\|Q_k\|_\infty+J\alpha_k.
\]

Iterating this inequality gives
\[
\|Q_k\|_\infty
\le
\|Q_0\|_\infty
+
J\sum_{t=0}^{k-1}\alpha_t.
\]
Substituting $\alpha_t=\alpha/(t+h)$,
\[
\|Q_k\|_\infty
\le
\|Q_0\|_\infty
+
\alpha J
\sum_{t=0}^{k-1}\frac1{t+h}.
\]
Using
\[
\sum_{t=0}^{k-1}\frac1{t+h}
\le
\frac1h
+
\int_0^k\frac{dx}{x+h}
=
\frac1h
+
\log\left(\frac{k+h}{h}\right),
\]
we obtain
\[
\|Q_k\|_\infty
\le
\|Q_0\|_\infty
+
\alpha J
\left(
\frac1h
+
\log\left(\frac{k+h}{h}\right)
\right)
=
b_k.
\]
Hence $b_k=\cO(\log k)$.
\end{proof}

\subsubsection{Lyapunov framework, error decomposition and drift control}
\label{subsec:async-inter-lyapunov}

We use the generalized Moreau-envelope Lyapunov function for the finite-sample analysis~\citep{chen2025nonasymptotictheoryseminormlyapunov},
\begin{align*}
M(Q)
\triangleq
\min_{u\in\mathbb R^{|\cS||\cO|}}
\left\{
\frac12\sp(u)^2
+
\frac{1}{2\theta}\|Q-u\|_q^2
\right\},
\end{align*}
for some $\theta>0$ and $q \ge 1$, which will be chosen later. Let $l_q= (|\cS||\cO|)^{-1/q}$ and $u_q =1$. Then, $l_q\|Q\|_q \le \|Q\|_\infty\le u_q\|Q\|_q$ for any $Q \in \bbR^{|\cS||\cO|}$.

The properties of the lyapunov function are stated in the following lemma.

\begin{lemma}[Properties of the Moreau envelope]\label{le:Moreau}
The following properties hold:
\begin{enumerate}
\item The function $M(\cdot)$ is convex, differentiable, and satisfies
\begin{align*}
    M(Q_2) \leq M(Q_1) + \langle \nabla M(Q_1), Q_2 - Q_1 \rangle + \frac{L}{2} \sp(Q_2 - Q_1)^2, \quad \forall\, Q_1, Q_2 \in \mathbb{R}^{|\cS||\cO|},
\end{align*}
where $L = (q-1)/(\ell_q^2 \theta)$.  
\item There exists a seminorm $p_m(\cdot)$, which satisfies $p_m(Q) = \min_{c \in \mathbb{R}} \|Q - c e\|_m$ for some norm $\|\cdot\|_m$, such that $M(Q) = p_m(Q)^2/2$. 
\item It holds that $\ell_m p_m(Q) \leq \sp(Q) \leq u_m p_m(Q)$ for all $Q \in \mathbb{R}^{|\mathcal{S}||\mathcal{O}|}$, where $\ell_m = (1 + \theta \ell_q^2)^{1/2}$ and $u_m = (1 + \theta u_q^2)^{1/2}$.
\item It holds for all $Q \in \mathbb{R}^{|\mathcal{S}||\mathcal{O}|}$ and $c \in \mathbb{R}$ that $\langle \nabla M(Q), c e \rangle = 0$.
\item It holds for all $Q_1, Q_2, Q_3 \in \mathbb{R}^{|\mathcal{S}||\cO|}$ that $\langle \nabla M(Q_1) - \nabla M(Q_2), Q_3 \rangle \leq L \sp(Q_1 - Q_2) \sp(Q_3)$.
\item $||\nabla p_m(Q)||_{m,*} \le 1$ where $||\cdot||_{m, \star}$ is the dual norm of $||\cdot||_m$
\end{enumerate}
\end{lemma}

\begin{proof}
Lemma~\ref{le:Moreau}(1)-(5) is identical to \citep[Proposition 4.1]{chen2025nonasymptoticguaranteesaveragerewardqlearning}. 

For Lemma~\ref{le:Moreau}(6), the proof of \citep[Lemma 6.5, Appendix B.3]{chen2025nonasymptoticguaranteesaveragerewardqlearning} yields the following,
\begin{align*}
|p_m(Q_1)-p_m(Q_2)|\le p_m(Q_1-Q_2) = \min_{c' \in \bbR}||Q_1-Q_2-c'e||_m \le ||Q_1-Q_2||_m
\end{align*}
Hence, $p_m(\cdot)$ is 1-Lipschitz continuous with respect to $||\cdot||_m$. Hence, \citep[Lemma 2.6]{shalev2012online} implies that $||\nabla p_m(Q)||_{m,*} \le 1$

\end{proof}


Let $M_k\triangleq M(Q_k-Q^*)$. Using \eqref{eq:intra_sync_algorithm} and Lemma \ref{le:Moreau}(1),
\begin{align}
\E[M_{k+1}]
\le{}&
\E[M_k]
+
\alpha_k
\E\!\left[
\left\langle
\nabla M(Q_k-Q^*),
\widehat T_{\rm k, intra}^J(Q_k)-Q_k
\right\rangle
\right]+
\frac{L\alpha_k^2}{2}
\E\!\left[
\sp\!\left(
\widehat T_{\rm k, intra}^J(Q_k)-Q_k
\right)^2
\right].
\label{eq:async_inter_moreau_start}
\end{align}
Adding and subtracting the ideal mean terms gives
\begin{equation}
\E[M_{k+1}]
\le
\E[M_k]
+
\alpha_k(T_{\rm drift}+T_{\rm greedy}+T_{\rm noise})
+
\frac{L\alpha_k^2}{2}T_{\rm quad }.
\label{eq:async_inter_decomposition}
\end{equation}
Here
\begin{align*}
T_{\rm drift}
&\triangleq
\E\!\left[
\left\langle
\nabla M(Q_k-Q^*),
H_{\rm intra}^J(Q_k)-Q_k
\right\rangle
\right],
\\
T_{\rm greedy}
&\triangleq
\E\!\left[
\left\langle
\nabla M(Q_k-Q^*),
\widetilde H_{\rm k, intra}^J(Q_k)-H_{\rm intra}^J(Q_k)
\right\rangle
\right],
\\
T_{\rm noise}
&\triangleq
\E\!\left[
\left\langle
\nabla M(Q_k-Q^*),
\omega_k
\right\rangle
\right],
\\
T_{\rm quad}
&\triangleq
\E\!\left[
\sp\!\left(
\widehat T_{\rm k, intra}^J(Q_k)-Q_k
\right)^2
\right].
\end{align*}

The terms have the following interpretations: $T_1$ is the deterministic contraction drift of the ideal $J$-step Bellman operator; $T_2$ is the bias incurred as part of the construction of the $J$-step estimator using approximately greedy options; $T_3$ is the error from the martingale difference noise; and $T_4$ is the quadratic noise term arising from the smoothness property of the lyapunov function. For notational clarity, we will henceforth write $\widetilde H_{\rm k, intra}$ as $\widetilde H_{k}$ in this section.

\begin{lemma}
\label{le:T1_intra_new}
Defining $\phi_1
\triangleq
1-\frac{\gamma u_m}{\ell_m}$, we get for every $k\ge1$,
\begin{align*}
T_{\rm drift}
\le
-2\phi_1\E[M_k],
\end{align*}
\end{lemma}

\begin{proof}

The proof follows from \cite[Lemma 6.5]{chen2025nonasymptoticguaranteesaveragerewardqlearning} using the contraction property of the $J$-step intra-option Bellman operator (Theorem~\ref{thm: J-intra}). Specifically,
since $H_{\rm intra}^J(Q^*)-Q^*=Jr^*\mathbf 1\in\text{ker}(\sp)$, we have,
\begin{align}
    \langle\nabla M(Q_k-Q^*),H_{\rm intra}^J(Q_k)-Q_k\rangle=\,&\langle\nabla M(Q_k-Q^*),H_{\rm intra}^J(Q_k)-H_{\rm intra}^J(Q^*)+Q^*-Q_k\rangle\nonumber\\
    =\,&\underbrace{\langle\nabla M(Q_k-Q^*),H_{\rm intra}^J(Q_k)-H_{\rm intra}^J(Q^*)\rangle}_{T_{1,1}}\nonumber\\
    &+\underbrace{\langle\nabla M(Q_k-Q^*),Q^*-Q_k\rangle}_{T_{1,2}}.\label{eq:pf:le:T1}
\end{align}
Next, we bound the terms $T_{1,1}$ and $T_{1,2}$ on the right-hand side of the previous inequality.

Using the fact that the Moreau envelope is differentiable,
\begin{align*}
    \nabla M(Q)=p_m(Q)\nabla p_m(Q),\quad \forall\,Q\in\mathbb{R}^{|\mathcal{S}||\mathcal{O}|}.
\end{align*}
It follows that for any $c\in\mathbb{R}$, we have
\begin{align*}
    T_{1,1}
    =\,&\langle\nabla M(Q_k-Q^*),H_{\rm intra}^J(Q_k)-H_{\rm intra}^J(Q^*)+ce\rangle\tag{Lemma \ref{le:Moreau} (4)}\\
    =\,&p_m(Q_k-Q^*)\langle\nabla p_m(Q_k-Q^*),H_{\rm intra}^J(Q_k)-H_{\rm intra}^J(Q^*)-ce\rangle\\
    \leq \,&p_m(Q_k-Q^*)\|\nabla p_m(Q_k-Q^*)\|_{m,*}\|H_{\rm intra}^J(Q_k)-H_{\rm intra}^J(Q^*)-ce\|_m,
\end{align*}
where $\|\cdot\|_{m,*}$ is the dual norm of $\|\cdot\|_m$. Using Lemma~\ref{le:Moreau}(6), we have
\begin{align}
    T_{1,1}\leq \,&p_m(Q_k-Q^*)\|\nabla p_m(Q_k-Q^*)\|_{m,*}\|H_{\rm intra}^J(Q_k)-H_{\rm intra}^J(Q^*)-ce\|_m\nonumber\\
    \leq \,&p_m(Q_k-Q^*)\|H_{\rm intra}^J(Q_k)-H_{\rm intra}^J(Q^*)-ce\|_m\nonumber\\
    = \,&p_m(Q_k-Q^*)p_m(H_{\rm intra}^J(Q_k)-H_{\rm intra}^J(Q^*))\tag{Choosing $c=\argmin_{c'\in\mathbb{R}}\|H_{\rm intra}^J(Q_k)-H_{\rm intra}^J(Q^*)-c'e\|_m$}\nonumber\\
    \leq  \,&\frac{1}{\ell_m}p_m(Q_k-Q^*)\sp(H_{\rm intra}^J(Q_k)-H_{\rm intra}^J(Q^*))\tag{Lemma \ref{le:Moreau} (3)}\nonumber\\
    \leq \,&\frac{\gamma_{\rm intra}}{\ell_m}p_m(Q_k-Q^*)\sp(Q_k-Q^*)\tag{Theorem~\ref{thm: J-intra}}\nonumber\\
    \leq \,&\frac{\gamma_{\rm intra} u_m}{\ell_m}p_m(Q_k-Q^*)^2\tag{Lemma \ref{le:Moreau} (3)}\nonumber\\
    =\,&\frac{2\gamma_{\rm intra} u_m}{\ell_m}M(Q_k-Q^*),\label{eq:pf:le:T11}
\end{align}
.

Next, for $T_{1,2}$, since $p_m(\cdot)$ is a convex function, we have
\begin{align*}
    p_m(0)-p_m(Q_k-Q^*)\geq \langle \nabla p_m(Q_k-Q^*),Q^*-Q_k\rangle.
\end{align*}
It follows that
\begin{align}
    T_{1,2}=\langle\nabla M(Q_k-Q^*),Q^*-Q_k\rangle
    &=p_m(Q_k-Q^*)\langle\nabla p_m(Q_k-Q^*),Q^*-Q_k\rangle\nonumber\\
    &\leq -p_m(Q_k-Q^*)^2\nonumber\\
    &=-2M(Q_k-Q^*).\label{eq:pf:le:T12}
\end{align}
Finally, combining (\ref{eq:pf:le:T11}) and (\ref{eq:pf:le:T12}) together in (\ref{eq:pf:le:T1}), we have 
\begin{align*}
    \langle\nabla M(Q_k-Q^*),H_{\rm intra}^J(Q_k)-Q_k\rangle\leq\,& T_{1,1}+T_{1,2}\\
    \leq\,&  -2\left(1-\frac{\gamma_{\rm intra} u_m}{\ell_m}\right)M(Q_k-Q^*)\\
    =\,&  -2\phi_1 M(Q_k-Q^*).
\end{align*}
Writing $\gamma = \gamma_{\rm intra}$ and taking expectations on both sides of the inequality finishes the proof.
\end{proof}

\subsubsection{Bounding $T_{\rm greedy}$}

We first recall the notation used in constructing
$\widetilde H_k^J$. Let $H\triangleq H_{\rm intra}$ denote the
population intra-option Bellman operator and let $\widehat H_k \triangleq \widehat H_{\rm k, intra}$ For $m=1,\ldots,J-1$, define
\[
\widehat d_{k,m}(s)
\in
\arg\max_{o\in\cO}
\bigl(\widehat H_k^{\,m}Q\bigr)(s,o).
\]
Thus, when $m$ Bellman steps remain, the rollout chooses the next option
using the empirical continuation value $\widehat H_k^mQ$.

For a deterministic option-selection rule $d:\cS\to\cO$, define the
population and empirical fixed-policy Bellman operators
\begin{align*}
(H_dQ)(s,o)
&\triangleq
r(s,o)
+
\sum_{s'}p(s'|s,o)
\left[
(1-\beta(s',o))Q(s',o)
+
\beta(s',o)Q(s',d(s'))
\right],
\\
(\widehat H_{k,d}Q)(s,o)
&\triangleq
\widehat r_k(s,o)
+
\sum_{s'}\widehat p_k(s'|s,o)
\left[
(1-\beta(s',o))Q(s',o)
+
\beta(s',o)Q(s',d(s'))
\right].
\end{align*}
The optimality operators satisfy
$HQ=\max_dH_dQ$ and
$\widehat H_kQ=\max_d\widehat H_{k,d}Q$ coordinatewise. In particular,
\[
\widehat H_k^JQ
=
\widehat H_{k,\widehat d_{k,J-1}}
\cdots
\widehat H_{k,\widehat d_{k,1}}
\widehat H_kQ,
\]
while the conditional mean of the fresh rollout is
\[
\widetilde H_k^JQ
=
H_{\widehat d_{k,J-1}}
\cdots
H_{\widehat d_{k,1}}
HQ.
\]
Thus, $\widetilde H_k^JQ$ and $\widehat H_k^JQ$ use exactly the same
sequence of empirical greedy rules and differ only in whether these rules
are evaluated under the population or empirical model.

We first state a uniform one-step empirical-model bound.

\begin{lemma}[One-step empirical Bellman estimation error]
\label{lem:empirical-bellman-one-step-explicit}
For every $k\ge1$,
\[
\E\left[
\left\|
(\widehat H_k-H)Q_k
\right\|_\infty^2
\right]
\le
\frac{
C_{\rm emp}(1+b_k^2)
}{k},
\]
where
\[
C_{\rm emp}
\triangleq
4\left(
|\cS|
+
2\log(|\cS||\cO|)
+
2
\right).
\]
Moreover, for every deterministic option-selection rule $d$,
\[
\E\left[
\left\|
(\widehat H_{k,d}-H_d)Q_k
\right\|_\infty^2
\right]
\le
\frac{
C_{\rm emp}(1+b_k^2)
}{k}.
\]
\end{lemma}

\begin{proof}
First note that, for every $(s,o)$,
\begin{align*}
&\left|
((\widehat H_k-H)Q_k)(s,o)
\right|
\\
&\le
|\widehat r_k(s,o)-r(s,o)|+
\sum_{s'}
|\widehat p_k(s'\mid s,o)-p(s'\mid s,o)|\times
\left|
(1-\beta(s',o))Q_k(s',o)
+
\beta(s',o)\max_{o'}Q_k(s',o')
\right|.
\end{align*}

where,
\[
r(s,o)
=
\E[R_t^{(1)}(s,o)]
=
\sum_{a\in\cA}
\pi(a\mid s,o)r(s,a),
\]
and
\[
p(s'\mid s,o)
=
\mathbb{P}(S_t^{(1)}(s,o)=s')
=
\sum_{a\in\cA}
\pi(a\mid s,o)p(s'\mid s,a).
\]

We first bound the empirical reward error. For each $(s,o)$, observe that
\[
\widehat r_k(s,o)-r(s,o)
=
\frac1k\sum_{t=0}^{k-1}
\left(
R_t^{(1)}(s,o)-\E[R_t^{(1)}(s,o)]
\right).
\]
Since $R_t^{(1)}(s,o)\in[0,1]$, Hoeffding's lemma~\citep{10.1093/acprof:oso/9780199535255.001.0001} gives
\[
\E\left[e^{\lambda(\widehat r_k(s,o)-r(s,o))}\right]
\le
\exp\left(\frac{\lambda^2}{8k}\right),
\]
and hence
\[
\mathbb{P}\left(
|\widehat r_k(s,o)-r(s,o)|\ge t
\right)
\le
2e^{-2kt^2}.
\]
A union bound gives
\[
\mathbb{P}\left(
\max_{s,o}|\widehat r_k(s,o)-r(s,o)|\ge t
\right)
\le
2|\cS||\cO| e^{-2kt^2}.
\]
Setting
\[
t_0^2
=
\frac{\log(2|\cS||\cO|)}{2k},
\]
and using the tail-integration formula,
\begin{equation}
\begin{aligned}
\E\left[
\max_{s,o}|\widehat r_k(s,o)-r(s,o)|^2
\right]
&=
\int_0^\infty
2t\,
\mathbb{P}\left(
\max_{s,o}|\widehat r_k(s,o)-r(s,o)|\ge t
\right)
dt
\nonumber\\
&\le
t_0^2
+
\int_{t_0}^\infty
4|\cS||\cO|t e^{-2kt^2}\,dt
\nonumber\\
&=
\frac{\log(2|\cS||\cO|)+1}{2k}.
\end{aligned}
\end{equation}
Therefore,
\begin{equation}
 \label{eq:empirical-reward-bound}  \E\left[
\max_{s,o}
|\widehat r_k(s,o)-r(s,o)|^2
\right]
\le
\frac{
\log(2|\cS||\cO|)+1
}{2k}.
\end{equation}

We next bound the empirical transition error. Fix $(s,o)$ and define
\[
Z_{s,o}
\triangleq
\left\|
\widehat p_k(\cdot\mid s,o)-p(\cdot\mid s,o)
\right\|_1.
\]
Enumerating $\cS=\{s_1,\ldots,s_{|\cS|}\}$ , we have
\[
\widehat p_k(s_j|s, o)
=
\frac1k
\sum_{t=0}^{k-1}
\mathbf 1\{S_t^{(1)}(s,o)=s_j\}.
\]
Since $\E[\widehat p_k(\cdot|s,o)] = p(\cdot|s, o)$, hence
\[
\E[(\widehat p_k(s_j|s, o)-p(s_j|s, o))^2]
=
\frac{p(s_j|s, o)(1-p(s_j|s, o))}{k},
\]
and therefore
\begin{align*}
\E\left[
\|\widehat p_k(\cdot\mid s,o)-p(\cdot\mid s,o)\|_2^2
\right]
&=
\frac1k
\sum_{j=1}^{|\cS|}
p(s_j|s, o)(1-p(s_j|s, o))
\\
&\le
\frac{1-\sum_{j=1}^{|\cS|}p(s_j|s, o)^2}{k} \le \frac1k.
\end{align*}
Hence, by $\|x\|_1\le\sqrt{|\cS|}\|x\|_2$ and Jensen's inequality,
\begin{align*}
\E[Z_{s,o}]
&\le
\sqrt{|\cS|}
\E\left[
\|\widehat p_k(\cdot\mid s,o)-p(\cdot\mid s,o)\|_2
\right]
\\
&\le
\sqrt{
|\cS|
\E\left[
\|\widehat p_k(\cdot\mid s,o)-p(\cdot\mid s,o)\|_2^2
\right]
}
\\
&\le
\sqrt{\frac{|\cS|}{k}}.
\end{align*}

We now control deviations of $Z_{s,o}$ around its mean. Replacing one
transition sample in the empirical distribution can change
$\widehat p_k(\cdot\mid s,o)$ by at most $2/k$ in $\ell_1$ norm.
Therefore, by McDiarmid's inequality~\citep{10.1093/acprof:oso/9780199535255.001.0001},
\[
\mathbb{P}\left(
Z_{s,o}-\E[Z_{s,o}]
\ge t
\right)
\le
\exp\left(
-\frac{kt^2}{2}
\right).
\]
Let
\[
Y_{s,o}
\triangleq
\bigl(
Z_{s,o}-\E[Z_{s,o}]
\bigr)_+ \implies \{Y_{s,o} \ge t\} = \{Z_{s,o}-\E[Z_{s,o}] \ge t\}.
\]
A union bound gives
\[
\mathbb{P}\left(
\max_{s,o}Y_{s,o}\ge t
\right)
\le
|\cS||\cO|\exp\left(
-\frac{kt^2}{2}
\right).
\]
Setting
\[
t_0^2
=
\frac{2\log |\cS||\cO|}{k},
\]
and integrating the tail,
\begin{align*}
\E\left[
\max_{s,o}Y_{s,o}^2
\right]
&=
\int_0^\infty
2t\,
\mathbb{P}\left(
\max_{s,o}Y_{s,o}\ge t
\right)dt
\\
&\le
t_0^2
+
\int_{t_0}^\infty
2t|\cS||\cO|
\exp\left(
-\frac{kt^2}{2}
\right)dt
\\
&=
\frac{2\log |\cS||\cO|}{k}
+
\frac{2}{k}.
\end{align*}
Since
\[
\max_{s,o}Z_{s,o}
\le
\sqrt{\frac{|\cS|}{k}}
+
\max_{s,o}Y_{s,o},
\]
using $(a+b)^2\le2a^2+2b^2$ yields
\begin{align}
\E\left[
\max_{s,o}
\left\|
\widehat p_k(\cdot\mid s,o)-p(\cdot\mid s,o)
\right\|_1^2
\right]
&\le
\frac{2|\cS|}{k}
+
2\E\left[
\max_{s,o}Y_{s,o}^2
\right]
\nonumber\\
&\le
\frac{
2|\cS|+4\log(|\cS||\cO|)+4
}{k}.
\label{eq:empirical-transition-bound}
\end{align}

Since $0\le\beta(s',o)\le1$,
\[
\left|
(1-\beta(s',o))Q_k(s',o)
+
\beta(s',o)\max_{o'}Q_k(s',o')
\right|
\le
\|Q_k\|_\infty
\le b_k.
\]
Therefore,
\[
\left|
((\widehat H_k-H)Q_k)(s,o)
\right|
\le
|\widehat r_k(s,o)-r(s,o)|
+
b_k
\|\widehat p_k(\cdot\mid s,o)-p(\cdot\mid s,o)\|_1.
\]
Taking the maximum over $(s,o)$ and using $(a+b)^2\le2a^2+2b^2$,
\begin{align*}
\left\|
(\widehat H_k-H)Q_k
\right\|_\infty^2
&\le
2
\max_{s,o}
|\widehat r_k(s,o)-r(s,o)|^2+
2b_k^2
\max_{s,o}
\|\widehat p_k(\cdot\mid s,o)-p(\cdot\mid s,o)\|_1^2.
\end{align*}
Taking expectations and applying
\eqref{eq:empirical-reward-bound} and
\eqref{eq:empirical-transition-bound} yields
\begin{align*}
\E\left[
\left\|
(\widehat H_k-H)Q_k
\right\|_\infty^2
\right]
&\le
\frac{\log(2|\cS||\cO|)+1}{2k}
+
\frac{(2|\cS|+4\log(|\cS||\cO|)+4)b_k^2}{k}.
\end{align*}
Hence
\[
\E\left[
\left\|
(\widehat H_k-H)Q_k
\right\|_\infty^2
\right]
\le
\frac{
4(|\cS|+2\log(|\cS||\cO|)+2)(1+b_k^2)
}{k}.
\]
\end{proof}

\begin{lemma}[Population-rollout versus empirical-model error]
\label{lem:tilde-vs-empirical-intra}
For every $k\ge1$,
\[
\E\left[
\left\|
\widetilde H_k^JQ_k-\widehat H_k^JQ_k
\right\|_\infty^2
\right]
\le
\frac{
C_{\rm emp}J^2
\left(1+(b_k+J)^2\right)
}{k}.
\]
\end{lemma}

\begin{proof}
For $m=1,\ldots,J-1$, define
\[
\widehat d_{k,m}(s)
\in
\arg\max_{o\in\cO}
\bigl(\widehat H_k^mQ_k\bigr)(s,o).
\]
By construction,
\[
\widetilde H_k^JQ_k
=
H_{\widehat d_{k,J-1}}
\cdots
H_{\widehat d_{k,1}}
HQ_k,
\]
whereas
\[
\widehat H_k^JQ_k
=
\widehat H_{k,\widehat d_{k,J-1}}
\cdots
\widehat H_{k,\widehat d_{k,1}}
\widehat H_kQ_k.
\]

We first verify sup-norm non-expansiveness. For any deterministic decision
rule $d$ and any $Q,Q'$,
\begin{align*}
|(H_dQ-H_dQ')(s,o)|
&\le
\sum_{s'}p(s'|s,o)
\left[
(1-\beta(s',o))|Q(s',o)-Q'(s',o)|
+
\beta(s',o)|Q(s',d(s'))-Q'(s',d(s'))|
\right]\\
&\le
\|Q-Q'\|_\infty.
\end{align*}
Hence
\[
\|H_dQ-H_dQ'\|_\infty
\le
\|Q-Q'\|_\infty.
\]
Replacing $p$ by $\widehat p_k$ gives
\[
\|\widehat H_{k,d}Q-\widehat H_{k,d}Q'\|_\infty
\le
\|Q-Q'\|_\infty.
\]
Similarly,
\[
\|HQ-HQ'\|_\infty
\le
\|Q-Q'\|_\infty,
\qquad
\|\widehat H_kQ-\widehat H_kQ'\|_\infty
\le
\|Q-Q'\|_\infty,
\]
since
\[
\left|
\max_{o'}Q(s',o')
-
\max_{o'}Q'(s',o')
\right|
\le
\|Q-Q'\|_\infty.
\]

For $r=0,\ldots,J$, define the hybrid iterates
\[
Z_r
\triangleq
H_{\widehat d_{k,J-1}}
\cdots
H_{\widehat d_{k,r}}
\widehat H_{k,\widehat d_{k,r-1}}
\cdots
\widehat H_{k,\widehat d_{k,1}}
\widehat H_kQ_k,
\]
with the convention that empty compositions are omitted. Thus
\[
Z_0
=
\widetilde H_k^JQ_k,
\qquad
Z_J
=
\widehat H_k^JQ_k.
\]
Therefore,
\[
\left\|
\widetilde H_k^JQ_k-\widehat H_k^JQ_k
\right\|_\infty
\le
\sum_{r=0}^{J-1}
\|Z_r-Z_{r+1}\|_\infty.
\]

For $r=0$, we have
\[
Z_0
=
H_{\widehat d_{k,J-1}}
\cdots
H_{\widehat d_{k,1}}
HQ_k,
\]
and
\[
Z_1
=
H_{\widehat d_{k,J-1}}
\cdots
H_{\widehat d_{k,1}}
\widehat H_kQ_k.
\]
By non-expansiveness of the outer population fixed-policy operators,
\[
\|Z_0-Z_1\|_\infty
\le
\|(\widehat H_k-H)Q_k\|_\infty.
\]

For $1\le r\le J-1$, define
\[
V_r
\triangleq
\widehat H_{k,\widehat d_{k,r-1}}
\cdots
\widehat H_{k,\widehat d_{k,1}}
\widehat H_kQ_k.
\]
Then
\[
Z_r
=
H_{\widehat d_{k,J-1}}
\cdots
H_{\widehat d_{k,r+1}}
H_{\widehat d_{k,r}}V_r,
\]
while
\[
Z_{r+1}
=
H_{\widehat d_{k,J-1}}
\cdots
H_{\widehat d_{k,r+1}}
\widehat H_{k,\widehat d_{k,r}}V_r.
\]
Again by non-expansiveness of the outer operators,
\[
\|Z_r-Z_{r+1}\|_\infty
\le
\left\|
\left(
\widehat H_{k,\widehat d_{k,r}}
-
H_{\widehat d_{k,r}}
\right)V_r
\right\|_\infty.
\]

Since rewards are bounded in $[0,1]$,
\[
\|\widehat H_kQ\|_\infty
\le
1+\|Q\|_\infty,
\qquad
\|\widehat H_{k,d}Q\|_\infty
\le
1+\|Q\|_\infty.
\]
Hence, using $\|Q_k\|_\infty\le b_k$ almost surely,
\[
\|V_r\|_\infty
\le
b_k+r
\le
b_k+J.
\]

Let
\[
e_{r,k}
\triangleq
\max_{s,o}
|\widehat r_k(s,o)-r(s,o)|,
\qquad
e_{p,k}
\triangleq
\max_{s,o}
\|\widehat p_k(\cdot|s,o)-p(\cdot|s,o)\|_1.
\]
For the first hybrid difference,
\[
\|(\widehat H_k-H)Q_k\|_\infty
\le
e_{r,k}+b_ke_{p,k}.
\]
For every $1\le r\le J-1$,
\[
\left\|
\left(
\widehat H_{k,\widehat d_{k,r}}
-
H_{\widehat d_{k,r}}
\right)V_r
\right\|_\infty
\le
e_{r,k}+\|V_r\|_\infty e_{p,k}
\le
e_{r,k}+(b_k+J)e_{p,k}.
\]
Therefore,
\[
\left\|
\widetilde H_k^JQ_k-\widehat H_k^JQ_k
\right\|_\infty
\le
J\left(
e_{r,k}+(b_k+J)e_{p,k}
\right).
\]
Squaring gives
\[
\left\|
\widetilde H_k^JQ_k-\widehat H_k^JQ_k
\right\|_\infty^2
\le
J^2
\left(
e_{r,k}+(b_k+J)e_{p,k}
\right)^2.
\]

By the argument of Lemma~\ref{lem:empirical-bellman-one-step-explicit}, replacing the a.s. upper bound of $\|Q_k\|_\infty \le b_k$ with $(b_k+J)$,
\[
\E\left[
\left(
e_{r,k}+(b_k+J)e_{p,k}
\right)^2
\right]
\le
\frac{
C_{\rm emp}
\left(
1+(b_k+J)^2
\right)
}{k}.
\]
Hence
\[
\E\left[
\left\|
\widetilde H_k^JQ_k-\widehat H_k^JQ_k
\right\|_\infty^2
\right]
\le
\frac{
C_{\rm emp}J^2
\left(
1+(b_k+J)^2
\right)
}{k},
\]
which proves the result.
\end{proof}
\begin{lemma}[Empirical versus population $J$-step Bellman error]
\label{lem:empirical-vs-population-intra}
For every $k\ge1$,
\[
\E\left[
\left\|
\widehat H_k^JQ_k-H^JQ_k
\right\|_\infty^2
\right]
\le
\frac{
C_{\rm emp}J^2
\left(1+(b_k+J)^2\right)
}{k}.
\]
\end{lemma}

\begin{proof}
Define
\[
Z_r
\triangleq
\widehat H_k^{\,J-r}H^rQ_k,
\qquad
r=0,\ldots,J.
\]
Then $Z_0=\widehat H_k^JQ_k$ and $Z_J=H^JQ_k$. Therefore,
\[
\left\|
\widehat H_k^JQ_k-H^JQ_k
\right\|_\infty
\le
\sum_{r=0}^{J-1}
\|Z_r-Z_{r+1}\|_\infty.
\]
By the sup-norm non-expansiveness of $\widehat H_k$,
\[
\|Z_r-Z_{r+1}\|_\infty
\le
\left\|
(\widehat H_k-H)H^rQ_k
\right\|_\infty.
\]
Moreover,
\[
\|H^rQ_k\|_\infty
\le
\|Q_k\|_\infty+r
\le
b_k+J.
\]
Hence, using Lemma~\ref{lem:empirical-bellman-one-step-explicit} and Cauchy--Schwarz,
\[
\E\left[
\left\|
\widehat H_k^JQ_k-H^JQ_k
\right\|_\infty^2
\right]
\le
\frac{
C_{\rm emp}J^2
\left(1+(b_k+J)^2\right)
}{k}.
\]
\end{proof}

\begin{lemma}[Finite-$J$ empirical-oracle error]
\label{lem:empirical-greedy-J}
For every $k\ge1$,
\[
\E\left[
\sp\left(
\widetilde H_k^JQ_k-H^JQ_k
\right)^2
\right]
\le
\frac{
4C_{\rm emp}J^2
\left(1+(b_k+J)^2\right)
}{k}.
\]
Consequently, since $b_k=\cO(J\log k)$,
\[
\E\left[
\sp\left(
\widetilde H_k^JQ_k-H^JQ_k
\right)^2
\right]
=
\widetilde{\cO}\left(
\frac{J^4}{k}
\right).
\]
\end{lemma}

\begin{proof}
Combining Lemma~\ref{lem:tilde-vs-empirical-intra} and \ref{lem:empirical-vs-population-intra}, and using $\sp(\cdot) \le \|\cdot\|_\infty$ yields Lemma~\ref{lem:empirical-greedy-J}.
\end{proof}

\begin{proposition}[Control of the approximately greedy perturbation]
\label{prop:T-greedy}
Recall
\[
T_{\rm greedy}
\triangleq
\E\left[
\left\langle
\nabla M(Q_k-Q^*),
\widetilde H_k^JQ_k-H^JQ_k
\right\rangle
\right].
\]
Then, for every $k\ge1$,
\[
T_{\rm greedy}
\le
\frac{\phi_1}{2}m_k
+
\frac{
4C_{\rm emp}J^2
\left(1+(b_k+J)^2\right)
}{
\phi_1\ell_m^2k
},
\]
where $m_k\triangleq\E[M(Q_k-Q^*)]$. In particular,
\[
T_{\rm greedy}
\le
\frac{\phi_1}{2}m_k
+
\widetilde{\cO}\left(
\frac{J^4}{k}
\right).
\]
\end{proposition}

\begin{proof}

Recall that $M(Q) = \frac12 p_m(Q)^2$ for some seminorm $p_m$ (Lemma~\ref{le:Moreau}). Hence, $\nabla M(Q) = p_m(Q)\nabla p_m(Q)$. Hence, we have
\begin{align*}
\left\langle
\nabla M(Q_k-Q^*),
\widetilde H_k^JQ_k-H^JQ_k
\right\rangle
&\le p_m(Q_k-Q^*)\left\langle\nabla p_m(Q_k-Q^*), \widetilde H_k^JQ_k-H^JQ_k\right\rangle\\
&\le p_m(Q_k-Q^*) \|\nabla p_m(Q_k-Q^*)\|_{m, *}\|\widetilde H_k^JQ_k-H^JQ_k\|_m \tag{Cauchy-schwarz inequality} 
\\&\le \frac{
\sqrt{2M(Q_k-Q^*)}
}{
\ell_m
}
\sp\left(
\widetilde H_k^JQ_k-H^JQ_k
\right).\tag{Lemma~\ref{le:Moreau}(3)(6)}
\end{align*}

Taking expectations and applying the AM-GM inequality with parameter
$\phi_1/2$ gives
\begin{align*}
T_{\rm greedy}
&\le
\frac{\phi_1}{2}m_k
+
\frac1{\phi_1\ell_m^2}
\E\left[
\sp\left(
\widetilde H_k^JQ_k-H^JQ_k
\right)^2
\right]
\\
&\le
\frac{\phi_1}{2}m_k
+
\frac{
4C_{\rm emp}J^2
\left(1+(b_k+J)^2\right)
}{
\phi_1\ell_m^2k
},
\end{align*}
where the final inequality follows from
Lemma~\ref{lem:empirical-greedy-J}.
Since $b_k=\cO(J\log k)$, the second term is
$\widetilde{\cO}(J^4/k)$.
\end{proof}
\subsubsection{Bounding $T_{\rm noise}$ and $T_{\rm quad}$}

\begin{proposition}
	\label{prop:T_noise_intra}
	It holds for any $k \geq 1$ that
	\begin{align*}
		T_{\rm noise} = 0.
	\end{align*}
\end{proposition}

\begin{proof}
By construction,
$\widetilde H_{k,\rm intra}^{\,J}(Q_k)
=
\E[\widehat T_{\rm k, intra}^{J}(Q_k)|\mathcal F_k]$.
Hence $\E[\omega_k|\mathcal F_k]=0$.
Since $\nabla M(Q_k-Q^*)$ is $\mathcal F_k$-measurable,
\[
T_{\rm noise}
=
\E\left[
\left\langle
\nabla M(Q_k-Q^*),
\E[\omega_k|\mathcal F_k]
\right\rangle
\right]
=
0.
\]
\end{proof}

\begin{proposition}[Control of the quadratic term]
\label{prop:T-quad-intra}
Recall
\[
T_{\rm quad}
\triangleq
\E\left[
\sp\left(
\widehat T_{\rm k, intra}^{J}(Q_k)-Q_k
\right)^2
\right].
\]
Then $\forall k\ge 1$,
\[
T_{\rm quad}
\le
16u_m^2m_k
+
2\left(
J+2\sp(Q^*)
\right)^2,
\]
where $m_k\triangleq\E[M(Q_k-Q^*)]$. Consequently,
\[
\frac{L\alpha_k^2}{2}T_{\rm quad}
\le
8Lu_m^2\alpha_k^2m_k
+
L\alpha_k^2
\left(
J+2\sp(Q^*)
\right)^2.
\]
In particular, if
\[
\alpha_k
\le
\frac{\phi_1}{16Lu_m^2},
\]
then
\[
\frac{L\alpha_k^2}{2}T_{\rm quad}
\le
\frac{\phi_1}{2}\alpha_km_k
+
L\alpha_k^2
\left(
J+2\sp(Q^*)
\right)^2.
\]
Thus, the first term can be absorbed into the negative drift, while the
remaining contribution is $\cO(\alpha_k^2)$.
\end{proposition}

\begin{proof}
Since rewards are bounded in $[0,1]$, for every $(s,o)$,
\[
\min_{s',o'}Q_k(s',o')
\le
\widehat T_{\rm k, intra}^{J}(Q_k)(s,o)
\le
J+\max_{s',o'}Q_k(s',o').
\]
Hence
\begin{align*}
\widehat T_{\rm k, intra}^{J}(Q_k)(s,o)-Q_k(s,o)
&\le
J+\max_{s',o'}Q_k(s',o')
-\min_{s',o'}Q_k(s',o')
\\
&=
J+2\sp(Q_k),
\end{align*}
while
\begin{align*}
\widehat T_{\rm k, intra}^{J}(Q_k)(s,o)-Q_k(s,o)
&\ge
\min_{s',o'}Q_k(s',o')
-\max_{s',o'}Q_k(s',o')
\\
&=
-2\sp(Q_k).
\end{align*}
Therefore,
\begin{align*}
\sp\left(
\widehat T_{\rm k, intra}^{J}(Q_k)-Q_k
\right)
&\le
\frac{
J+4\sp(Q_k)
}{2}\le
J+2\sp(Q_k).
\end{align*}
It follows that
\begin{align*}
T_{\rm quad}
&\le
\E\left[
\left(
J+2\sp(Q_k)
\right)^2
\right]
\\
&\le
\E\left[
\left(
J
+
2\sp(Q_k-Q^*)
+
2\sp(Q^*)
\right)^2
\right]
\\
&\le
8\E\left[
\sp(Q_k-Q^*)^2
\right]
+
2\left(
J+2\sp(Q^*)
\right)^2.
\end{align*}
Using Lemma~\ref{le:Moreau}, we have
\[
\E\left[
\sp(Q_k-Q^*)^2
\right]
\le
2u_m^2m_k.
\]
Therefore,
\[
T_{\rm quad}
\le
16u_m^2m_k
+
2\left(
J+2\sp(Q^*)
\right)^2.
\]
Multiplying by $L\alpha_k^2/2$ gives
\[
\frac{L\alpha_k^2}{2}T_{\rm quad}
\le
8Lu_m^2\alpha_k^2m_k
+
L\alpha_k^2
\left(
J+2\sp(Q^*)
\right)^2.
\]
Finally, if
\[
\alpha_k
\le
\frac{\phi_1}{16Lu_m^2} \implies
8Lu_m^2\alpha_k^2
\le
\frac{\phi_1}{2}\alpha_k,
\]
which proves the result.
\end{proof}

\subsubsection{Unrolling the Lyapunov recursion}
\label{sec:intra-unrolling}

We now combine the preceding bounds and instantiate the generalized
Moreau envelope for the span seminorm.

Following \cite{chen2025nonasymptoticguaranteesaveragerewardqlearning, chen2025nonasymptotictheoryseminormlyapunov}, choose
\[
q\triangleq2\log (|\cS||\cO|),
\qquad
\theta
\triangleq
\left(
\frac{1+\gamma}{2\gamma}
\right)^2-1.
\]
where $\gamma$ is the contraction of the $J$-step intra-option Bellman operator (Theorem~\ref{thm: J-intra}).
Then
\[
\ell_{q}\|x\|_q
\le
\|x\|_\infty
\le
u_{q}\|x\|_q,
\qquad
\ell_{q}=e^{-1/2},
\qquad
u_{q}=1.
\]
Consequently,
\[
\ell_m
=
\sqrt{1+\frac{\theta}{e}},
\qquad
u_m
=
\sqrt{1+\theta}.
\]

Moreover, since $x\mapsto \|x\|_q^2/2$ is $(q-1)$-smooth with
respect to $\|\cdot\|_q$, the corresponding Moreau envelope is
$L$-smooth with respect to the span seminorm with
\[
L
=
\frac{e(q-1)}{\theta}.
\]
In particular,
\begin{align}
\frac{u_m^2}{\ell_m^2}
&=
\frac{1+\theta}{1+\theta/e}
\le e,
\label{eq:moreau-ratio}
\\
\gamma\frac{u_m}{\ell_m}
&\le
\gamma\sqrt{1+\theta}
=
\frac{1+\gamma}{2},
\label{eq:moreau-drift}
\end{align}
and therefore
\[
\phi_1
\triangleq
1-\gamma\frac{u_m}{\ell_m}
\ge
\frac{1-\gamma}{2}.
\]
Furthermore,
\begin{align}
Lu_m^2
&=
e(q-1)\frac{1+\theta}{\theta}
\nonumber\\
&\le
\frac{4e\log (|\cS||\cO|)}{1-\gamma}.
\label{eq:Lum-bound}
\end{align}

\begin{proposition}[One-step Lyapunov recursion]
\label{prop:intra-one-step-final}
Let
\[
m_k
\triangleq
\E[M(Q_k-Q^*)].
\]
Suppose
\[
\alpha_k
\le
\frac{\phi_1}{16Lu_m^2}.
\]
Then, for every $k\ge1$,
\begin{align}
m_{k+1}
\le{}&
(1-\phi_1\alpha_k)m_k
+
L\alpha_k^2
\left(
J+2\sp(Q^*)
\right)^2
\nonumber\\
&+
\frac{
4\alpha_k C_dJ^2 
\left(1+(b_k+J)^2\right)
}{
\phi_1\ell_m^2 k
},
\label{eq:intra-final-recursion}
\end{align}
where
\[
C_d
\triangleq
|\cS|
+
2\log (|\cS||\cO|)+2,
\]
and
\[
b_k
=
\|Q_0\|_\infty
+
\alpha J
\left(
\frac1h
+
\log\left(\frac{k+h}{h}\right)
\right).
\]
\end{proposition}

\begin{proof}
The Moreau smoothness inequality gives
\begin{align*}
m_{k+1}
\le
m_k
+
\alpha_k
\left(
T_{\rm drift}
+
T_{\rm noise}
+
T_{\rm greedy}
\right)
+
\frac{L\alpha_k^2}{2}
T_{\rm quad}.
\end{align*}
By Lemma~\ref{le:T1_intra_new},
\[
T_{\rm drift}
\le
-2\phi_1m_k.
\]
Proposition~\ref{prop:T-greedy} gives
\[
T_{\rm greedy}
\le
\frac{\phi_1}{2}m_k
+
\frac{
4C_dJ^2 
\left(1+(b_k+J)^2\right)
}{
\phi_1\ell_m^2 k
},
\]
while Proposition~\ref{prop:T_noise_intra} gives
\[
T_{\rm noise}=0.
\]
Finally, Proposition~\ref{prop:T-quad-intra} yields
\[
\frac{L\alpha_k^2}{2}T_{\rm quad}
\le
\frac{\phi_1}{2}\alpha_km_k
+
L\alpha_k^2
\left(
J+2\sp(Q^*)
\right)^2.
\]
Combining the four bounds gives
\[
-2\phi_1
+
\frac{\phi_1}{2}
+
\frac{\phi_1}{2}
=
-\phi_1,
\]
which proves~\eqref{eq:intra-final-recursion}.
\end{proof}

We now state the full version of Theorem~\ref{thm: sync-intra-ftb}.

\begin{theorem}[Synchronous intra-option Q-learning]
\label{thm:intra-final}
Let
\[
\alpha_k
=
\frac{\alpha}{k+h},
\qquad
\alpha
=
\frac{4}{1-\gamma},
\]
and suppose
\begin{align}
h
\ge
\frac{
512e\log (|\cS||\cO|)
}{
(1-\gamma)^3
}.
\label{eq:intra-h-condition}
\end{align}
Then the stepsize condition of
Proposition~\ref{prop:intra-one-step-final} is satisfied.
Moreover, for every $k\ge2$,
\begin{align}
\E\left[
\sp(Q_k-Q^*)^2
\right]
\le{}&
e\Delta_1^2
\frac{
h(h-1)
}{
(k+h-2)(k+h-1)
}
\nonumber\\
&+
\frac{
128e\log (|\cS||\cO|)
}{
(1-\gamma)^3
}
\left(
J+2\sp(Q^*)
\right)^2
\frac{k-1}{
(k+h-2)(k+h-1)
}
\nonumber\\
&+
\frac{
64eC_dJ^2 
}{
(1-\gamma)^2
}
\left(
1+(b_{k-1}+J)^2
\right)\times
\frac{
(1+\log(k))
}{
(k+h-2)
},
\label{eq:intra-explicit-final}
\end{align}
where
\[
\Delta_1
\triangleq
\sp(Q_0-Q^*)
+
\frac{\alpha}{h}
\left(
J+2\sp(Q_0)
\right).
\]
Consequently,
\begin{align}
\E\left[
\sp(Q_k-Q^*)^2
\right]
=
\widetilde{\cO}\left(\frac1k\right).
\label{eq:intra-final-rate}
\end{align}
More explicitly,
\begin{align}
\E\left[
\sp(Q_k-Q^*)^2
\right]
=
\cO\left(
\frac{
\left(J+\sp(Q^*)\right)^2\log (|\cS||\cO|)
}{
(1-\gamma)^3k
}
+
\frac{
C_dJ^2 
\left(1+(b_k+J)^2\right)
}{
(1-\gamma)^2k
}
\right),
\label{eq:intra-final-order}
\end{align}
up to the faster-decaying initialization term.
\end{theorem}

\begin{proof}
By~\eqref{eq:Lum-bound} and
$\phi_1\ge(1-\gamma)/2$,
\[
\frac{\phi_1}{16Lu_m^2}
\ge
\frac{(1-\gamma)^2}{
128e\log (|\cS||\cO|)
}.
\]
Since
\[
\alpha_k
\le
\frac{\alpha}{h}
=
\frac{4}{(1-\gamma)h},
\]
condition~\eqref{eq:intra-h-condition} ensures
\[
\alpha_k
\le
\frac{\phi_1}{16Lu_m^2}.
\]

Moreover,
\[
\alpha\phi_1
\ge
\frac{4}{1-\gamma}
\frac{1-\gamma}{2}
=
2.
\]
Therefore, from
Proposition~\ref{prop:intra-one-step-final},
\begin{align}
m_{k+1}
\le{}&
\left(
1-\frac{2}{k+h}
\right)m_k
+
\frac{B}{(k+h)^2}
+
\frac{
G\left(1+(b_k+J)^2\right)
}{
k(k+h)
},
\label{eq:intra-recursion-simple}
\end{align}
where
\[
B
\triangleq
L\alpha^2
\left(
J+2\sp(Q^*)
\right)^2
\]
and
\[
G
\triangleq
\frac{
4\alpha C_dJ^2 
}{
\phi_1\ell_m^2
}.
\]

For $1\le i\le k-1$,
\begin{align*}
\prod_{j=i+1}^{k-1}
\left(
1-\frac{2}{j+h}
\right)
&=
\prod_{j=i+1}^{k-1}
\frac{j+h-2}{j+h}
\\
&=
\frac{
(i+h-1)(i+h)
}{
(k+h-2)(k+h-1)
}.
\end{align*}
Hence, unrolling~\eqref{eq:intra-recursion-simple} from iteration $1$
gives
\begin{align}
m_k
\le{}&
m_1
\frac{
h(h-1)
}{
(k+h-2)(k+h-1)
}
\nonumber\\
&+
\frac{B}{
(k+h-2)(k+h-1)
}
\sum_{i=1}^{k-1}
\frac{i+h-1}{i+h}
\nonumber\\
&+
\frac{G}{
(k+h-2)(k+h-1)
}
\sum_{i=1}^{k-1}
\left(
1+(b_i+J)^2
\right)
\frac{i+h-1}{i}.
\label{eq:intra-unrolled}
\end{align}
Since
\[
\frac{i+h-1}{i+h}\le1,
\]
the second sum satisfies
\[
\sum_{i=1}^{k-1}
\frac{i+h-1}{i+h}
\le
k-1.
\]
Furthermore, $b_i$ is nondecreasing, so
\begin{align*}
&
\sum_{i=1}^{k-1}
\left(
1+(b_i+J)^2
\right)
\frac{i+h-1}{i}
\\
&\le
\left(
1+(b_{k-1}+J)^2
\right)
\sum_{i=1}^{k-1}
\left(
1+\frac{h-1}{i}
\right)
\\
&\le
\left(
1+(b_{k-1}+J)^2
\right)
\left[
(k-1)
+
(h-1)(1+\log(k-1))
\right].
\end{align*}
Thus,
\begin{align}
m_k
\le{}&
m_1
\frac{
h(h-1)
}{
(k+h-2)(k+h-1)
}
\nonumber\\
&+
B
\frac{k-1}{
(k+h-2)(k+h-1)
}
\nonumber\\
&+
G
\left(
1+(b_{k-1}+J)^2
\right)
\frac{
(k-1)+(h-1)(1+\log(k-1))
}{
(k+h-2)(k+h-1)
}.
\label{eq:intra-m-bound}
\end{align}

It remains to convert the Moreau error to the span error. By norm
equivalence,
\[
\sp(Q_k-Q^*)^2
\le
2u_m^2M(Q_k-Q^*),
\]
and hence
\[
\E[\sp(Q_k-Q^*)^2]
\le
2u_m^2m_k.
\]
Also,
\[
M(Q_1-Q^*)
\le
\frac{
\sp(Q_1-Q^*)^2
}{
2\ell_m^2
}.
\]
Since
\[
\sp(Q_1-Q^*)
\le
\sp(Q_0-Q^*)
+
\alpha_0
\sp\left(
\widehat T_{\rm k, intra}^J(Q_0)-Q_0
\right),
\]
Proposition~\ref{prop:T-quad-intra} and the corresponding pathwise bound give
\[
\sp(Q_1-Q^*)
\le
\Delta_1.
\]
Therefore, using~\eqref{eq:moreau-ratio},
\[
2u_m^2m_1
\le
\frac{u_m^2}{\ell_m^2}
\Delta_1^2
\le
e\Delta_1^2.
\]

For the quadratic contribution,
\begin{align*}
2u_m^2B
&=
2Lu_m^2\alpha^2
\left(
J+2\sp(Q^*)
\right)^2
\\
&\le
\frac{
128e\log (|\cS||\cO|)
}{
(1-\gamma)^3
}
\left(
J+2\sp(Q^*)
\right)^2.
\end{align*}
Finally,
\begin{align*}
2u_m^2G
&=
\frac{
8\alpha C_dJ^2 u_m^2
}{
\phi_1\ell_m^2
}
\\
&\le
\frac{
64eC_dJ^2 
}{
(1-\gamma)^2
},
\end{align*}
where we used
$u_m^2/\ell_m^2\le e$,
$\phi_1\ge(1-\gamma)/2$, and
$\alpha=4/(1-\gamma)$.
Substituting these estimates into~\eqref{eq:intra-m-bound}
proves~\eqref{eq:intra-explicit-final}.

Finally, since
\[
b_k
=
\|Q_0\|_\infty
+
\cO\left(
\frac{J}{1-\gamma}\log k
\right),
\]
the empirical-greedy contribution is
$\widetilde{\cO}(1/k)$, while the initialization term is
$\cO(k^{-2})$ and the quadratic term is $\cO(k^{-1})$.
This proves~\eqref{eq:intra-final-rate}.
\end{proof}

\begin{corollary}[Sample complexity of synchronous intra-option Q-learning]
\label{cor:sample-complexity-intra}
Suppose Assumptions~\ref{assm:smdp} and~\ref{ass:intra-assump} hold.
Consider synchronous intra-option Q-learning with
\[
\alpha_k
=
\frac{4}{1-\gamma}\frac{1}{k+h},
\qquad
h
\ge
\frac{512e\log(|\cS||\cO|)}{(1-\gamma)^3}.
\]
Then, for every $\epsilon>0$, it suffices to take
\[
k
\ge
\widetilde{\cO}\left(\frac{1}{\epsilon^2}\right)
\]
to guarantee
\[
\E[\sp(Q_k-Q^*)]\le\epsilon.
\]
More explicitly, it suffices to take
\begin{align*}
k
=
\widetilde{\cO}\Bigg(
\frac{1}{\epsilon^2}
\Bigg[
&
h
\left(
\sp(Q_0)
+
\frac{J}{1-\gamma}
\right)^2
+
\frac{
J^2\log(|\cS||\cO|)
}{
(1-\gamma)^5
}+
\frac{
C_dJ^2
}{
(1-\gamma)^2
}
\left(
1+\|Q_0\|_\infty^2
+
\frac{J^2}{(1-\gamma)^2}
\right)
\Bigg]
\Bigg),
\end{align*}
where
\[
C_d
\triangleq
|\cS|
+
2\log(|\cS||\cO|)+2.
\]
\end{corollary}

\begin{proof}
From Theorem~\ref{thm:intra-final}, for every $k\ge2$,
\begin{align*}
\E[\sp(Q_k-Q^*)^2]
\le{}&
e\Delta_1^2
\frac{h(h-1)}
{(k+h-2)(k+h-1)}+
\frac{
128e\log (|\cS||\cO|)
}{
(1-\gamma)^3
}
\left(
J+2\sp(Q^*)
\right)^2
\frac{k-1}
{(k+h-2)(k+h-1)}
\\
&+
\frac{
64eC_dJ^2
}{
(1-\gamma)^2
}
\left(
1+(b_{k-1}+J)^2
\right)
\frac{1+\log k}{k+h-2},
\end{align*}
where $d=|\cS||\cO|$. Since $h\ge2$,
\[
\frac{h(h-1)}
{(k+h-2)(k+h-1)}
\le
\frac{h}{k},
\qquad
\frac{k-1}
{(k+h-2)(k+h-1)}
\le
\frac1k,
\]
and
\[
\frac1{k+h-2}\le\frac1k.
\]
Therefore,
\begin{align*}
\E[\sp(Q_k-Q^*)^2]
\le
\frac1k
\Bigg[
&
eh\Delta_1^2
+
\frac{
128e\log (|\cS||\cO|)
}{
(1-\gamma)^3
}
\left(
J+2\sp(Q^*)
\right)^2
\\
&+
\frac{
64eC_dJ^2
}{
(1-\gamma)^2
}
\left(
1+(b_{k-1}+J)^2
\right)
(1+\log k)
\Bigg].
\end{align*}

Since $Q^*$ is a span fixed point of $H_{\rm intra}^J$,
\[
\sp(Q^*)
\le
\sp(H_{\rm intra}^JQ^*-H_{\rm intra}^J0)
+
\sp(H_{\rm intra}^J0)
\le
\gamma\sp(Q^*)+J.
\]
Hence
\[
\sp(Q^*)\le\frac{J}{1-\gamma}.
\]
Therefore,
\[
\sp(Q_0-Q^*)
\le
\sp(Q_0)+\frac{J}{1-\gamma},
\]
and, recalling
\[
\Delta_1
=
\sp(Q_0-Q^*)
+
\frac{\alpha}{h}
\left(
J+2\sp(Q_0)
\right),
\qquad
\alpha=\frac4{1-\gamma},
\]
we have
\[
\Delta_1
=
\cO\left(
\sp(Q_0)+\frac{J}{1-\gamma}
\right).
\]

Moreover, Lemma~\ref{le:growth-intra-sync} gives
\[
b_k
=
\|Q_0\|_\infty
+
\cO\left(
\frac{J}{1-\gamma}\log k
\right).
\]
Consequently,
\[
\left(
1+(b_{k-1}+J)^2
\right)(1+\log k)
=
\widetilde{\cO}\left(
1+\|Q_0\|_\infty^2
+
\frac{J^2}{(1-\gamma)^2}
\right).
\]
Substituting these bounds yields
\begin{align*}
\E[\sp(Q_k-Q^*)^2]
=
\widetilde{\cO}\Bigg(
\frac1k
\Bigg[
&
h
\left(
\sp(Q_0)+\frac{J}{1-\gamma}
\right)^2
+
\frac{
J^2\log (|\cS||\cO|)
}{
(1-\gamma)^5
}
\\
&+
\frac{
C_dJ^2
}{
(1-\gamma)^2
}
\left(
1+\|Q_0\|_\infty^2
+
\frac{J^2}{(1-\gamma)^2}
\right)
\Bigg]
\Bigg).
\end{align*}
Finally, Jensen's inequality gives
\[
\E[\sp(Q_k-Q^*)]
\le
\sqrt{
\E[\sp(Q_k-Q^*)^2]
}.
\]
Thus the stated choice of $k$ guarantees
$\E[\sp(Q_k-Q^*)]\le\epsilon$.
\end{proof}

\subsection{Finite-time bounds for synchronous inter-option Q-learning}
\label{subsec:sync-inter}

We recall that the synchronous inter-option Q-learning update is
\begin{equation}
\label{eq:inter-sync-algorithm}
Q_{k+1}
=
Q_k
+
\alpha_k
\Bigl(
\widetilde H_{k,\rm inter}^{\,J}(Q_k)
-
Q_k
+
\omega_k
\Bigr),
\qquad
\alpha_k=\frac{\alpha}{k+h},
\end{equation}
where
\[
\omega_k
\triangleq
\widehat T_{\rm k, inter}^{J}(Q_k)
-
\widetilde H_{k,\rm inter}^{\,J}(Q_k).
\]
Here, $\widetilde H_{k,\rm inter}^{\,J}(Q_k)$ is the conditional mean of
the fresh $J$-step rollout given the history before iteration $k$.

Recall also that
\[
H_kQ
\triangleq
\widehat\rho_k
+
(1-\widehat\lambda_k)Q
+
\widehat\lambda_k \mathcal P Q,
\]
where
\[
(\mathcal P Q)(s,o)
\triangleq
\sum_{s'}
\widehat p(s'|s,o)
\max_{o'}Q(s',o').
\]
Thus $H_k$ is the population normalized inter-option Bellman operator with
$\rho,\lambda$ replaced by their rolling estimates
$\widehat\rho_k,\widehat\lambda_k$.

Let $\widehat H_k$ denote $\widehat H_{\rm k, inter}$, the empirical Bellman operator associated with
$H_k$, obtained by replacing the option-transition kernel
$\widehat p$ by its empirical estimate constructed from samples available
before iteration $k$.

We first establish the almost-sure growth bound used throughout the
analysis.

\begin{lemma}[Almost-sure growth of the synchronous inter-option iterates]
\label{lem:growth-inter-sync}
For every $k\ge1$, almost surely,
\[
\|Q_k\|_\infty\le b_k,
\]
where
\[
b_k
\triangleq
\|Q_0\|_\infty
+
\alpha J
\left(
\frac1h
+
\log\left(\frac{k+h}{h}\right)
\right)
=
\cO(\log k).
\]
\end{lemma}

\begin{proof}
Since $0\le\widehat\rho_k(x)\le1$, every realized $J$-step inter-option
target satisfies, pathwise,
\[
\min_xQ_k(x)
\le
\widehat T_{\rm k, inter}^{J}(Q_k)(x)
\le
J+\max_xQ_k(x).
\]
Indeed, the first $J$ terms of the target are values of
$\widehat\rho_k$ and hence lie in $[0,1]$, while the terminal contribution
is either a single Q-value or the maximum over the option values at the
terminal state.

Since
\[
Q_{k+1}(x)
=
(1-\alpha_k)Q_k(x)
+
\alpha_k\widehat T_{\rm k, inter}^{J}(Q_k)(x)
\]
and $\alpha_k\in[0,1]$,
\[
Q_{k+1}(x)
\le
\max_{x'}Q_k(x')
+
J\alpha_k,
\qquad
Q_{k+1}(x)
\ge
\min_{x'}Q_k(x').
\]
Therefore
\[
\|Q_{k+1}\|_\infty
\le
\|Q_k\|_\infty+J\alpha_k.
\]
Iterating and using $\alpha_t=\alpha/(t+h)$ gives
\[
\|Q_k\|_\infty
\le
\|Q_0\|_\infty
+
\alpha J\sum_{t=0}^{k-1}\frac1{t+h}.
\]
Finally,
\[
\sum_{t=0}^{k-1}\frac1{t+h}
\le
\frac1h+\log\left(\frac{k+h}{h}\right),
\]
which proves the result.
\end{proof}

\subsubsection{Lyapunov framework, error decomposition and drift control}

We use the same generalized Moreau-envelope Lyapunov function as in
Section~\ref{subsec:sync-intra},
\[
M(Q)
\triangleq
M_{q,\theta}(Q)
=
\min_{u}
\left\{
\frac12\sp(u)^2
+
\frac1{2\theta}\|Q-u\|_q^2
\right\}.
\]
By Lemma~\ref{le:Moreau}, there exists a seminorm $p_m$ such that
$M(Q)=p_m(Q)^2/2$ and
\[
\ell_mp_m(Q)\le\sp(Q)\le u_mp_m(Q).
\]
Moreover,
\[
M(Q_2)
\le
M(Q_1)
+
\langle\nabla M(Q_1),Q_2-Q_1\rangle
+
\frac L2\sp(Q_2-Q_1)^2.
\]

Let $M_k\triangleq M(Q_k-Q^*)$ and
$m_k\triangleq\E[M_k]$. Using~\eqref{eq:inter-sync-algorithm},
\[
\E[M_{k+1}]
\le
m_k
+
\alpha_k
\E\left[
\left\langle
\nabla M(Q_k-Q^*),
\widehat T_{\rm k, inter}^{J}(Q_k)-Q_k
\right\rangle
\right]
+
\frac{L\alpha_k^2}{2}T_{\rm quad}.
\]

We decompose the conditional-mean error as
\[
\widetilde H_{k,\rm inter}^{\,J}Q_k-\bar H_{\rm inter}^JQ_k
=
\left(
\widetilde H_{k,\rm inter}^{\,J}Q_k-H_k^JQ_k
\right)
+
\left(
H_k^JQ_k-\bar H_{\rm inter}^JQ_k
\right).
\]
Hence
\[
m_{k+1}
\le
m_k
+
\alpha_k
\left(
T_{\rm drift}
+
T_{\rm greedy}
+
T_{\rm est}
+
T_{\rm noise}
\right)
+
\frac{L\alpha_k^2}{2}T_{\rm quad},
\]
where
\begin{align*}
T_{\rm drift}
&\triangleq
\E\left[
\left\langle
\nabla M(Q_k-Q^*),
\bar H_{\rm inter}^JQ_k-Q_k
\right\rangle
\right],
\\
T_{\rm greedy}
&\triangleq
\E\left[
\left\langle
\nabla M(Q_k-Q^*),
\widetilde H_{k,\rm inter}^{\,J}Q_k-H_k^JQ_k
\right\rangle
\right],
\\
T_{\rm est}
&\triangleq
\E\left[
\left\langle
\nabla M(Q_k-Q^*),
H_k^JQ_k-\bar H_{\rm inter}^JQ_k
\right\rangle
\right],
\\
T_{\rm noise}
&\triangleq
\E\left[
\left\langle
\nabla M(Q_k-Q^*),
\omega_k
\right\rangle
\right],
\\
T_{\rm quad}
&\triangleq
\E\left[
\sp\left(
\widehat T_{\rm k, inter}^{J}(Q_k)-Q_k
\right)^2
\right].
\end{align*}

Thus, relative to the intra-option analysis, the only additional
first-order term is $T_{\rm est}$.

\begin{lemma}
\label{lem:T-drift-inter}
Define
\[
\phi_1
\triangleq
1-\frac{\gamma u_m}{\ell_m}, \quad \gamma = \gamma_{\rm inter}.
\]
Then, for every $k\ge1$,
\[
T_{\rm drift}
\le
-2\phi_1m_k.
\]
\end{lemma}

\begin{proof}
The proof is identical to Lemma~\ref{le:T1_intra_new}, using the
$J$-step contraction of $\bar H_{\rm inter}$ from
Theorem~\ref{thm: J-inter}.
\end{proof}

\subsubsection{Bounding $T_{\rm greedy}$}

We first quantify the error introduced by selecting the intermediate
options using empirical finite-horizon continuation values. Recall that
$H_k$ denotes the normalized inter-option Bellman operator with
$\rho,\lambda$ replaced by $\widehat\rho_k,\widehat\lambda_k$, while
$\widehat H_k$ additionally replaces the option-transition kernel
$\widehat p$ by its empirical estimate $\widehat p_k$.

For a deterministic option-selection rule $d:\cS\to\cO$, define
\begin{align*}
(H_{k,d}Q)(s,o)
\triangleq{}&
\widehat\rho_k(s,o)
+
(1-\widehat\lambda_k(s,o))Q(s,o)+
\widehat\lambda_k(s,o)
\sum_{s'}
\widehat p(s'|s,o)Q(s',d(s')),
\end{align*}
and
\begin{align*}
(\widehat H_{k,d}Q)(s,o)
\triangleq{}&
\widehat\rho_k(s,o)
+
(1-\widehat\lambda_k(s,o))Q(s,o)+
\widehat\lambda_k(s,o)
\sum_{s'}
\widehat p_k(s'|s,o)Q(s',d(s')).
\end{align*}
The corresponding optimality operators are $H_k$ and $\widehat H_k$,
respectively.

For $m=1,\ldots,J-1$, let
\[
\widehat d_{k,m}(s)
\in
\arg\max_{o\in\cO}
\bigl(\widehat H_k^{\,m}Q_k\bigr)(s,o).
\]
Then
\[
\widetilde H_{k,\rm inter}^{\,J}Q_k
=
H_{k,\widehat d_{k,J-1}}
\cdots
H_{k,\widehat d_{k,1}}
H_kQ_k,
\]
whereas
\[
\widehat H_k^JQ_k
=
\widehat H_{k,\widehat d_{k,J-1}}
\cdots
\widehat H_{k,\widehat d_{k,1}}
\widehat H_kQ_k.
\]
Thus, the two compositions use exactly the same sequence of empirical
greedy rules and differ only through the option-transition kernel.

Define
\[
e_{p,k}
\triangleq
\max_{s,o}
\left\|
\widehat p_k(\cdot|s,o)-\widehat p(\cdot|s,o)
\right\|_1.
\]
The transition-kernel concentration argument (Lemma~\ref{lem:empirical-bellman-one-step-explicit}) gives
\[
\E[e_{p,k}^2]
\le
\frac{C_p}{k},
\qquad
C_p
\triangleq
2|\cS|+4\log(|\cS||\cO|)+4.
\]

\begin{lemma}[Population-rollout versus empirical-model error]
\label{lem:inter-tilde-vs-empirical}
For every $k\ge1$,
\[
\E\left[
\left\|
\widetilde H_{k,\mathrm{inter}}^{\,J}Q_k
-
\widehat H_k^JQ_k
\right\|_\infty^2
\right]
\le
\frac{
C_pJ^2(b_k+J)^2
}{k}.
\]
\end{lemma}

\begin{proof}
For $m=1,\ldots,J-1$, define the empirical greedy rules
\[
\widehat d_{k,m}(s)
\in
\arg\max_{o\in\cO}
\bigl(\widehat H_k^mQ_k\bigr)(s,o).
\]
By construction,
\[
\widetilde H_{k,\mathrm{inter}}^{\,J}Q_k
=
H_{k,\widehat d_{k,J-1}}
\cdots
H_{k,\widehat d_{k,1}}
H_kQ_k,
\]
whereas
\[
\widehat H_k^JQ_k
=
\widehat H_{k,\widehat d_{k,J-1}}
\cdots
\widehat H_{k,\widehat d_{k,1}}
\widehat H_kQ_k.
\]
Thus, the two compositions use exactly the same empirical greedy rules and
differ only through the option-transition kernel.

We first verify sup-norm non-expansiveness. For any deterministic
decision rule $d$ and any $Q,Q'$,
\[
|(H_{k,d}Q-H_{k,d}Q')(s,o)|
\le
(1-\widehat\lambda_k(s,o))|Q(s,o)-Q'(s,o)|
+
\widehat\lambda_k(s,o)
\sum_{s'}
\widehat p(s'|s,o)
|Q(s',d(s'))-Q'(s',d(s'))|.
\]
Since $0\le\widehat\lambda_k(s,o)\le1$ and
$\widehat p(\cdot|s,o)$ is a probability distribution,
\[
\|H_{k,d}Q-H_{k,d}Q'\|_\infty
\le
\|Q-Q'\|_\infty.
\]
Replacing $\widehat p$ by $\widehat p_k$ gives
\[
\|\widehat H_{k,d}Q-\widehat H_{k,d}Q'\|_\infty
\le
\|Q-Q'\|_\infty.
\]
Likewise, using
\[
\left|
\max_{o'}Q(s',o')-\max_{o'}Q'(s',o')
\right|
\le
\|Q-Q'\|_\infty,
\]
we obtain
\[
\|H_kQ-H_kQ'\|_\infty
\le
\|Q-Q'\|_\infty,
\qquad
\|\widehat H_kQ-\widehat H_kQ'\|_\infty
\le
\|Q-Q'\|_\infty.
\]

We now replace the empirical-model operators by their population-model
counterparts one stage at a time. Define
\[
Z_0
\triangleq
H_{k,\widehat d_{k,J-1}}
\cdots
H_{k,\widehat d_{k,1}}
H_kQ_k,
\]
and, for $1\le r\le J-1$,
\[
Z_r
\triangleq
H_{k,\widehat d_{k,J-1}}
\cdots
H_{k,\widehat d_{k,r}}
\widehat H_{k,\widehat d_{k,r-1}}
\cdots
\widehat H_{k,\widehat d_{k,1}}
\widehat H_kQ_k,
\]
with
\[
Z_J
\triangleq
\widehat H_{k,\widehat d_{k,J-1}}
\cdots
\widehat H_{k,\widehat d_{k,1}}
\widehat H_kQ_k.
\]
Hence,
\[
Z_0
=
\widetilde H_{k,\mathrm{inter}}^{\,J}Q_k,
\qquad
Z_J
=
\widehat H_k^JQ_k,
\]
and therefore
\[
\left\|
\widetilde H_{k,\mathrm{inter}}^{\,J}Q_k
-
\widehat H_k^JQ_k
\right\|_\infty
\le
\sum_{r=0}^{J-1}
\|Z_r-Z_{r+1}\|_\infty.
\]

For the first hybrid difference,
\[
Z_0
=
H_{k,\widehat d_{k,J-1}}
\cdots
H_{k,\widehat d_{k,1}}
H_kQ_k,
\]
whereas
\[
Z_1
=
H_{k,\widehat d_{k,J-1}}
\cdots
H_{k,\widehat d_{k,1}}
\widehat H_kQ_k.
\]
By non-expansiveness of the outer population fixed-policy operators,
\[
\|Z_0-Z_1\|_\infty
\le
\|(\widehat H_k-H_k)Q_k\|_\infty.
\]

For $1\le r\le J-1$, define
\[
V_r
\triangleq
\widehat H_{k,\widehat d_{k,r-1}}
\cdots
\widehat H_{k,\widehat d_{k,1}}
\widehat H_kQ_k.
\]
Then
\[
Z_r
=
H_{k,\widehat d_{k,J-1}}
\cdots
H_{k,\widehat d_{k,r+1}}
H_{k,\widehat d_{k,r}}V_r,
\]
while
\[
Z_{r+1}
=
H_{k,\widehat d_{k,J-1}}
\cdots
H_{k,\widehat d_{k,r+1}}
\widehat H_{k,\widehat d_{k,r}}V_r.
\]
Hence, again by non-expansiveness,
\[
\|Z_r-Z_{r+1}\|_\infty
\le
\left\|
\left(
\widehat H_{k,\widehat d_{k,r}}
-
H_{k,\widehat d_{k,r}}
\right)V_r
\right\|_\infty.
\]

Let
\[
e_{p,k}
\triangleq
\max_{s,o}
\left\|
\widehat p_k(\cdot|s,o)-\widehat p(\cdot|s,o)
\right\|_1.
\]
Since $H_{k,d}$ and $\widehat H_{k,d}$ differ only through the
option-transition kernel, for every deterministic rule $d$ and every $V$,
\[
|((\widehat H_{k,d}-H_{k,d})V)(s,o)|
=
\widehat\lambda_k(s,o)
\left|
\sum_{s'}
\bigl(
\widehat p_k(s'|s,o)-\widehat p(s'|s,o)
\bigr)V(s',d(s'))
\right|
\le
e_{p,k}\|V\|_\infty.
\]
Therefore,
\[
\|(\widehat H_{k,d}-H_{k,d})V\|_\infty
\le
e_{p,k}\|V\|_\infty.
\]
The same argument for the optimality operators gives
\[
\|(\widehat H_k-H_k)V\|_\infty
\le
e_{p,k}\|V\|_\infty.
\]

It remains to control the intermediate iterates. Since
$0\le\widehat\rho_k\le1$,
\[
\|\widehat H_kV\|_\infty
\le
1+\|V\|_\infty,
\qquad
\|\widehat H_{k,d}V\|_\infty
\le
1+\|V\|_\infty.
\]
Using $\|Q_k\|_\infty\le b_k$ almost surely therefore gives
\[
\|V_r\|_\infty
\le
b_k+r
\le
b_k+J.
\]
Consequently,
\[
\|Z_0-Z_1\|_\infty
\le
e_{p,k}b_k
\le
e_{p,k}(b_k+J),
\]
and, for every $1\le r\le J-1$,
\[
\|Z_r-Z_{r+1}\|_\infty
\le
e_{p,k}(b_k+J).
\]
Thus,
\[
\left\|
\widetilde H_{k,\mathrm{inter}}^{\,J}Q_k
-
\widehat H_k^JQ_k
\right\|_\infty
\le
Je_{p,k}(b_k+J).
\]
Squaring and taking expectations gives
\[
\E\left[
\left\|
\widetilde H_{k,\mathrm{inter}}^{\,J}Q_k
-
\widehat H_k^JQ_k
\right\|_\infty^2
\right]
\le
J^2(b_k+J)^2\E[e_{p,k}^2].
\]
Finally, the empirical transition-kernel concentration bound from Lemma~\ref{lem:empirical-bellman-one-step-explicit} applied to the inter-option setting gives
\[
\E[e_{p,k}^2]
\le
\frac{C_p}{k},
\]
and hence
\[
\E\left[
\left\|
\widetilde H_{k,\mathrm{inter}}^{\,J}Q_k
-
\widehat H_k^JQ_k
\right\|_\infty^2
\right]
\le
\frac{
C_pJ^2(b_k+J)^2
}{k}.
\]
\end{proof}

\begin{lemma}[Empirical versus population $J$-step Bellman error]
\label{lem:inter-empirical-vs-population}
For every $k\ge1$,
\[
\E\left[
\left\|
\widehat H_k^JQ_k-H_k^JQ_k
\right\|_\infty^2
\right]
\le
\frac{
C_pJ^2(b_k+J)^2
}{k}.
\]
\end{lemma}

\begin{proof}
Define
\[
Z_r
\triangleq
\widehat H_k^{\,J-r}H_k^rQ_k,
\qquad
r=0,\ldots,J.
\]
Then $Z_0=\widehat H_k^JQ_k$ and $Z_J=H_k^JQ_k$. Hence
\[
\left\|
\widehat H_k^JQ_k-H_k^JQ_k
\right\|_\infty
\le
\sum_{r=0}^{J-1}
\|Z_r-Z_{r+1}\|_\infty.
\]
By sup-norm non-expansiveness of $\widehat H_k$,
\[
\|Z_r-Z_{r+1}\|_\infty
\le
\left\|
(\widehat H_k-H_k)H_k^rQ_k
\right\|_\infty.
\]
Since
\[
\|H_k^rQ_k\|_\infty
\le
b_k+r
\le
b_k+J,
\]
the one-step empirical model comparison gives, pathwise,
\[
\left\|
(\widehat H_k-H_k)H_k^rQ_k
\right\|_\infty
\le
e_{p,k}(b_k+J).
\]
Therefore
\[
\left\|
\widehat H_k^JQ_k-H_k^JQ_k
\right\|_\infty
\le
Je_{p,k}(b_k+J).
\]
Squaring and using $\E[e_{p,k}^2]\le C_p/k$ proves the result.
\end{proof}

\begin{lemma}[Finite-$J$ empirical-greedy error]
\label{lem:inter-finite-j-greedy}
For every $k\ge1$,
\[
\E\left[
\sp\left(
\widetilde H_{k,\rm inter}^{\,J}Q_k
-
H_k^JQ_k
\right)^2
\right]
\le
\frac{
4C_pJ^2(b_k+J)^2
}{k}.
\]
In particular, since $b_k=\cO(J\log k)$,
\[
\E\left[
\sp\left(
\widetilde H_{k,\rm inter}^{\,J}Q_k
-
H_k^JQ_k
\right)^2
\right]
=
\widetilde{\cO}\left(
\frac{J^4}{k}
\right).
\]
\end{lemma}

\begin{proof}
Decompose
\[
\widetilde H_{k,\rm inter}^{\,J}Q_k-H_k^JQ_k
=
\left(
\widetilde H_{k,\rm inter}^{\,J}Q_k-\widehat H_k^JQ_k
\right)
+
\left(
\widehat H_k^JQ_k-H_k^JQ_k
\right).
\]
Using $\sp(x)\le\|x\|_\infty$ and
$(a+b)^2\le2a^2+2b^2$,
\begin{align*}
&
\E\left[
\sp\left(
\widetilde H_{k,\rm inter}^{\,J}Q_k-H_k^JQ_k
\right)^2
\right]\le
2
\E\left[
\left\|
\widetilde H_{k,\rm inter}^{\,J}Q_k-\widehat H_k^JQ_k
\right\|_\infty^2
\right]+
2
\E\left[
\left\|
\widehat H_k^JQ_k-H_k^JQ_k
\right\|_\infty^2
\right].
\end{align*}
Applying
Lemmas~\ref{lem:inter-tilde-vs-empirical}
and~\ref{lem:inter-empirical-vs-population}
gives
\[
\E\left[
\sp\left(
\widetilde H_{k,\rm inter}^{\,J}Q_k-H_k^JQ_k
\right)^2
\right]
\le
\frac{
4C_pJ^2(b_k+J)^2
}{k}.
\]
The final claim follows from $b_k=\cO(J\log k)$.
\end{proof}

\begin{proposition}[Control of the approximately greedy perturbation]
\label{prop:T-greedy-inter}
Recall
\[
T_{\rm greedy}
\triangleq
\E\left[
\left\langle
\nabla M(Q_k-Q^*),
\widetilde H_{k,\rm inter}^{\,J}Q_k-H_k^JQ_k
\right\rangle
\right].
\]
Then, for every $k\ge1$,
\[
T_{\rm greedy}
\le
\frac{\phi_1}{4}m_k
+
\frac{
8C_pJ^2(b_k+J)^2
}{
\phi_1\ell_m^2k
}.
\]
In particular,
\[
T_{\rm greedy}
\le
\frac{\phi_1}{4}m_k
+
\widetilde{\cO}\left(
\frac{J^4}{k}
\right).
\]
\end{proposition}

\begin{proof}
By the Moreau-envelope inequality argument from Proposition~\ref{prop:T-greedy},
\[
\left\langle
\nabla M(Q_k-Q^*),z
\right\rangle
\le
\frac{\sqrt{2M_k}}{\ell_m}\sp(z).
\]
Applying the AM-GM inequality with parameter $\phi_1/4$ gives
\[
T_{\rm greedy}
\le
\frac{\phi_1}{4}m_k
+
\frac{2}{\phi_1\ell_m^2}
\E\left[
\sp\left(
\widetilde H_{k,\rm inter}^{\,J}Q_k-H_k^JQ_k
\right)^2
\right].
\]
Applying Lemma~\ref{lem:inter-finite-j-greedy} yields
\[
T_{\rm greedy}
\le
\frac{\phi_1}{4}m_k
+
\frac{
8C_pJ^2(b_k+J)^2
}{
\phi_1\ell_m^2k
}.
\]
Since $b_k=\cO(J\log k)$, the second term is
$\widetilde{\cO}(J^4/k)$.
\end{proof}
\subsubsection{Bounding $T_{\rm est}$}

We now control the additional error caused by estimating the reward rate
and inverse mean holding time.

Define
\[
e_{\rho,k}
\triangleq
\|\widehat\rho_k-\rho\|_\infty,
\qquad
e_{\lambda,k}
\triangleq
\|\widehat\lambda_k-\lambda\|_\infty.
\]

\begin{lemma}[Accuracy of the rolling inter-option estimates]
\label{lem:inter-rolling}
Under Assumption~\ref{assm:smdp}, for every $x\in\cS\times\cO$,
\[
\E|\widehat\rho_k(x)-\rho(x)|^2
\le
\frac{c^2}{k+1},
\qquad
\E|\widehat\lambda_k(x)-\lambda(x)|^2
\le
\frac{c^2}{k+1}.
\]
Consequently,
\[
\E[e_{\rho,k}^2]
\le
\frac{|\cS||\cO|c^2}{k+1},
\qquad
\E[e_{\lambda,k}^2]
\le
\frac{|\cS||\cO|c^2}{k+1}.
\]
\end{lemma}

\begin{proof}
Fix $x$. Writing
\[
B_k(x)
\triangleq
1+\sum_{t=0}^{k-1}L_t(x),
\]
we have $B_k(x)\ge k+1$. Since
$\rho(x)=\E[\widetilde R(x)]/\E[L(x)]$,
\[
\E[\widetilde R_t(x)-\rho(x)L_t(x)]=0
\]
and
\[
\widehat\rho_k(x)-\rho(x)
=
\frac{
\sum_{t=0}^{k-1}
(\widetilde R_t(x)-\rho(x)L_t(x))
-\rho(x)
}{
B_k(x)
}.
\]
The summands are independent and mean-zero. Since
$0\le\widetilde R_t(x)\le L_t(x)$ and $0\le\rho(x)\le1$,
\[
|\widetilde R_t(x)-\rho(x)L_t(x)|
\le
L_t(x).
\]
Therefore
\[
\E|\widehat\rho_k(x)-\rho(x)|^2
\le
\frac{
kc^2+1
}{
(k+1)^2
}
\le
\frac{c^2}{k+1}.
\]

Similarly, since $\lambda(x)=1/\E[L(x)]$,
\[
\E[1-\lambda(x)L_t(x)]=0
\]
and
\[
\widehat\lambda_k(x)-\lambda(x)
=
\frac{
\sum_{t=0}^{k-1}(1-\lambda(x)L_t(x))
+
1-\lambda(x)
}{
B_k(x)
}.
\]
Moreover,
\[
\E[(1-\lambda(x)L)^2|x]
=
\lambda(x)^2\text{Var}(L|x)
\le
\E[L^2|x]
\le
c^2.
\]
Hence
\[
\E|\widehat\lambda_k(x)-\lambda(x)|^2
\le
\frac{c^2}{k+1}.
\]
Finally,
\[
e_{\rho,k}^2
\le
\sum_x|\widehat\rho_k(x)-\rho(x)|^2,
\qquad
e_{\lambda,k}^2
\le
\sum_x|\widehat\lambda_k(x)-\lambda(x)|^2,
\]
which proves the maximal bounds.
\end{proof}


\begin{lemma}[Finite-$J$ parameter-estimation error]
\label{lem:inter-parameter-error}
For every $Q$,
\[
\sp(H_k^JQ-\bar H_{\rm inter}^JQ)
\le
Je_{\rho,k}
+
J\bigl(2\sp(Q)+J-1\bigr)e_{\lambda,k}.
\]
Consequently,
\[
\E\left[
\sp(H_k^JQ_k-\bar H_{\rm inter}^JQ_k)^2
\right]
\le
\mathcal E_k,
\]
where
\[
\mathcal E_k
\triangleq
\frac{
2|\cS||\cO|c^2J^2
}{
k+1
}
\left[
1+
(2b_k+J-1)^2
\right].
\]
\end{lemma}

\begin{proof}
The one-step operators satisfy
\[
\bar H_{\rm inter}Q
=
\rho+Q+\lambda(\mathcal PQ-Q),
\qquad
H_kQ
=
\widehat\rho_k+Q+\widehat\lambda_k(\mathcal PQ-Q).
\]
Hence
\[
H_kQ-\bar H_{\rm inter}Q
=
(\widehat\rho_k-\rho)
+
(\widehat\lambda_k-\lambda)(\mathcal PQ-Q).
\]
Since
\[
\|\mathcal PQ-Q\|_\infty
\le
2\sp(Q),
\]
we obtain
\[
\sp(H_kQ-\bar H_{\rm inter}Q)
\le
e_{\rho,k}
+
2\sp(Q)e_{\lambda,k}.
\]

We now compare the $J$-fold compositions by replacing one occurrence of
$\bar H_{\rm inter}$ by $H_k$ at a time. For $r=0,\ldots,J$, define
\[
Z_r
\triangleq
H_k^{J-r}\bar H_{\rm inter}^{\,r}Q.
\]
Then
\[
Z_0=H_k^JQ,
\qquad
Z_J=\bar H_{\rm inter}^JQ.
\]
Therefore, by the triangle inequality for the span seminorm,
\[
\sp(H_k^JQ-\bar H_{\rm inter}^JQ)
\le
\sum_{r=0}^{J-1}\sp(Z_r-Z_{r+1}).
\]

For each $r=0,\ldots,J-1$,
\[
Z_r
=
H_k^{J-r-1}H_k\bar H_{\rm inter}^{\,r}Q,
\]
whereas
\[
Z_{r+1}
=
H_k^{J-r-1}\bar H_{\rm inter}^{\,r+1}Q.
\]
Hence
\[
Z_r-Z_{r+1}
=
H_k^{J-r-1}
\bigl(
H_k\bar H_{\rm inter}^{\,r}Q
\bigr)
-
H_k^{J-r-1}
\bigl(
\bar H_{\rm inter}^{\,r+1}Q
\bigr).
\]
By the span non-expansiveness of $H_k$, established in
Lemma~\ref{lem:inter-tilde-vs-empirical},
\[
\sp(Z_r-Z_{r+1})
\le
\sp\left(
H_k\bar H_{\rm inter}^{\,r}Q
-
\bar H_{\rm inter}^{\,r+1}Q
\right).
\]
Since
\[
\bar H_{\rm inter}^{\,r+1}Q
=
\bar H_{\rm inter}
\bigl(
\bar H_{\rm inter}^{\,r}Q
\bigr),
\]
the one-step bound gives
\[
\sp(Z_r-Z_{r+1})
\le
e_{\rho,k}
+
2\sp(\bar H_{\rm inter}^{\,r}Q)e_{\lambda,k}.
\]
Summing over $r$ yields
\[
\sp(H_k^JQ-\bar H_{\rm inter}^JQ)
\le
Je_{\rho,k}
+
2e_{\lambda,k}
\sum_{r=0}^{J-1}
\sp(\bar H_{\rm inter}^{\,r}Q).
\]

Because rewards are bounded by one hence using $\widetilde R \le L \enspace a.s. \implies \rho = \E[{\widetilde R}]/\E[L] \le 1$,
\[
\sp(\bar H_{\rm inter}V)
\le
\sp(V)+1.
\]
Iterating this inequality gives
\[
\sp(\bar H_{\rm inter}^{\,r}Q)
\le
\sp(Q)+r.
\]
Therefore,
\[
\sum_{r=0}^{J-1}
\sp(\bar H_{\rm inter}^{\,r}Q)
\le
J\sp(Q)+\frac{J(J-1)}{2},
\]
and hence
\[
\sp(H_k^JQ-\bar H_{\rm inter}^JQ)
\le
Je_{\rho,k}
+
J\bigl(2\sp(Q)+J-1\bigr)e_{\lambda,k}.
\]

Setting $Q=Q_k$, using $\sp(Q_k)\le\|Q_k\|_\infty\le b_k$, and applying
$(a+b)^2\le2a^2+2b^2$ gives
\[
\sp(H_k^JQ_k-\bar H_{\rm inter}^JQ_k)^2
\le
2J^2e_{\rho,k}^2
+
2J^2(2b_k+J-1)^2e_{\lambda,k}^2.
\]
Taking expectations and using Lemma~\ref{lem:inter-rolling},
\[
\E[e_{\rho,k}^2]
\le
\frac{|\cS||\cO|c^2}{k+1},
\qquad
\E[e_{\lambda,k}^2]
\le
\frac{|\cS||\cO|c^2}{k+1},
\]
we obtain
\[
\E\left[
\sp(H_k^JQ_k-\bar H_{\rm inter}^JQ_k)^2
\right]
\le
\frac{
2|\cS||\cO|c^2J^2
}{
k+1
}
\left[
1+
(2b_k+J-1)^2
\right].
\]
\end{proof}




\begin{proposition}[Control of the parameter-estimation perturbation]
\label{prop:T-est-inter}
For every $k\ge1$,
\[
T_{\rm est}
\le
\frac{\phi_1}{4}m_k
+
\frac{2\mathcal E_k}{\phi_1\ell_m^2}.
\]
In particular,
\[
T_{\rm est}
\le
\frac{\phi_1}{4}m_k
+
\widetilde{\cO}\left(\frac1k\right).
\]
\end{proposition}

\begin{proof}
The Moreau pairing inequality gives
\[
T_{\rm est}
\le
\E\left[
\frac{\sqrt{2M_k}}{\ell_m}
\sp(H_k^JQ_k-\bar H_{\rm inter}^JQ_k)
\right].
\]
Applying the AM-GM inequality with parameter $\phi_1/4$ yields
\[
T_{\rm est}
\le
\frac{\phi_1}{4}m_k
+
\frac2{\phi_1\ell_m^2}
\E\left[
\sp(H_k^JQ_k-\bar H_{\rm inter}^JQ_k)^2
\right].
\]
Lemma~\ref{lem:inter-parameter-error} proves the result.
\end{proof}

\subsubsection{Bounding $T_{\rm noise}$ and $T_{\rm quad}$}

\begin{proposition}
\label{prop:T-noise-inter}
For every $k\ge1$,
\[
T_{\rm noise}=0.
\]
\end{proposition}

\begin{proof}
By construction,
$\widetilde H_{k,\rm inter}^{\,J}(Q_k)
=
\E[\widehat T_{\rm k, inter}^{J}(Q_k)|\mathcal F_k]$.
Hence $\E[\omega_k|\mathcal F_k]=0$.
Since $\nabla M(Q_k-Q^*)$ is $\mathcal F_k$-measurable,
\[
T_{\rm noise}
=
\E\left[
\left\langle
\nabla M(Q_k-Q^*),
\E[\omega_k|\mathcal F_k]
\right\rangle
\right]
=
0.
\]
\end{proof}

\begin{proposition}[Control of the quadratic term]
\label{prop:T-quad-inter}
For every $k\ge1$,
\[
T_{\rm quad}
\le
16u_m^2m_k
+
2\left(
J+2\sp(Q^*)
\right)^2.
\]
Consequently,
\[
\frac{L\alpha_k^2}{2}T_{\rm quad}
\le
8Lu_m^2\alpha_k^2m_k
+
L\alpha_k^2
\left(
J+2\sp(Q^*)
\right)^2.
\]
In particular, if
\[
\alpha_k
\le
\frac{\phi_1}{16Lu_m^2},
\]
then
\[
\frac{L\alpha_k^2}{2}T_{\rm quad}
\le
\frac{\phi_1}{2}\alpha_km_k
+
L\alpha_k^2
\left(
J+2\sp(Q^*)
\right)^2.
\]
\end{proposition}

\begin{proof}
Since $0\le\widehat\rho_k(x)\le1$, every realized inter-option target
satisfies
\[
\min_xQ_k(x)
\le
\widehat T_{\rm k, inter}^{J}(Q_k)(x)
\le
J+\max_xQ_k(x).
\]
Hence
\[
\sp\left(
\widehat T_{\rm k, inter}^{J}(Q_k)-Q_k
\right)
\le
J+2\sp(Q_k).
\]
Using
\[
\sp(Q_k)
\le
\sp(Q_k-Q^*)+\sp(Q^*)
\]
and $(a+b)^2\le2a^2+2b^2$,
\[
T_{\rm quad}
\le
8\E[\sp(Q_k-Q^*)^2]
+
2(J+2\sp(Q^*))^2.
\]
Since
\[
\E[\sp(Q_k-Q^*)^2]
\le
2u_m^2m_k,
\]
the first claim follows. The remaining statements are immediate.
\end{proof}

\subsubsection{Unrolling the Lyapunov recursion}

We now use exactly the same Moreau-envelope choice as in the
intra-option analysis. Let
\[
q\triangleq2\log (|\cS||\cO|),
\qquad
\theta
\triangleq
\left(
\frac{1+\gamma}{2\gamma}
\right)^2-1.
\]
Then
\[
\ell_m=\sqrt{1+\theta/e},
\qquad
u_m=\sqrt{1+\theta},
\qquad
L=\frac{e(q-1)}{\theta},
\]
and
\[
\frac{u_m^2}{\ell_m^2}\le e,
\qquad
\phi_1
=
1-\gamma\frac{u_m}{\ell_m}
\ge
\frac{1-\gamma}{2},
\qquad
Lu_m^2
\le
\frac{4e\log (|\cS||\cO|)}{1-\gamma}.
\]

\begin{proposition}[One-step Lyapunov recursion]
\label{prop:inter-one-step-final}
Suppose
\[
\alpha_k
\le
\frac{\phi_1}{16Lu_m^2}.
\]
Then, for every $k\ge1$,
\begin{align*}
m_{k+1}
\le{}&
(1-\phi_1\alpha_k)m_k
+
L\alpha_k^2
(J+2\sp(Q^*))^2
\\
&+
\frac{
2\alpha_kC_pJ^2(1+(b_k+J))^2
}{
\phi_1\ell_m^2k
}
+
\frac{
2\alpha_k\mathcal E_k
}{
\phi_1\ell_m^2
}.
\end{align*}
\end{proposition}

\begin{proof}
Combining
Lemma~\ref{lem:T-drift-inter},
Proposition~\ref{prop:T-greedy-inter},
Proposition~\ref{prop:T-est-inter},
Proposition~\ref{prop:T-noise-inter}, and
Proposition~\ref{prop:T-quad-inter} gives
\[
-2\phi_1
+
\frac{\phi_1}{4}
+
\frac{\phi_1}{4}
+
\frac{\phi_1}{2}
=
-\phi_1
\]
for the coefficient multiplying $\alpha_km_k$, which proves the claim.
\end{proof}

We now state the full version of Theorem~\ref{thm:sync-inter-ftb}.
\begin{theorem}[Synchronous inter-option Q-learning]
\label{thm:inter-final}
Let
\[
\alpha_k=\frac{\alpha}{k+h},
\qquad
\alpha=\frac4{1-\gamma},
\]
and suppose
\[
h
\ge
\frac{
512e\log (|\cS||\cO|)
}{
(1-\gamma)^3
}.
\]
Then the stepsize condition of
Proposition~\ref{prop:inter-one-step-final} is satisfied.

Define
\[
\Delta_1
\triangleq
\sp(Q_0-Q^*)
+
\frac{\alpha}{h}
\left(
J+2\sp(Q_0)
\right)
\]
and
\[
G_k
\triangleq
(2|\cS|+4\log (|\cS||\cO|)+4)J^2 (b_k+J)^2
+
2|\cS||\cO|c^2J^2
\left[
1+(2b_k+J-1)^2
\right].
\]
Then, for every $k\ge2$,
\begin{align*}
\E[\sp(Q_k-Q^*)^2]
\le{}&
e\Delta_1^2
\frac{
h(h-1)
}{
(k+h-2)(k+h-1)
}
\\
&+
\frac{
128e\log (|\cS||\cO|)
}{
(1-\gamma)^3
}
(J+2\sp(Q^*))^2
\frac{
k-1
}{
(k+h-2)(k+h-1)
}
\\
&+
\frac{
32e
}{
(1-\gamma)^2
}
G_{k-1}
\frac{
(k-1)+(h-1)(1+\log(k-1))
}{
(k+h-2)(k+h-1)
}.
\end{align*}
Consequently,
\[
\E[\sp(Q_k-Q^*)^2]
=
\widetilde{\cO}\left(\frac1k\right).
\]
More explicitly, up to the faster-decaying initialization term,
\[
\E[\sp(Q_k-Q^*)^2]
=
\cO\left(
\frac{
(J+\sp(Q^*))^2\log (|\cS||\cO|)
}{
(1-\gamma)^3k
}
+
\frac{
C_pJ^2 (b_k+J)^2
}{
(1-\gamma)^2k
}
+
\frac{
|\cS||\cO|c^2J^2
[1+(2b_k+J)^2]
}{
(1-\gamma)^2k
}
\right).
\]
\end{theorem}

\begin{proof}
Since
\[
\frac{\phi_1}{16Lu_m^2}
\ge
\frac{(1-\gamma)^2}{128e\log (|\cS||\cO|)}
\]
and $\alpha_k\le4/((1-\gamma)h)$, the stated lower bound on $h$
ensures the stepsize condition.

Moreover,
\[
\alpha\phi_1
\ge
\frac4{1-\gamma}\frac{1-\gamma}{2}
=
2.
\]
Therefore Proposition~\ref{prop:inter-one-step-final} implies
\[
m_{k+1}
\le
\left(
1-\frac2{k+h}
\right)m_k
+
\frac{B}{(k+h)^2}
+
\frac{
D_k
}{
k(k+h)
},
\]
where
\[
B
\triangleq
L\alpha^2(J+2\sp(Q^*))^2
\]
and, using $(k+1)^{-1}\le k^{-1}$,
\[
D_k
\triangleq
\frac{2\alpha}{\phi_1\ell_m^2}G_k.
\]

For $1\le i\le k-1$,
\[
\prod_{j=i+1}^{k-1}
\left(
1-\frac2{j+h}
\right)
=
\frac{
(i+h-1)(i+h)
}{
(k+h-2)(k+h-1)
}.
\]
Unrolling from iteration $1$ therefore gives
\begin{align*}
m_k
\le{}&
m_1
\frac{
h(h-1)
}{
(k+h-2)(k+h-1)
}
\\
&+
B
\frac{
k-1
}{
(k+h-2)(k+h-1)
}
\\
&+
D_{k-1}
\frac{
(k-1)+(h-1)(1+\log(k-1))
}{
(k+h-2)(k+h-1)
},
\end{align*}
where we used that $b_k$, and hence $D_k$, is nondecreasing.

By norm equivalence,
\[
\E[\sp(Q_k-Q^*)^2]
\le
2u_m^2m_k.
\]
Furthermore,
\[
2u_m^2m_1
\le
\frac{u_m^2}{\ell_m^2}\Delta_1^2
\le
e\Delta_1^2.
\]
For the quadratic contribution,
\[
2u_m^2B
=
2Lu_m^2\alpha^2
(J+2\sp(Q^*))^2
\le
\frac{
128e\log (|\cS||\cO|)
}{
(1-\gamma)^3
}
(J+2\sp(Q^*))^2.
\]
Finally,
\[
2u_m^2D_k
=
\frac{
4\alpha u_m^2
}{
\phi_1\ell_m^2
}
G_k
\le
\frac{
32e
}{
(1-\gamma)^2
}
G_k,
\]
where we used
$u_m^2/\ell_m^2\le e$,
$\phi_1\ge(1-\gamma)/2$, and
$\alpha=4/(1-\gamma)$.

Substitution proves the explicit bound.

Finally,
\[
b_k
=
\|Q_0\|_\infty
+
\cO\left(
\frac{J}{1-\gamma}\log k
\right),
\]
so both first-order perturbations contribute
$\widetilde{\cO}(1/k)$, while the quadratic contribution is
$\cO(1/k)$ and the initialization term is $\cO(k^{-2})$.
Hence
\[
\E[\sp(Q_k-Q^*)^2]
=
\widetilde{\cO}(1/k).
\]
\end{proof}

\begin{corollary}[Rate and sample complexity]
\label{cor:sync_inter_new_rate}
Under the assumptions and stepsize conditions of
Theorem~\ref{thm:inter-final}, the iterates generated by
Algorithm~\ref{alg:sync-inter} satisfy
\[
\E\!\left[\sp(Q_k-Q^*)^2\right]
=
\widetilde{\cO}\!\left(\frac{1}{k}\right),
\qquad
\E\!\left[\sp(Q_k-Q^*)\right]
=
\widetilde{\cO}\!\left(\frac{1}{\sqrt{k}}\right).
\]
More explicitly, up to the faster-decaying initialization term,
\[
\E\!\left[\sp(Q_k-Q^*)^2\right]
=
\widetilde{\cO}\!\left(
\frac{1}{k}
\left[
\frac{(J+\sp(Q^*))^2\log(|\cS||\cO|)}{(1-\gamma)^3}
+
\frac{C_pJ^2(b_k+J)^2}{(1-\gamma)^2}
+
\frac{|\cS||\cO|c^2J^2\bigl(1+(2b_k+J)^2\bigr)}{(1-\gamma)^2}
\right]
\right),
\]
where $C_p=2|\cS|+4\log(|\cS||\cO|)+4$ and
$b_k=\cO(\log k)$.

Consequently, for every $\varepsilon>0$, it suffices to take
\[
k
=
\widetilde{\cO}\!\left(\frac{1}{\varepsilon^2}\right)
\]
to guarantee
\[
\E[\sp(Q_k-Q^*)]\le\varepsilon.
\]
\end{corollary}

\begin{proof}
Theorem~\ref{thm:inter-final} gives
\[
\E[\sp(Q_k-Q^*)^2]
=
\widetilde{\cO}\!\left(\frac{1}{k}\right),
\]
since $b_k=\cO(\log k)$, while the initialization contribution decays as
$\cO(k^{-2})$. Therefore, by Jensen's inequality,
\[
\E[\sp(Q_k-Q^*)]
\le
\sqrt{\E[\sp(Q_k-Q^*)^2]}
=
\widetilde{\cO}\!\left(\frac{1}{\sqrt{k}}\right).
\]
Hence, choosing
$k=\widetilde{\cO}(\varepsilon^{-2})$
ensures $\E[\sp(Q_k-Q^*)]\le\varepsilon$.
\end{proof}



\newpage
\section{\MakeUppercase{Application: Average-reward Q-learning with Reward Machines}}\label{app:HRM}

In this section we show that the J-step span contraction (Theorem~\ref{thm: J-inter} and~\ref{thm: J-intra}) and the finite-time bounds for synchronous intra- and inter-option Q learning (Theorems~\ref{thm:sync-inter-ftb} and~\ref{thm: sync-intra-ftb}) hold for average-reward reinforcement learning in Reward Machines~\citep{toro_reward_machines}. Reward machines (RMs) are a type of finite state machine, which take as input an interaction history and outputs a reward function. RMs are used to formalise tasks specified by linear temporal logic, sub-goal sequences or regular languages, in which reward is non-Markovian. Our results show the first non-asymptotic convergence guarantees for average reward Q-learning for hierarchical RMs.

\subsection{Average-reward MDPs with reward machines}
Using the formalism introduced in~\cite{toro_reward_machines}, we define average-reward MDPRMs as follows:

\begin{definition}[Average-reward MDPRM]
\label{def:armrm}
An average-reward Markov decision process with a reward
machine (\textsc{ar-MDPRM}) can be given by a tuple
\[
   \mathcal{T}
   \;=\;
   \bigl\langle\, S,\,A,\,p,\,
                 \mathcal{P},\,L,\,
                 U,\,u_0,\,F,\,\delta_u,\,\delta_r \,\bigr\rangle,
\]
where $(S,A,p)$ is the underlying primitive MDP,
$\mathcal{P}$ is a finite set of propositional symbols,
$L:S\times A\times S\to 2^{\mathcal{P}}$ is the labelling function,
$U$ is a finite set of RM states with initial state $u_0\in U$ and
terminal set $F$ (with $U\cap F=\emptyset$),
$\delta_u:U\times 2^{\mathcal{P}}\to U\cup F$ is the
state-transition function, and
$\delta_r:U\to[\,S\times A\times S\to\mathbb{R}\,]$ is the
state-reward function, where we aim to maximise the cumulative average reward given by
\[
   \bar r^{\,\pi}
   \;=\;
   \lim_{T\to\infty}\frac1T\,\mathbb{E}^{\pi}
       \!\Bigl[\,\sum_{t=0}^{T-1} R_{t+1}\,\Bigr],
   \qquad
   R_{t+1}=\delta_r(u_t)(s_t,a_t,s_{t+1}).
\]
\end{definition}
Without loss of generality, we assume rewards in $[0,1]$ consistent with our Assumption~\ref{assm:smdp}.
The reward signal $R_{t+1}$ is non-Markovian in the original state space $S$, but is Markovian in the augmented state space $\mathcal{X}=S\times U$, following Observation $1$ in~\cite{toro_reward_machines}, repeated here for clarity: 

\begin{observation}
\textit{Given an MDPRM}
$T = \langle S, A, p, \gamma, P, L, U, u_0, F, \delta_u, \delta_r \rangle$,
\textit{let $\mathcal{M}_T$ be the MDP}
$(S', A', r', p', \gamma')$
\textit{defined such that}
$S' = S \times (U \cup F)$,
$A' = A$,
$\gamma' = \gamma$,
\begin{equation}
p'\big((s', u') \mid (s, u), a\big) =
\begin{cases}
p(s' \mid s, a) & \text{if } u \in F \text{ and } u' = u, \\
p(s' \mid s, a) & \text{if } u \in U \text{ and } u' = \delta_u(u, L(s, a, s')), \\
0 & \text{otherwise,}
\end{cases}
\end{equation}
and
\begin{equation}
r'\big((s, u), a, (s', u')\big)
= \delta_r(u)(s, a, s')
\quad \text{if } u \notin F \ (\text{zero otherwise}).
\end{equation}
\textit{Then any policy for $\mathcal{M}_T$ achieves the same expected discounted return in $T$, and vice versa.}
\end{observation}

Rewriting the average reward analogue as,
\begin{proposition}[Cross-product reduction]
\label{prop:cross-product}
Let $\mathcal{T}$ be an \textsc{ar-MDPRM}. Define the augmented MDP
$\mathcal{M}_{\mathcal{T}}=(\mathcal{X},A,\bar p,\bar r)$ by
\[
   \mathcal{X}\;=\;S\times(U\cup F),
   \qquad
   \bar p\bigl((s',u')\,\bigm|\,(s,u),a\bigr)
     \;=\;
     \begin{cases}
        p(s'|s,a) & \text{if } u\in U \text{ and } u'=\delta_u(u,L(s,a,s')),\\
        p(s'|s,a)\,\mathbf{1}\{u'=u\} & \text{if } u\in F,\\
        0 & \text{otherwise},
     \end{cases}
\]
and $\bar r\bigl((s,u),a,(s',u')\bigr)=\delta_r(u)(s,a,s')$ when $u\in U$
(and $0$ on $F$). Then for any policy
$\pi:\mathcal{X}\to A$, the expected average reward induced in $\mathcal{T}$ and $\mathcal{M}_{\mathcal{T}}$ are equal.
\end{proposition}

\paragraph{The HRM construction:} Hierarchical Reward Machines (HRMs) introduced in~\cite{toro_reward_machines}, define options for each ordered pair $\langle u,u_t\rangle\in U\times(U\cup F)$ such that
$u_t=\delta_u(u,\sigma)$ for some $\sigma\in 2^{\mathcal{P}}$. Concretely, we can define the options with the initiation set $\mathcal{I}_{\langle u,u_t\rangle}=\{(s,u):s\in S\}$, intra-option policy $\pi_{\langle u,u_t\rangle}(a\,|\,s)$, and termination condition 
\begin{align}
\beta_{\langle u, u_t \rangle}(s', u') &=
\begin{cases}
1 & \text{if } u' \neq u \text{ or } s' \text{ is terminal}, \\
0 & \text{otherwise,}
\end{cases}
\end{align}
i.e. the option terminates when the underlying state in the RM changes, or the environment resets. This induces an SMDP on the option-level state
space $\mathcal{X}=S\times U$, with the holding time
\[
   \ell\bigl((s,u),\langle u,u_t\rangle\bigr)
   \;=\;
   \inf\bigl\{\,k\ge 1:u_k\neq u\;\text{or}\;s_k\text{ terminal}\bigr\}.
\]
The reward $\hat r((s,u),\langle u,u_t\rangle)$ and expected holding
time $\hat \ell((s,u),\langle u,u_t\rangle)$ are then defined as before, and the normalised inter-option and
intra-option Bellman operators $\bar H_{\mathrm{inter}}$ and $H_{\mathrm{intra}}$ of
Eqs.~\ref{eq:normalised_inter_operator},~\ref{eq:intra-operator} act on
$Q\in\mathbb{R}^{\mathcal{X}\times\mathcal{O}}$ exactly as before.

However, note that HRM options may terminate deterministically and hence need not satisfy the strict condition $\beta(x,o)<1$. In this case we apply the lazy intra-option operator $H_{\rm \eta, intra}=(1-\eta)I+\eta H_{\rm intra}$ (Appendix~\ref{app:lazy}). Its optimality class coincides with that of $H_{\rm intra}$, while its augmented kernels have uniformly positive self-loops. A model-free $J$-step estimator is obtained by inserting an independent Bernoulli-$\eta$ lazy branch at each rollout step. Specifically, to simulate a $J$-step rollout for the lazy intra-option operator
$H_{\mathrm{\eta, intra}}\triangleq (1-\eta)I+\eta H_{\mathrm{intra}}$,
we augment the original intra-option rollout with an independent lazy coin toss at
each step.  At the $i$th step of the $J$-step rollout with $X_k^{(i)}=(S_k^{(i)},O_k^{(i)})$, draw
$L_k^{(i+1)}\sim\mathrm{Bernoulli}(\eta)$.  If
$L_k^{(i+1)}=0$, set $X_k^{(i+1)}=X_k^{(i)}$ and incur zero reward.  If
$L_k^{(i+1)}=1$, perform one ordinary intra-option transition as in Algorithm~\ref{alg:sync-intra}: sample the
primitive action and next state, observe the one-step reward, sample the
termination indicator according to $\beta$, and, upon termination, choose the
next option greedily with respect to the remaining-horizon lazy empirical Bellman value
$
O_k^{(i+1)}
\in
\arg\max_{o'\in\mathcal O}
\left(\widehat H_{k,\eta, \mathrm{intra}}^{\,J-i-1}Q_k\right)
(S_k^{(i+1)},o')
$. Here, $\widehat H_{k,\eta, \mathrm{intra}}$ is the empirical Bellman operator corresponding to $H_{\eta, \mathrm{intra}}$, constructed like the other empirical Bellman operators in Appendix~\ref{subsec:sync_estimators}.
Thus, the lazy $J$-step estimator is obtained from the original intra-option
rollout by inserting an independent Bernoulli-$\eta$ self-loop at every stage. Therefore the contraction and finite-time analysis will apply to the lazy intra-option recursion.

\subsection{Verifying Assumptions}
We now verify that the Assumptions~\ref{ass:inter-assumption},~\ref{ass:intra-assump} hold for the HRM framework.

\begin{assumption}[HRM unichain and reachability]
\label{ass:rm-unichain}
The augmented MDP $\mathcal{M}_{\mathcal{T}}$ defined in
Proposition~\ref{prop:cross-product} satisfies:
\begin{enumerate}
\item Assumption~\ref{ass:intra-assump} for intra-option and Assumption~\ref{ass:inter-assumption} for inter-option setting hold with $\mathcal{X}$ as $S$, i.e., for every deterministic stationary policy $\mu:\mathcal{X}\to\mathcal{O}$, the induced Markov chain on $\mathcal{X}\times\mathcal{O}$ is unichain.
\item The positivity conditions hold on $\mathcal{X}$ i.e., $p(x\,|\,x,o)>0$ for all $x\in \mathcal{X}$ and $o\in O$ (this can be relaxed to the standard aperiodicity transformation of \cite[Section~8.5.4]{mdp_book}).
\item The holding times satisfy
      $\mathbb{E}\bigl[\,\ell((s,u),o)^2\,\bigr]\le c^2$ for all
      $(s,u)\in\mathcal{X}$ and $o\in\mathcal{O}$.
\end{enumerate}
\end{assumption}

\textit{Remark.} If the RM has no terminal states, i.e. $F=\emptyset$ (the never-ending case considered in~\cite{toro_reward_machines}), and the underlying MDP is unichain, then the augmented MDP $\mathcal{M}_{\mathcal{T}}$ is unichain if the RM is strongly connected on the recurrent class of every induced chain. Under the analogous hypothesis in the discounted reward setting ~\cite{toro_reward_machines} prove convergence for \textsc{CRM}. If the RM has terminal states, we can introduce episode restarts, where upon reaching a state $u \in F$, the augmented chain is reinitialised to $(s_0, u_0)$ following a fixed distribution. In this case, the resulting kernel $\bar{p}$ is again unichain and aperiodic. 

\textit{Remark.} Under the intra-option policy $\pi_{\langle u,u_t\rangle}$, the hitting time of the set $\{u'\neq u\}\cup\{s'\text{ terminal}\}$ from $(s,u)$ is given by $\ell((s,u),\langle u,u_t\rangle)$. On the finite cross-product state space $\mathcal{X}$, any irreducible chain has hitting time with finite moments of every order, so Assumption~\ref{ass:rm-unichain}(3) follows whenever the intra-option chain is irreducible on $\{(s,u):s\in S, u \in U \}$. This relaxes the almost-surely bounded holding-time assumption required by a naive application of \cite{chen2025nonasymptotictheoryseminormlyapunov, chen2025nonasymptoticguaranteesaveragerewardqlearning} to the second-moment condition already used in our main results, and is consistent with
the asymptotic analysis of \cite{yu2024asynchronous}.

\par The $J$-step rollouts are now defined on the option level chain over $\mathcal{X}\times\mathcal{O}$, and the resulting estimators follow as in Section~\ref{sec:algo}, with intermediate options chosen greedily with respect to empirical operators, and the bias is controlled exactly as in Appendix~\ref{sec:synchronous}, where the next state $x_{i+1}=(s_{i+1},u_{i+1})$ is obtained following the RM transition $u_{i+1}=\delta_u(u_i,L(s_i,a_i,s_{i+1}))$.

\subsection{J-step span contraction for MDPRM operators}

\begin{theorem}[HRM inter-option and intra-option J-step contraction]
\label{thm:rm-hrm-contraction}
Under Assumption~\ref{ass:rm-unichain}(1)--(3) applied to the HRM
formulation, the inter-option operator $\bar H_{\mathrm{inter}}$ of
Eq.~(\ref{eq:normalised_inter_operator}) and the lazy intra-option operator $H^{(\eta)}_{\mathrm{intra}}$ satisfy, for some
$1\le J\le |\mathcal{X}||\mathcal{O}|$ and $\gamma_{\rm inter}, \gamma_{\rm \eta, intra}\in(0,1)$,
\[
   \mathrm{sp}\bigl(\bar H_{\mathrm{inter}}^{J}Q-\bar H_{\mathrm{inter}}^{J}Q'\bigr)
   \;\le\;\gamma_{\rm inter}\,\mathrm{sp}(Q-Q'),
   \qquad
   \mathrm{sp}\bigl((H_{\mathrm{\eta, intra}}^{J}Q-H_{\mathrm{\eta, intra}}^{J}Q'\bigr)
   \;\le\;\gamma_{\rm \eta, intra}\,\mathrm{sp}(Q-Q').
\]
\end{theorem}

\begin{proof}[Proof sketch]
In our formulation, the options operate on a finite
state-option space $\mathcal{X}\times\mathcal{O}$ of size
$|S||U|\cdot|\mathcal{O}|$ with finite expected holding time. Lemma~\ref{le:overlap} (overlap lemma) along with the policy-sandwich argument for the $J$-fold Bellman iterate, hold exactly with $\mathcal{X}$ as the state space instead of $S$,
proving the $J$-step contraction for $H_{\mathrm{intra}}$ and $\bar H_{\mathrm{inter}}$.
\end{proof}

\subsection{Finite-time bounds for HRM}
Combining Theorem~\ref{thm:rm-hrm-contraction} with Theorems~\ref{thm:  sync-intra-ftb},~\ref{thm:sync-inter-ftb} yields finite-time guarantees for average-reward Q-learning on MDPRMs.

\begin{corollary}[Finite-time bound for synchronous Q-learning on MDPRMs]
\label{cor:sync-rm}
    Consider the synchronous J-step recursion
\[
   Q_{k+1}
   \;=\;
   Q_k\,+\,\alpha_k\bigl(\widehat H^{J}(Q_k)-Q_k\bigr)
\]
on $\mathcal{X}\times\mathcal{O}$ where $\widehat H^{J}$ is any of the J-step estimators defined before.  Under Assumption~\ref{ass:rm-unichain} and
$\alpha_k=\frac{4}{(1-\gamma)(k+h)}$ with $\gamma$ according to inter/ intra-option, and $h$ as in Theorems~\ref{thm:sync-inter-ftb} and~\ref{thm: sync-intra-ftb},
\[
   \mathbb{E}\bigl[\,\mathrm{sp}(Q_k-Q^{*})^{2}\,\bigr]
   \;=\;\widetilde O(1/k),
\]
where the leading constants are the same as in Theorems~\ref{thm:sync-inter-ftb} and~\ref{thm: sync-intra-ftb}, with $|S||\mathcal{O}|$ replaced by $|S||U||\mathcal{O}|$. 
\end{corollary}

    

\textit{Remark.} The inter-option Q iterates of Corollary~\ref{cor:sync-rm} converge to a fixed point of $\bar H_{\mathrm{inter}}$ on $\mathcal{X}\times\mathcal{O}$. If we assume non-zero probability of
option termination in $|S||U|$ steps (following Assumption 4.2 in~\cite{wan2024convergenceaveragerewardqlearningweakly}), greedy policies with respect to this fixed point are hierarchically optimal for the original \textsc{ar-MDPRM} with respect to the given option policies~\cite{JMLR:v8:ghavamzadeh07a}. Our finite-time analysis gives the first non-asymptotic guarantees of hierarchical optimality for HRM in the average-reward setting.

\paragraph{RM as sub-goal sequence tasks}: For RMs which are acyclic in $U$, i.e., if a sequence of sub-goals is encoded as a chain RM $u_0\to u_1\to
\cdots\to u_{|U|-1}\to t$, which restarts on reaching $t$, Corollary~\ref{cor:sync-rm} applies directly. Using $\mathrm{sp}(Q^*)\le J/(1-\gamma)$, the intra-option bound becomes
$\mathbb{E}[\mathrm{sp}(Q_k-Q^{*})^{2}]=
\widetilde O\bigl(J^{2}\log(|S||U||\mathcal{O}|)/((1-\gamma)^{5}k)\bigr)$, which depends logarithmically on $|U|$ explicitly, but the size of the RM also enters the bound through $J, \gamma$ of the cross-product chain. The inter-option bound has an additional term in $|S||U||\mathcal{O}|\,c^2J^2$ which makes the bound linear in $|U|$.

\newpage
\section{\MakeUppercase{Additional experimental results}}
\label{sec:extra-experiments}

\begin{figure}[h!]
    \centering

    \begin{subfigure}{0.45\textwidth}
        \centering
        \includegraphics[width=\linewidth]{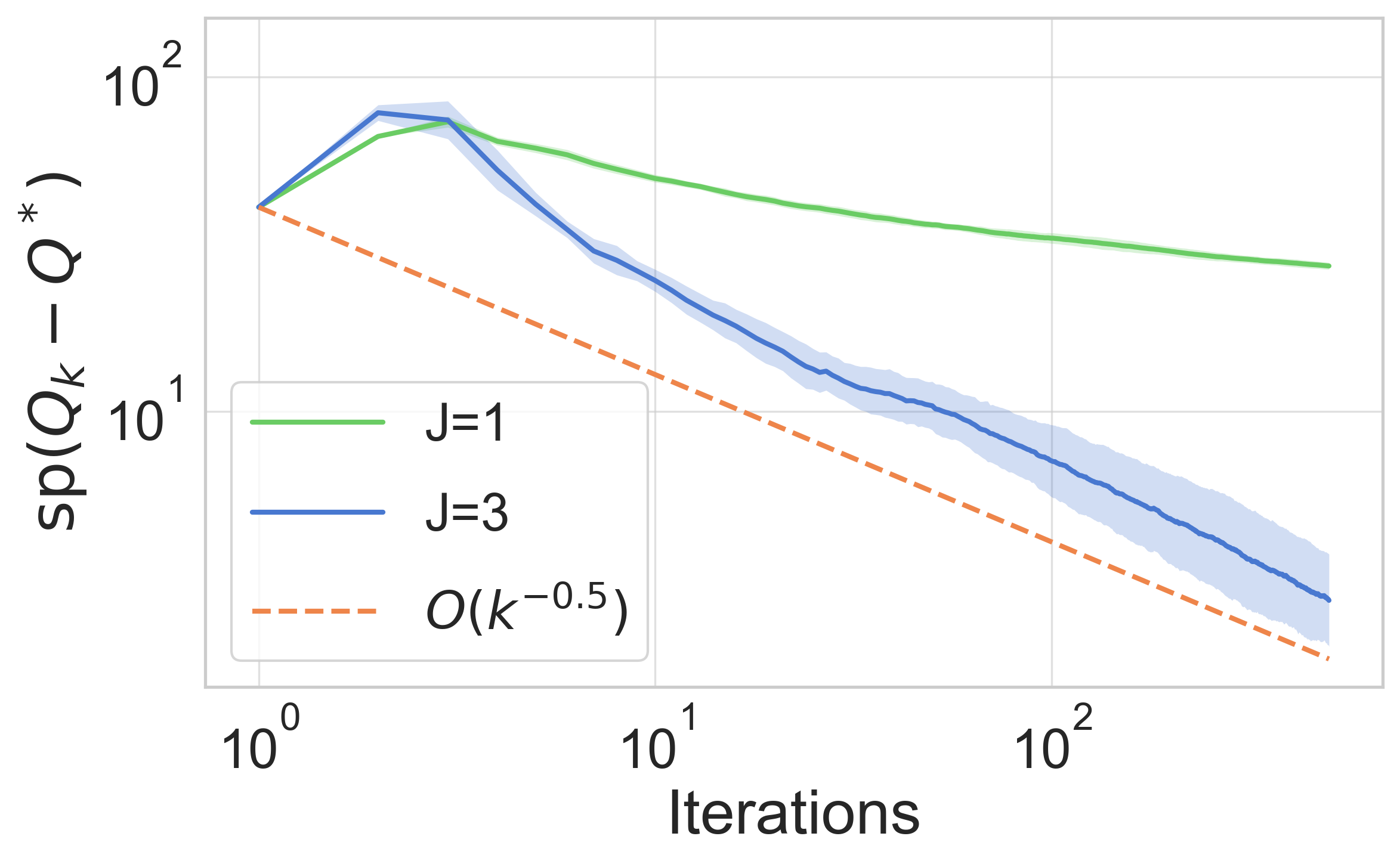}
        \caption{Taxi}
        \label{fig:sub1}
    \end{subfigure}
    \hfill
    \begin{subfigure}{0.45\textwidth}
        \centering
        \includegraphics[width=\linewidth]{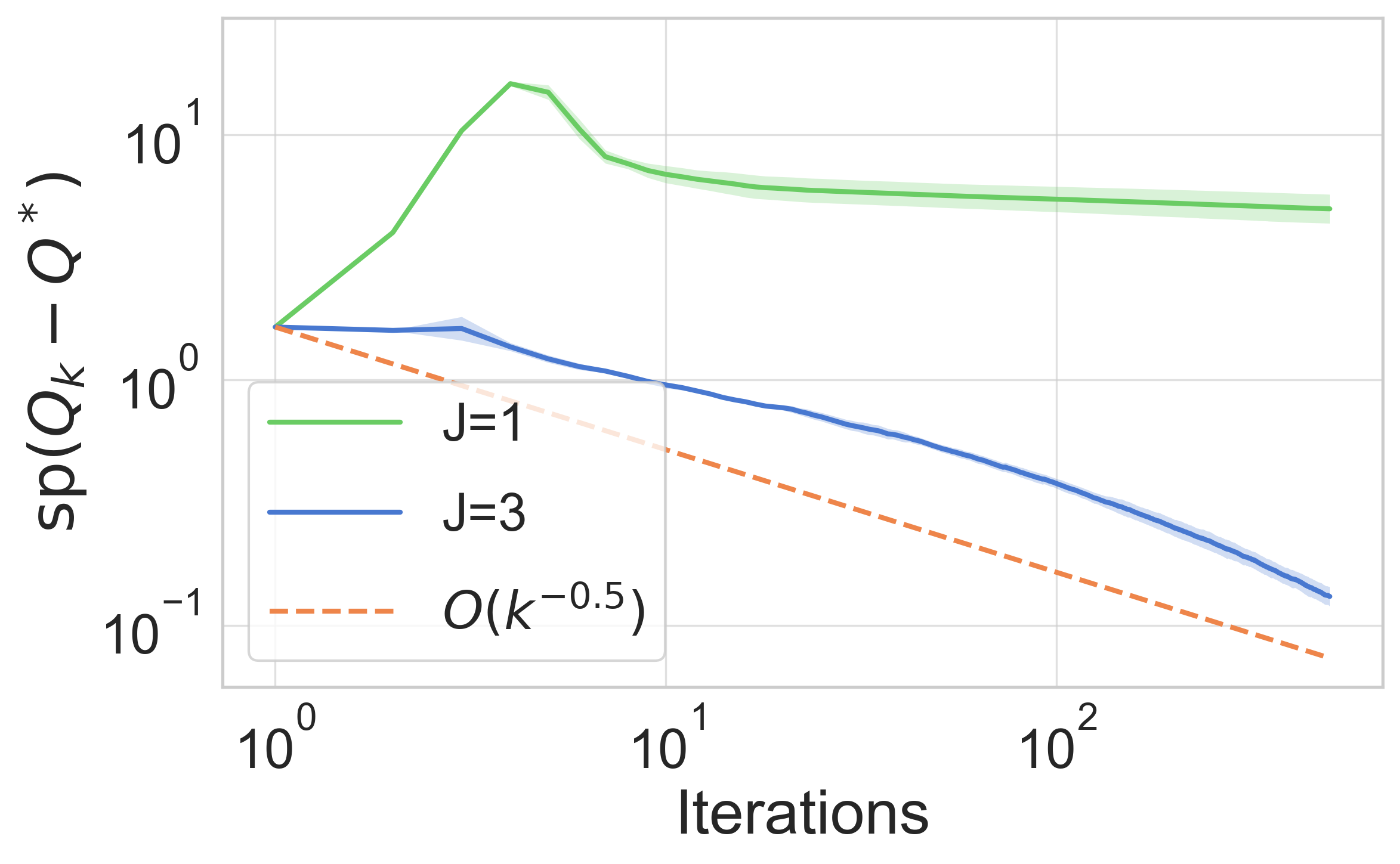}
        \caption{Office gridworld}
        \label{fig:sub2}
    \end{subfigure}

    \vspace{0.5cm}

    \begin{subfigure}{0.45\textwidth}
        \centering
        \includegraphics[width=\linewidth]{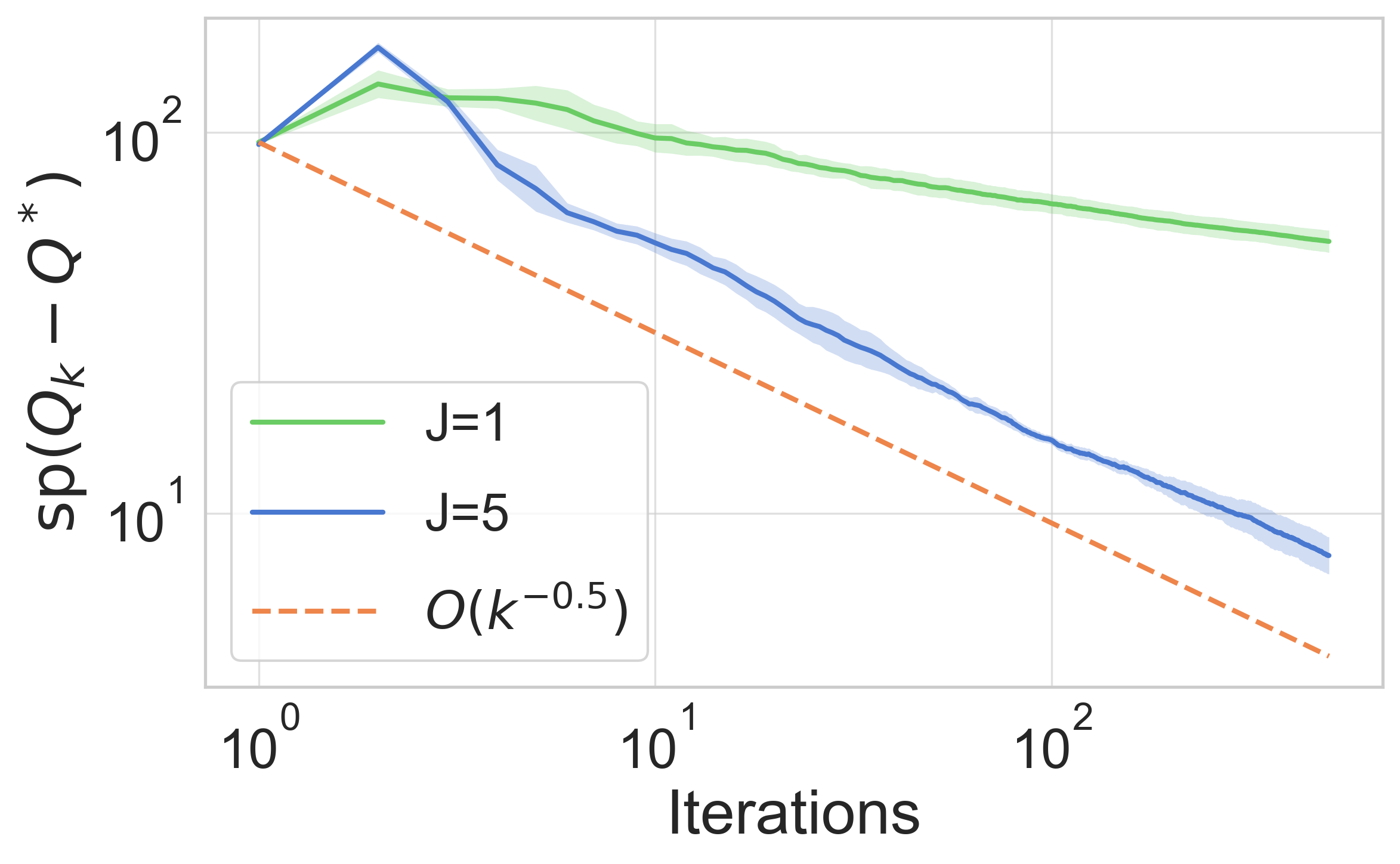}
        \caption{Taxi}
        \label{fig:sub3}
    \end{subfigure}
    \hfill
    \begin{subfigure}{0.45\textwidth}
        \centering
        \includegraphics[width=\linewidth]{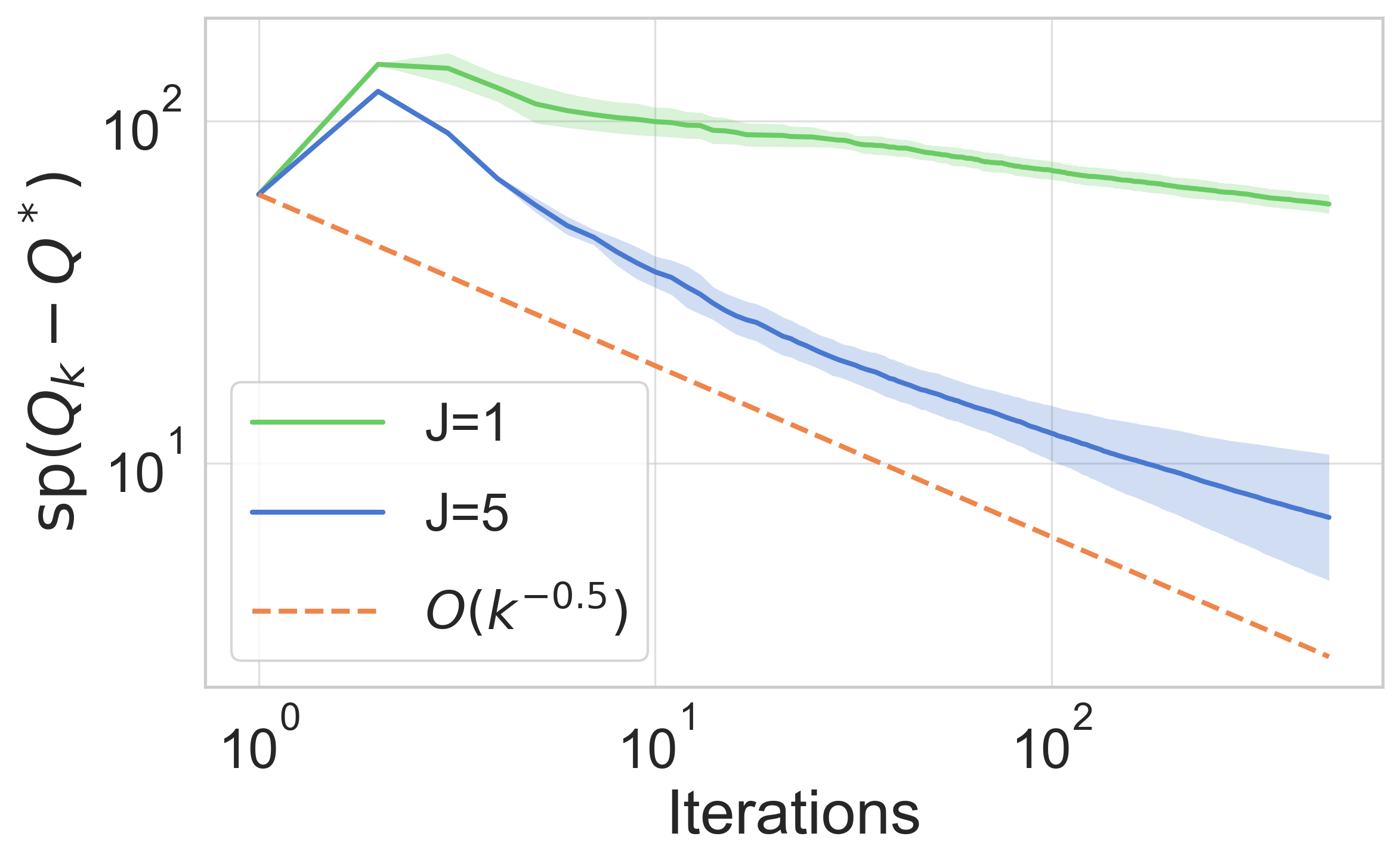}
        \caption{Office gridworld}
        \label{fig:sub4}
    \end{subfigure}

    \caption{Evolution of $\mathrm{sp}(Q_k-Q^*)$ with iterations $k$ for synchronous learning for intra-option setting (a), (b) and inter-option setting (c), (d) for \emph{Taxi} and \emph{Office gridworld} respectively.}
    \label{fig:sync_extra}
\end{figure}







In this section, we provide results for two additional benchmarks- \emph{Taxi}~\citep{taxi_diet} and \emph{Office gridworld}, and provide additional experimental details.

\paragraph{Environments: } 
\emph{Four Rooms: } The environment consists of four labeled $3\times 3$ rooms in a $2\times 2$ grid, as shown in Fig~\ref{fig:mdprm2}. The state space has $36$ cells, and the action space consists of $4$ primitive actions \emph{up, down, left, bottom} with stochastic transitions (transitions to incorrect state in perpendicular direction with probability $0.1$), and the walls act as reflectors. The environment deals with a patrolling task, requiring to visit specific locations $A \to B \to C \to D \to A$ repeatedly in that order, producing a reward of $+1$. In the labelled MDP, the current observation is not sufficient to predict the next optimal action and this has to be combined with also the current RM state. The cross product MDP on which the task is Markovian has $36 \times 4 = 144$ states.

\begin{figure}[h!]
\centering
\begin{subfigure}{.5\textwidth}
  \centering
  \includegraphics[width=.4\linewidth]{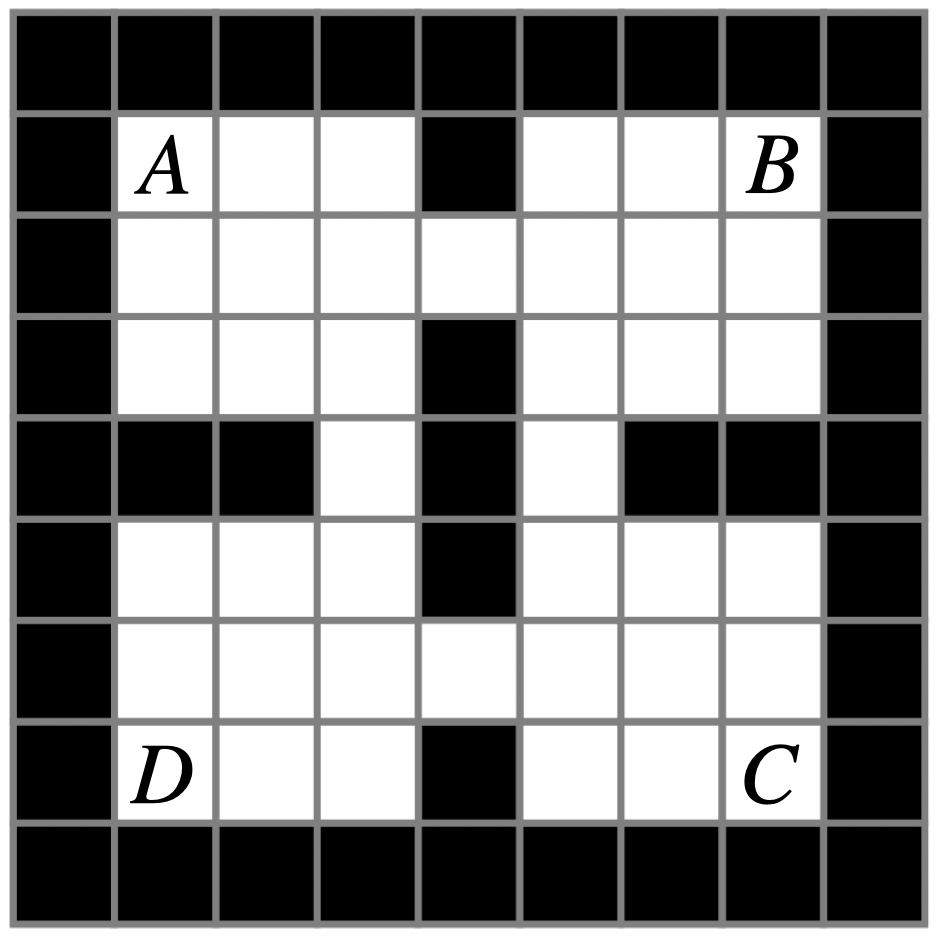}
  \caption{ Labeled-MDP}
  \label{fig:sub1}
\end{subfigure}%
\begin{subfigure}[b]{0.5\textwidth}
        \centering
        \begin{tikzpicture}[
            node distance=1.0cm,
            state/.style={circle, draw, thick, minimum size=0.6cm, font=\large},
            arr/.style={-{Stealth[length=2pt]}, thick}
        ]
            \node[state] (qA) {$q_A$};
            \node[state, right=of qA] (qB) {$q_B$};
            \node[state, below=of qB] (qC) {$q_C$};
            \node[state, below=of qA] (qD) {$q_D$};

            \draw[arr][-{Stealth[scale=1.0]}] (qA) -- node[above] {$B$} (qB);
            \draw[arr][-{Stealth[scale=1.0]}] (qB) -- node[right] {$C$} (qC);
            \draw[arr][-{Stealth[scale=1.0]}] (qC) -- node[below] {$D$} (qD);
            \draw[arr][-{Stealth[scale=1.0]}] (qD) -- node[left]  {$A$} (qA);

            \node at ($(qA)!0.5!(qD)+ (0.55,0)$) {$r=1$};
        \end{tikzpicture}
        \caption{RM}
    \end{subfigure}
\caption{\emph{Four Rooms} gridworld with $3\times 3$ rooms, four primitive actions, and a reward machine corresponding to a patrolling task.}\label{fig:mdprm2}

\end{figure}

\emph{Taxi: } We use the standard \emph{Taxi} domain~\citep{taxi_diet}, with a $5 \times 5$ grid, and $4$ possible passenger locations, with an additional fifth location as \emph{inside} the taxi, giving a $5\times 5$ (taxi location) $\times 5$ (passenger location) $\times 4$ (destination location) $= 500$ states. The action space consists of $4$ primitive actions \emph{up, down, left, right} for movement along with \emph{pickup} and \emph{drop-off}. To make the task average-reward, we remove the episode termination and reset the environment after a successful delivery.

\emph{Office gridworld: } We adapt the \emph{Office gridworld} as described in~\cite{toro_reward_machines}, with a $12 \times 9$ gridworld, with two \emph{coffee} locations, one \emph{mail} location and an office location at $o$. The agent must fetch \emph{coffee} and \emph{mail}, and arrive at the \emph{office}. The non-Markovian reward can be encoded by a $4$ state reward machine.

\paragraph{Options: } For the \emph{Four Rooms} environment, we use $4$ hallway-reaching options, one for each hallway, with BFS-greedy intra-option policy with $\epsilon=0.05$ exploration. For \emph{Taxi}, we use four \emph{go-to-location} options, one for each designated location of the passenger (R,G,Y,B), and two single-step options with primitive actions \emph{pickup} and \emph{drop-off}. For \emph{Office gridworld}, we use four options corresponding to reaching locations for the \emph{mail}, \emph{office}, and two different locations for \emph{coffee}.

\paragraph{Algorithms and stepsizes:} We show convergence results for intra and inter-option $Q$-learning, each with $J\in \{1,3 \}$, and show contraction factors for different values of $J$ in Table~\ref{tab:contraction_all}. We use decaying stepsize $\alpha_k = \alpha/(k+h)$ with $\alpha = 4/(1-\gamma)$. 
Each curve in Figures~\ref{fig:sync_extra} is averaged over $>5$ independent seeds, and the shaded band show one standard deviation across seeds. 

\textbf{Computational resources:} All experiments are tabular and were run on a single laptop CPU (Apple M2 MacBook, 8-core CPU, 16 GB unified memory, Python 3.9.6). Each algorithm stores tables of size $|\mathcal{S}||\mathcal{O}|$, total memory used is under 1GB. The entire experimental sweep, including all environments and all algorithms, completes in under three hours of wall-clock time, no GPU was used.\footnote{The code is available at \url{https://anonymous.4open.science/r/NA_convergence-3B11/}.}

\paragraph{Experimental results: }
Figures~\ref{fig:sync_extra} show the span seminorm error $\mathrm{sp}(Q_k - Q^*)$ for \emph{Taxi} and \emph{Office gridworld} domain. Similar to \emph{Four Rooms}, the estimator with higher $J$ for each of the environments converges strictly faster than the $J=1$ one-step estimator, following the $\mathcal{O}(k^{-1/2})$ reference line. 


\begin{figure}[h!]
    \centering
    \includegraphics[width=0.5\textwidth]{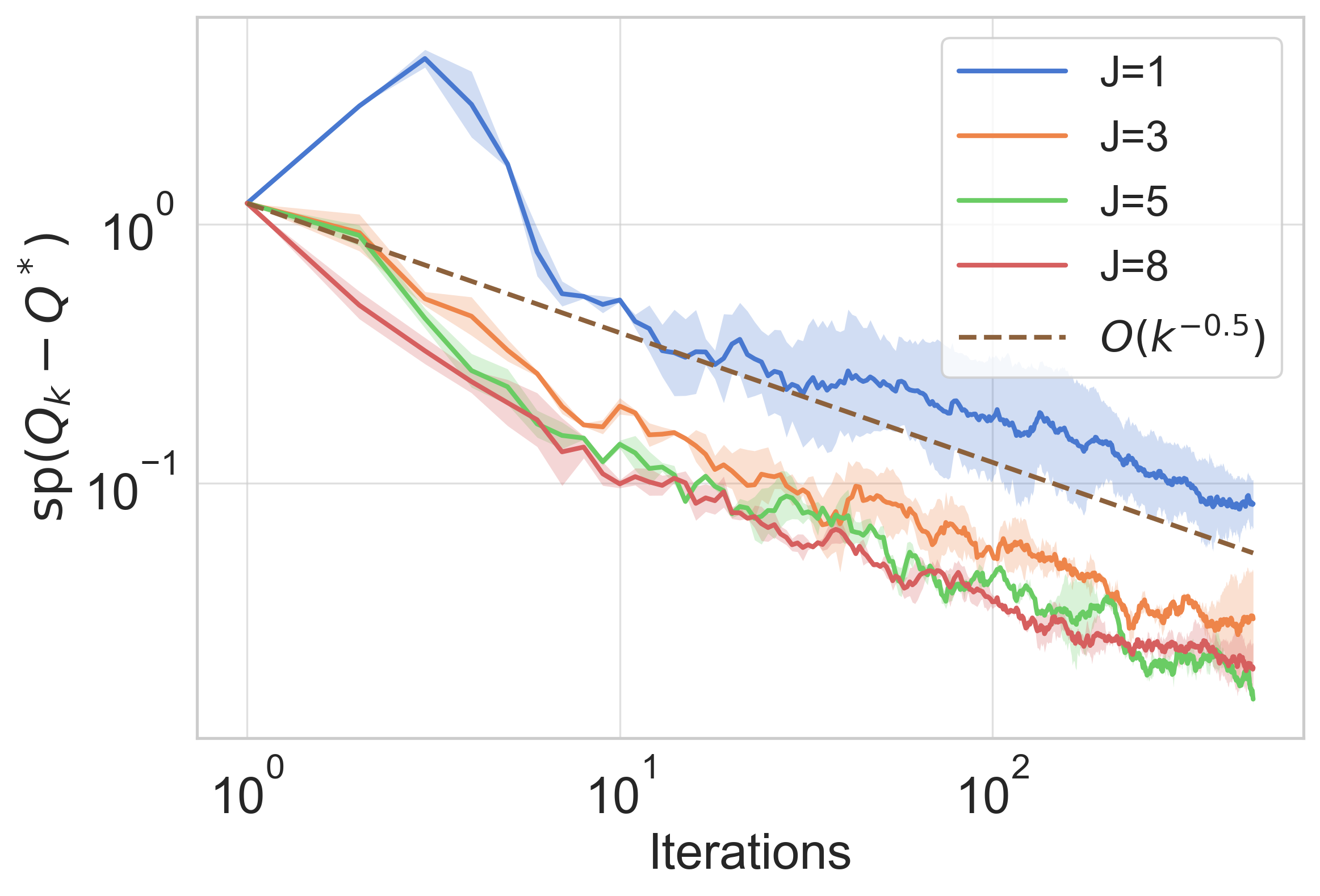} 
    \caption{Evolution of $\mathrm{sp}(Q_k - Q^*)$ over iterations $k$ for synchronous intra-option setting for \emph{Four Rooms} environment.}
    \label{fig:J-comparison}
\end{figure}

In Figure~\ref{fig:J-comparison} we plot the span seminorm error $\mathrm{sp}(Q_k - Q^*)$ over $k$ iterations, for different values of $J$, and we observe that with increasing values of $J$ the slope of the convergence curve becomes steeper with diminishing returns for higher values of $J$.


\begin{table}[h!]
\centering
\caption{Contraction factor $\gamma$ for the three environments for different values of $J$ for the synchronous setting.}
\label{tab:J_results}
\small
\setlength{\tabcolsep}{6pt}

\begin{tabular}{lcccccc}
\toprule
J &
\multicolumn{2}{c}{Four Rooms} &
\multicolumn{2}{c}{Taxi} &
\multicolumn{2}{c}{Office gridworld} \\
\cmidrule(lr){2-3}
\cmidrule(lr){4-5}
\cmidrule(lr){6-7}

 &
Intra-option & Inter-option & 
Intra-option & Inter-option &
Intra-option & Inter-option \\
\midrule

1 & $0.81 \pm 0.07$ & $0.821 \pm 0.05$ & $0.79 \pm 0.08$ & $0.83 \pm 0.05$ & $0.75 \pm 0.07$ & $0.83 \pm 0.07$ \\
2 & $0.68 \pm 0.08$ &$ 0.70 \pm 0.06$ & $0.68 \pm 0.08$ & $0.73 \pm 0.05$ & $0.63 \pm 0.06$ & $0.76 \pm 0.07$ \\
3 & $0.56 \pm 0.09$ & $0.62 \pm 0.06$ & $0.58 \pm 0.08$ & $0.66 \pm 0.05$ & $0.54 \pm 0.06$ & $0.72 \pm 0.06$ \\
4 & $0.48 \pm 0.08$ & $0.55 \pm 0.07$ & $0.48 \pm 0.09$ & $0.59 \pm 0.04$ & $0.48 \pm 0.05$ & $0.66 \pm 0.06$ \\

\bottomrule
\end{tabular}
\label{tab:contraction_all}
\end{table}

\begin{figure}[h!]
    \centering

    \begin{subfigure}[b]{0.32\linewidth}
        \centering
        \includegraphics[width=\linewidth]{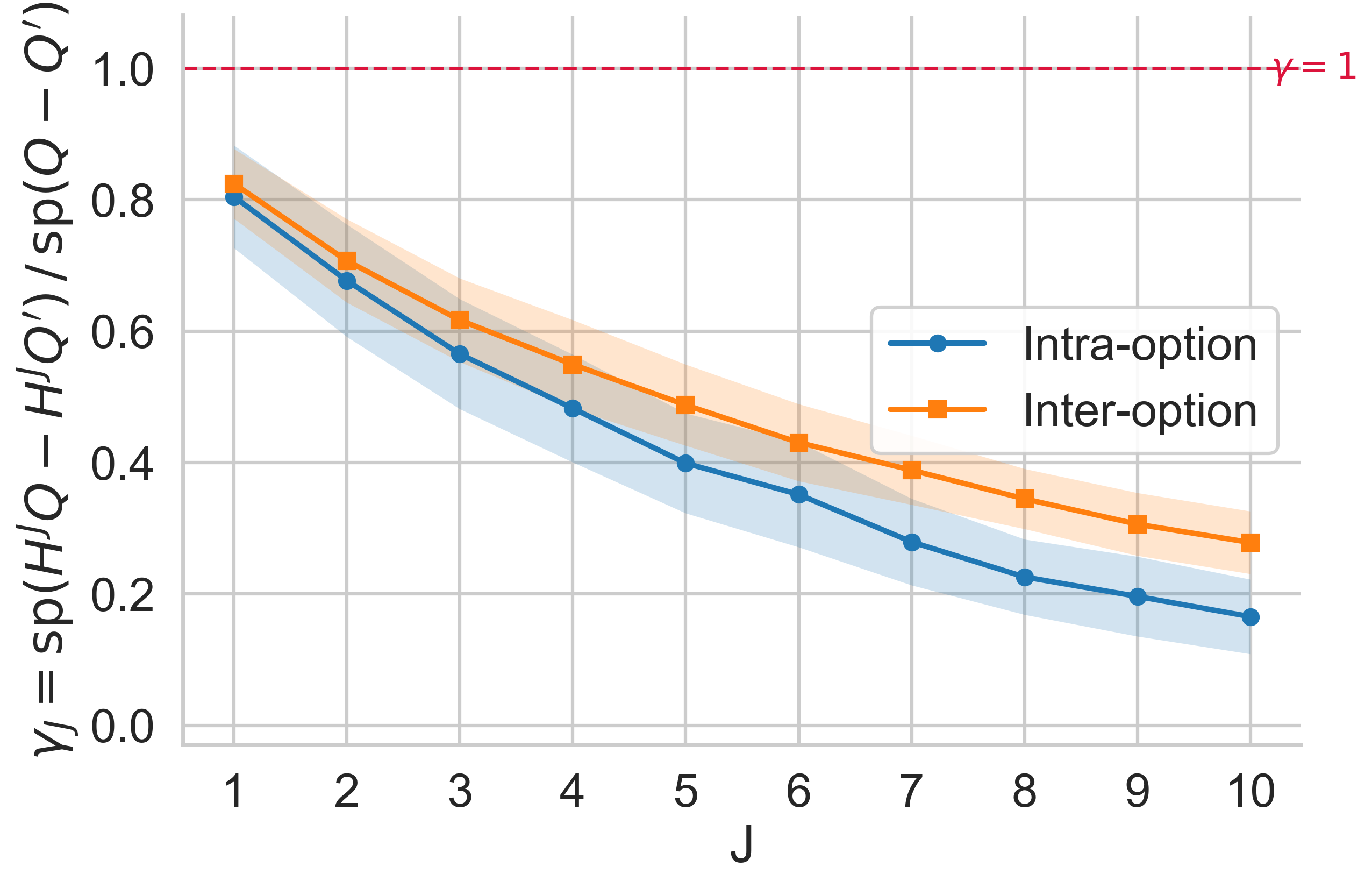}
        \caption{$\gamma$ vs $J$}
        \label{fig:gamma_vs_J}
    \end{subfigure}
    \hfill
    \begin{subfigure}[b]{0.32\linewidth}
        \centering
        \includegraphics[width=\linewidth]{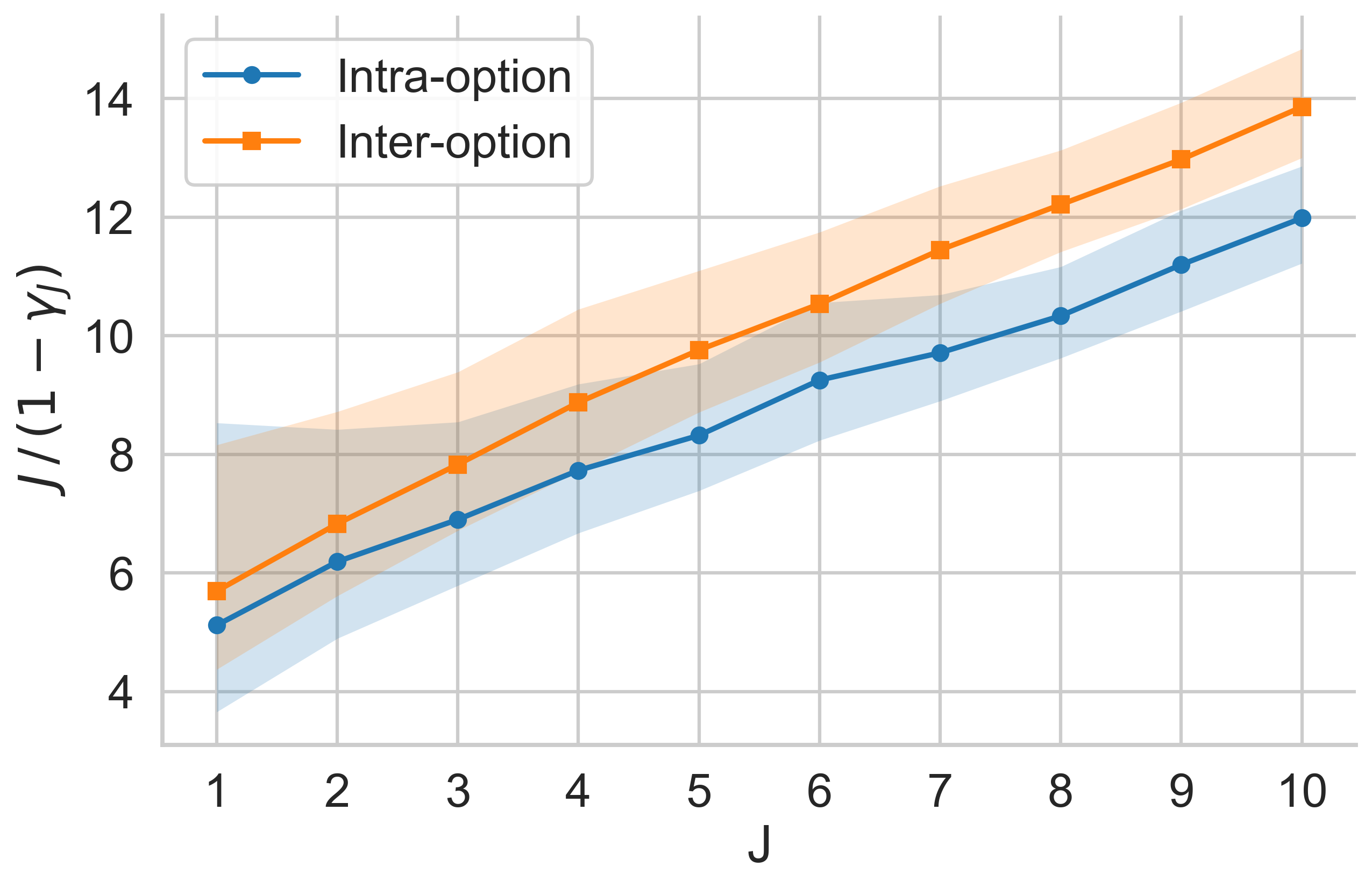}
        \caption{$\frac{J}{(1-\gamma)}$ vs $J$}
        \label{fig:J_over_1_minus_gamma}
    \end{subfigure}
    \hfill
    \begin{subfigure}[b]{0.32\linewidth}
        \centering
        \includegraphics[width=\linewidth]{images/4room/j_gamma_tradeoff/c_finite_time_bound.png}
        \caption{$\frac{(\sp(Q^*)+J)^2}{(1-\gamma)^3}$ vs $J$}
        \label{fig:c_finite_time_bound}
    \end{subfigure}
    \caption{$J$-step contraction and $J$-step selection guideline for intra-and inter-option operators for \emph{Four rooms} environment.}
    \label{fig:j_gamma_tradeoff}
\end{figure}



\subsection{A guideline for selecting $J$} 
\label{app:guideline}
The finite-time bounds of Theorems~\ref{thm:sync-inter-ftb} and~\ref{thm: sync-intra-ftb} vary with $J$ in two different ways as shown in Appendix~\ref{sec:synchronous}. Given $d \triangleq |\mathcal{S}|\,|\mathcal{O}|$ and $b_k = \mathcal{O}(J\log k)$ (Lemma~\ref{le:growth-intra-sync}), we have the intra-option bound following Theorem~\ref{thm:intra-final}:
\begin{equation}\label{eq:bound_explanation}
    \mathbb{E}\!\left[\mathrm{sp}(Q_k-Q^\star)^2\right]
= \mathcal{O}\!\Bigg(
\underbrace{\frac{(\mathrm{sp}(Q^\star)+J)^2\,\log d}{(1-\gamma)^{3}\,k}}_{\textrm{(i) rollout variance}}
\;+\;
\underbrace{\frac{C_d\,J^2\big(1+(b_k+J)^2\big)}{(1-\gamma)^{2}\,k}}_{\textrm{(ii) approximately-greedy bias}}
\Bigg),
\qquad C_d = |\mathcal{S}| + 2\log d + 2 .
\end{equation}

Following Theorem~\ref{thm:inter-final}, the inter-option bound has an additional term of order $d\,c^2 J^2\big[1+(2b_k+J)^2\big]/\big((1-\gamma)^{2}k\big)$. In Equation~\eqref{eq:bound_explanation}, term~(i) is due to the rollout noise $T_{\mathrm{quad}}$, controlled in Proposition~\ref{prop:T-quad-intra}, and term~(ii) is due to the approximately greedy perturbation $T_{\mathrm{greedy}}$ controlled in Proposition~\ref{prop:T-greedy}, which arises from choosing the intermediate options using the empirical Bellman operators $\widehat{H}_k$, and since $b_k=\mathcal{O}(J\log k)$, this term varies with $\widetilde{\mathcal{O}}(J^4)$ and also inherits $|S|$ through $C_d$. The inter-option terms are analogous and additionally carries $|\mathcal{S}|\,|\mathcal{O}|\,c^2$.

The two terms scale differently: firstly, the rollout variance term carries the factor $(1-\gamma)^{-3}$ while the estimation terms carries $(1-\gamma)^{-2}$, so for a domain with slow contraction the bound would be more sensitive to the first term. Secondly, the estimation term~(ii) is the worst-case bound on the model-free estimation error, and the factors $|\mathcal{S}|$ for intra- and $|\mathcal{S}||\mathcal{O}|$ for inter-option come from a union bound over the entire state-option space (Lemma~\ref{lem:empirical-bellman-one-step-explicit}). This particular bound is deliberately conservative, and we show so in our experiments which run the approximately greedy algorithm and show that increasing $J$ yields faster convergence, confirming that the $\widetilde{\mathcal{O}}(J^4)$ is far from tight in practice.

We track the intrinsic rollout-contraction trade-off i.e. the term~(i) of the finite-time bound (without the loose estimation term) in Equation~\eqref{eq:bound_explanation}
$\frac{\big(\mathrm{sp}(Q^\star)+J\big)^2}{\big(1-\gamma\big)^{3}},$ plotted against $J$ in Figure~\ref{fig:j_gamma_tradeoff}(c), for both intra- and inter-option operators in the \emph{Four Rooms} environment. We observe that increasing $J$ strengthens the span contraction lowering the denominator $(1-\gamma)^{-3}$, but increases the length of the rollout and the target variance $(\mathrm{sp}(Q^\star)+J)^2$. This trade-off is balanced at an internal optimal $J^*$, which is consistent with the diminishing returns observed for increasing $J$ in the convergence curves in Figure~\ref{fig:J-comparison}.


\end{document}